\documentclass[twoside,11pt]{article}
\usepackage[top=1in, bottom=1in, left=1in, right=1in]{geometry}

\usepackage{amsthm}
\usepackage{hyperref}
\usepackage{url}
\usepackage{hyperref}
\usepackage{url}
\usepackage[utf8]{inputenc} % allow utf-8 input
\usepackage[T1]{fontenc}    % use 8-bit T1 fonts
\usepackage{booktabs}       % professional-quality tables
\usepackage{amsfonts}       % blackboard math symbols
\usepackage{nicefrac}       % compact symbols for 1/2, etc.
\usepackage{microtype}      % microtypography
\usepackage{mathtools}
\usepackage{amssymb}
\usepackage[shortlabels]{enumitem}
\usepackage{xcolor,colortbl}
\usepackage{comment}
\usepackage{subcaption}
\usepackage{graphicx}
\usepackage{makecell}
\usepackage{multirow}
\usepackage{float}
\usepackage{booktabs,caption}
\usepackage[flushleft]{threeparttable}
\usepackage{ulem}
\usepackage{algorithm}
\usepackage{algpseudocode}
\usepackage[T1]{fontenc}
\usepackage{lmodern}
\usepackage{authblk}
\usepackage[round]{natbib}

\newcommand{\mbf}[1]{\boldsymbol{#1}}
\newcommand{\mbb}[1]{\mathbb{#1}}
\newcommand{\mcal}[1]{\mathcal{#1}}
\newcommand{\mrm}[1]{\textrm{#1}}
\newcommand\numberthis{\addtocounter{equation}{1}\tag{\theequation}}
\newcommand{\inprod}[2]{\ensuremath{\left\langle{#1}, {#2}\right\rangle}}
\newcommand{\norm}[1]{\left\|{#1}\right\|}
\newcommand{\abs}[1]{\left|{#1}\right|}
\newtheorem{theorem}{Theorem}
\newtheorem{definition}[theorem]{Definition}
\newtheorem{remark}[theorem]{Remark}
\newtheorem{lemma}[theorem]{Lemma}
\newtheorem{corollary}[theorem]{Corollary}

\newtheorem{proposition}[theorem]{Proposition}

\title{Variable Importance Identification Through Lazy Training for Binary Classification}
\author[1]{Anand Singh}
\author[2]{Luke Pennella}
\author[3]{Eshan Kabir}
\author[4]{Xiaoxi Shen}

\affil[1]{Department of Mathematics, Georgia Institute of Technology }
\affil[2]{Department of Mathematics, University of Connecticut}
\affil[3]{Department of Mathematics, Columbia University}
\affil[4]{Department of Mathematics, Texas State University}

\begin{document}

\maketitle

\begin{abstract}%   <- trailing '%' for backward compatibility of .sty file
    Deep neural networks have been widely used in many applications (e.g., computer vision and natural language processing); however, understanding their explainability remains a challenging task. Recently, substantial research has been devoted to improving the explainability of deep neural networks, with most of this work focusing on the regression framework. In this paper, we instead focus on the binary classification framework and adopt a variable-importance framework combined with the idea of lazy training to propose an efficient algorithm for identifying important features. From a theoretical perspective, our method relies on only a minimal set of assumptions and achieves well-controlled error rates. The validity of the proposed method and algorithm is examined through extensive simulation studies and real-data applications.
\end{abstract}

\textbf{Keywords}:  Feature Attribution, Local Rademacher Complexity, Lazy Training, Neural Tangent Kernel, Variable Importance

\section{Introduction}
As a foundational building block in modern artificial intelligence (AI), deep neural networks play vital roles in various AI methods. Although deep neural networks exhibit superior predictive performance compared to classical statistical models, their explainability remains limited. However, evaluating the contribution of input variables to predicting a response is of great interest in many areas of scientific research, including healthcare, education, and genetic studies. In classical statistical models, this problem is often formulated as a hypothesis-testing task to assess the significance of a given feature. The goal of this paper is to propose an approach for identifying important variables using deep ReLU neural networks in binary classification settings.

\subsection{Related Literature}
In recent years, many methods have been proposed to make deep neural networks explainable. According to \citet{zhang2025unified}, these methods can be broadly classified into three categories: \textbf{feature attribution}, \textbf{data attribution}, and \textbf{component attribution}. Perturbation-based methods \citep{lundberg2017unified, petsiuk2018rise} and gradient-based methods \citep{simonyan2013deep, smilkov2017smoothgrad} are the two most commonly used approaches across all three categories. Our work falls within the category of perturbation-based feature attribution, which quantifies how model outputs change when input features are modified. In perturbation-based feature attribution, feature importance is generally measured by the difference in feature scores, where one score is calculated using all features and the other is obtained by modifying the features of interest.

From a statistical perspective, identifying important features can be formulated as a hypothesis testing problem:
\begin{equation}\label{Eq: Variable Importance Hypothesis}
H_0:\mrm{the features of interest are NOT important vs }H_1:\mrm{the features of interest are important}.
\end{equation}
Several methods have been proposed to test this hypothesis, and most test statistics are expressed as differences in feature importance scores (e.g., differences in mean squared error in regression settings). For instance, \citet{horel2020significance} applied a Lindeberg–Feller–type central limit theorem for stochastic processes together with a second-order functional delta method to construct a test statistic for evaluating feature importance. Similarly, \citet{shen2021goodness}, \citet{shen2024exploration} and \citet{dai2022significance} defined feature importance scores based on differences in mean squared errors and established the asymptotic normality of the resulting test statistics. More recently, due to the wide variety of deep learning models and training algorithms, model-agnostic methods for testing feature significance have become increasingly popular. In particular, \citet{williamson2023general} proposed a general framework for model-agnostic variable importance, and \citet{gao2022lazy} built on this framework by incorporating the lazy training regime \citep{chizat2019lazy} to develop an efficient procedure for testing variable importance in deep neural networks.

Additionally, other applications of testing the hypothesis (\ref{Eq: Variable Importance Hypothesis}) using (deep) neural networks include a quasi-likelihood ratio test proposed by \citet{shen2022sieve} and a permutation test proposed by \citet{mandel2024permutation}.

\subsection{Our Contributions}
In this paper, we build upon the general variable importance framework to propose a new approach and algorithm for identifying important input features using lazy-trained deep neural network features in binary classification problems. The contributions of our work are threefold:
\begin{itemize}
    \item \textbf{Methodologically}, most existing methods focus on regression settings, whereas we develop our methodology under the binary classification setting. Moreover, by formulating the problem through the likelihood function, our framework naturally extends to other types of response variables (e.g., count responses) within the generalized linear model framework \citep{mccullagh2019generalized}.

    \item \textbf{Theoretically}, we improve the error rate from $\mathcal{O}_p(n^{-1/2})$ in \citet{gao2022lazy} to $o_p(n^{-1/2})$, which is essential for the validity of the general variable importance framework proposed by \citet{williamson2023general}. In addition, we substantially reduce the number of assumptions required for our theoretical results. Specifically, the only assumptions imposed concern the eigenvalue decay rate of the kernel matrix constructed from neural tangent features, the order of the regularization parameter in the penalized logistic regression problem, and the assumptions on the total number of weights in a deep ReLU neural network.

    \item \textbf{Computationally}, we implement the proposed algorithm as a Python library. Empirical studies demonstrate that our approach is more computationally efficient than retraining-based methods. Furthermore, simulation studies are conducted to empirically validate the assumed eigenvalue decay of the neural tangent kernel matrix and to evaluate the Type I error and power of the proposed test.
\end{itemize}

\subsection{Notations}
Throughout the rest of the paper, bold font alphabetic letters and Greek letters will be used to denote vectors or matrices. For a pseudo-metric space $(T,d)$, $N(\varepsilon, T, d)$ denotes the covering number, which is the minimum number of $\varepsilon$-balls needed to cover $T$ with respect to the metric $d$. For functions $f, g\in L_2(\mcal{X}, P)$ with $P$ being a probability measure, $\norm{\cdot}$ and $\inprod{\cdot}{\cdot}$ denote the $L_2$ norm and the inner product, i.e.,
\begin{align*}
	\norm{f}=\left[\int_{\mcal{X}}f^2(\mbf{x})\mrm{d}P(\mbf{x})\right]^{1/2}, \quad \inprod{f}{g}=\int_{\mcal{X}}f(\mbf{x})g(\mbf{x})\mrm{d}P(\mbf{x}).
\end{align*} 
Moreover, $\norm{\cdot}_n$ and $\inprod{\cdot}{\cdot}_n$ represent the $L_2$-norm and inner product with respect to the empirical probability measure $\mbb{P}_n$, i.e., if $\mbf{x}_1,\ldots,\mbf{x}_n$ are the observed data points,
\begin{align*}
	\norm{f}_n=\left[\frac{1}{n}\sum_{i=1}^nf^2(\mbf{x}_i)\right]^{1/2}, \quad \inprod{f}{g}_n=\frac{1}{n}\sum_{i=1}^nf(\mbf{x}_i)g(\mbf{x}_i).
\end{align*} 
For a matrix $\mbf{A}\in\mbb{R}^{m\times n}$, $[\mbf{A}]_{ij}$, $[\mbf{A}]_{i,:}$ and $[\mbf{A}]_{:,j}$ represent the $(i,j)$th component of $\mbf{A}$, the $i$th row of $\mbf{A}$, and the $j$th column of $\mbf{A}$ respectively. $\norm{\mbf{A}}_{op}$ denotes the operator norm: $\norm{\mbf{A}}_{op}=\sup_{\norm{\mbf{x}}=1}\norm{\mbf{Ax}}$, which is the same as the largest eigenvalue of $\mbf{A}^T\mbf{A}$; $\norm{\mbf{A}}_F$ denotes the Frobenius norm: $\norm{\mbf{A}}_F=\sqrt{\mrm{tr}(\mbf{A}^T\mbf{A})}$ and $\norm{\mbf{A}}_{2,1}$ denotes the $(2,1)$-norm of $\mbf{A}$, that is $\norm{\mbf{A}}_{2, 1}=\sum_{i=1}^n\norm{[\mbf{A}]_{:,i}}$, which is the sum of the $\ell_2$-norm of each column in $\mbf{A}$. Additionally, for a square matrix $\mbf{B}\in\mbb{R}^{n\times n}$, $\lambda_1(\mbf{B}), \ldots,\lambda_n(\mbf{B})$ will be used to denote its eigenvalues. For any $p$-dimensional vector $\mbf{u}$ and $S\subset [p]:=\{1, 2, \ldots, p\}$, we refer to the elements of $\mbf{u}$ with index in $S$ and not in $S$ as $\mbf{u}_S$ and $\mbf{u}_{-S}$, respectively.

In terms of asymptotic notations, suppose that $\{a_n\}, \{b_n\}$ are two sequences, we denote $a_n=o(b_n)$ if $\lim_{n\to\infty}\frac{a_n}{b_n}=0$ and $a_n=\omega(b_n)$ if $\lim_{n\to\infty}\frac{a_n}{b_n}=\infty$. In addition, $a_n=\mcal{O}(b_n)$ if $\abs{\frac{a_n}{b_n}}$ is bounded and $a_n\lesssim_\alpha b_n$ means $a_n\leq C(\alpha)b_n$, where $C(\alpha)$ is some constant depending on $\alpha$ only. Moreover, let $\{X_n\}$ and $\{Y_n\}$ be sequences of random variables, we denote $X_n=o_p(Y_n)$ if $\abs{X_n/Y_n}\xrightarrow{p}0$ and $X_n=\mcal{O}_p(Y_n)$ if $\abs{X_n/Y_n}$ is bounded in probability, i.e. $\sup_{n}\mbb{P}(\abs{X_n/Y_n}\geq k)\to0$ as $k\to\infty$.

\subsection{Organization of the Paper}
The rest of the paper is organized as follows. Section 2 provides some preliminary materials to provide readers with sufficient background in kernels and reproducing kernel Hilbert spaces, local Rademacher complexity, and the framework of variable importance. The main theoretical results and the lazy variable importance (VI) framework for binary classification are provided in Section 3, followed by some simulation and experimental results in Section 4. All the proofs are given in the appendices.

\section{Preliminaries}
\subsection{Kernels and Reproducing Kernel Hilbert Space}
Let $\mcal{X}$ be an arbitrary set. The idea of a kernel is to define a comparison function $K:\mcal{X}\times\mcal{X}\to\mbb{R}$ to measure the similarity of a pair of two inputs from $\mcal{X}$. In particular, positive definite kernels are the most widely used ones in statistics and machine learning.

\begin{definition}
    A positive definite kernel on a set $\mcal{X}$ is a function $K:\mcal{X}\times\mcal{X}\to\mbb{R}$ that is
    \begin{itemize}
        \item \textbf{Symmetric}: $K(x, y)=K(y,x)$ for all $x,y\in\mcal{X}$.
        \item \textbf{Positive Semidefinite}: for all $N\in\mbb{N}$, $(a_1,\ldots,a_N)\in\mbb{R}^N$ and $x_1,\ldots,x_N\in\mcal{X}$,
        $$
        \sum_{i=1}^n\sum_{j=1}^na_ia_jK(x_i, x_j)\geq0.
        $$
    \end{itemize}
\end{definition}

An important property of positive definite kernels is the reproducing property, which means that for a function $f$ in a reproducing kernel Hilbert space (RKHS) $(\mcal{H}, \inprod{\cdot}{\cdot}_{\mcal{H}})$, its evaluation at a point $x\in\mcal{X}$ can be represented as the inner product between $f$ and $K(\cdot,x)$.

\begin{definition}[Reproducing Kernel Hilbert Space]
    Let $\mcal{X}$ be a set and $\mcal{H}\subset\{h:\mcal{X}\to\mbb{R}\}$ be a class of functions forming a Hilbert space with inner product $\inprod{\cdot}{\cdot}_{\mcal{H}}$. The function $K:\mcal{X}\times\mcal{X}\to\mbb{R}$ is called a reproducing kernel of $\mcal{H}$ if
    \begin{itemize}
        \item $K(\cdot, x)$ is an element in $\mcal{H}$ for all $x\in\mcal{X}$.
        \item For every $x\in\mcal{X}$ and $h\in\mcal{H}$, the following reproducing property holds:
            $$
            f(x)=\inprod{f}{K(\cdot, x)}_{\mcal{H}}.
            $$
    \end{itemize}
    If a reproducing kernel exists, then $\mcal{H}$ is called a reproducing kernel Hilbert space (RKHS).
\end{definition}

The seminar paper by \citet{aronszajn1950theory} developed an important property that a kernel function is positive definite if and only if it is a reproducing kernel. Consequently, if $\mcal{H}$ is an RKHS associated with a kernel function $K$ and let $\phi:\mcal{X}\to\mcal{H}$ with $\phi(x)=K(\cdot, x)$, then for any $x,y\in\mcal{H}$,
$$
\inprod{\phi(x)}{\phi(y)}_{\mcal{H}}=\inprod{K(\cdot, x)}{K(\cdot,y)}_{\mcal{H}}=K(x,y).
$$
The map $\phi$ is commonly known as the feature map. In other words, the kernel function maps an element $x$ in an arbitrary set $\mcal{X}$ to the element $\phi(x)=K(\cdot, x)$ in a high dimensional feature space $\mcal{H}$ and the nonlinear input-output relationship could become a linear relationship in the high dimensional feature space $\mcal{H}$.

% Introduction of neural tangent kernel.

\subsection{Local Rademacher Complexity}\label{Sec: Local Rademacher Complexity}
Given a function class $\mathcal{F}$, the Rademacher complexity is a quantitative measure of its complexity, which essentially evaluates the alignment or correlation between the vector of predicted values and a vector of random noise. However, Rademacher complexity provides global estimates of the complexity of the function class and it does not reflect the fact that good learning algorithms often pick functions having small errors. As a result, suboptimal rates will be obtained in some situations \citep{bartlett2005local}. Local Rademacher complexities are similar to Rademacher complexities except that they restrict to a small subset of the function class. 

The main results rely heavily on some results related to local Rademacher complexity in \citet{bartlett2005local}. We summarize some key concepts and results in this subsection.

\begin{definition}[Star-shaped Class]
    Let $\mcal{F}$ be a class of functions and $f_0$ is a given function. $\mcal{F}$ is said to be star-shaped around $f_0$ if $f_0+\gamma(f-f_0)\in\mcal{F}$ for any $f\in\mcal{F}$ and $\gamma\in[0,1]$. The star hull of $\mcal{F}$ around $f_0$ is defined as
    $$
    \mrm{star}(\mcal{F}, f_0)=\{f_0+\gamma(f-f_0): \gamma\in[0,1], f\in\mcal{F}\}.
    $$
\end{definition}

\begin{definition}[Sub-root Function]
    A function $\psi:[0,\infty)\to[0,\infty)$ is sub-root if it is nonnegative, nondecreasing and if $r\mapsto\psi(r)/\sqrt{r}$ is nonincreasing for $r>0$.
\end{definition}

The following lemma shows that local Rademacher complexities are sub-root when the function class $\mcal{F}$ is star-shaped.
\begin{lemma}[Lemma 3.4 in \citet{bartlett2005local}]\label{Lm: localRad is sub-root}
    If the class $\mcal{F}$ is star-shaped around $\hat{f}$ (which may depend on the data),  then the (random) function $\psi$ defined for $r\geq0$ by
    $$
    \psi(r)=\mbb{E}_\xi\left[\left.\sup_{f\in\mcal{F}, \norm{f-\hat{f}}\leq \sqrt{r}}\frac{1}{n}\sum_{i=1}^n\xi_if(X_i)\right|X_1,\ldots, X_n\right],
    $$
    is sub-root and $r\mapsto\mbb{E}[\psi(r)]$ is also sub-root.
\end{lemma}

It is well-known in machine learning theory that for a uniformly bounded function class $\mcal{F}$, with high probability, $\mbb{E}[f(X)]$ can be upper bounded by the empirical mean $\frac{1}{n}\sum_{i=1}^nf(X_i)$ and the (empirical) Rademacher complexity of $\mcal{F}$ (Theorem 3.3 in \citet{mohri2018foundation}). The following theorem shows similar results in terms of local Rademacher complexity.
\begin{theorem}[Theorem 3.3 in \citet{bartlett2005local}]\label{Thm: Thm 3.3 in Bartlett et al 2005}
    Let $\mcal{F}$ be a class of functions with ranges in $[a,b]$ and assume that there are some functional $T:\mcal{F}\to\mbb{R}^+$ and some constant $B$ such that for every $f\in\mcal{F}$, $\mrm{Var}[f]\leq T(f)\leq D\mbb{E}[f]$. Let $\psi$ be a sub-root function and let $r^*$ be the fixed point of $\psi$, i.e. $\psi(r^*)=r^*$. Assume that $\psi$ satisfies for any $r\geq r^*$,
            $$
            D\mbb{E}\left[\sup_{f\in\mcal{F}, T(f)\leq r}\frac{1}{n}\sum_{i=1}^n\xi_if(X_i)\right]\leq\psi(r).
            $$
            Then with $c_1=704$ and $c_2=26$, for any $C>1$ and every $\delta>0$, with probability at least $1-\delta$,
            $$
            \mbb{E}[f(X)]\leq\frac{C}{C-1}\frac{1}{n}\sum_{i=1}^nf(X_i)+\frac{c_1C}{D}r^*+\frac{11(b-1)+c_2CD}{n}\log\frac{1}{\delta},\quad\forall f\in\mcal{F}
            $$
            Also, with probability at least $1-\delta$,
            $$
            \frac{1}{n}\sum_{i=1}^nf(X_i)\leq\frac{C+1}{C}\mbb{E}[f(X)]+\frac{c_1C}{D}r^*+\frac{11(b-a)+c_2CD}{n}\log\frac{1}{\delta},\quad\forall f\in\mcal{F}
            $$
\end{theorem}

\subsection{The Framework of Variable Importance}\label{Sec: VI Framework}
\citet{williamson2023general} proposed a general framework for nonparametric inference on interpretable algorithm-agnostic variable importance. Let $S\subset[p]$ be the index set of the features subgroup of interest, and let $\mcal{F}$ be a rich class of functions from $\mcal{X}\to\mbb{R}$ endowed with a norm $\norm{\cdot}_{\mcal{F}}$. Define
\begin{equation}\label{Eq: F-S VI Framework}
\mcal{F}_{-S}:=\left\{f\in\mcal{F}:f(\mbf{u})=f(\mbf{v})\mrm{ for all }\mbf{u}, \mbf{v}\in\mcal{X}\mrm{ satisfying }\mbf{u}_{-S}=\mbf{v}_{-S}\right\},
\end{equation}
to be the class of functions in $\mcal{F}$ whose evaluation ignores elements of the input with index in $S$. Additionally, suppose that $V(f, P_0)$ is a measure of predictiveness of a given candidate prediction function $f\in\mcal{F}$ when $P_0$ is the true data-generating distribution, with large values of $V(f, P_0)$ implying high predictiveness.  If $P_0$ is known, a natural candidate prediction function would be
\begin{equation}\label{Eq: f0 VI Framework}
f_0\in\mrm{argmax}_{f\in\mcal{F}}V(f,P_0).
\end{equation}
Similarly, we can define
\begin{equation}\label{Eq: f0,-S VI Framework}
f_{0,-S}\in\mrm{argmax}_{f\in\mcal{F}_{-S}}V(f, P_0).
\end{equation}
Then the population-level important of the variable $\mbf{X}_{S}$ relative to the full feature vector $\mbf{X}$ is defined as the amount of predictiveness lost by excluding $\mbf{X}_{S}$ from $\mbf{X}$:
$$
\psi_{0,S}:=V(f_0, P_0)-V(f_{0,-S}, P_0).
$$
Once the data are observed, a natural estimator of $\psi_{0,S}$ is 
$$
\hat{\psi}_{n,S}:=V(\hat{f}_n,\mbb{P}_n)-V(\hat{f}_{n,-S},\mbb{P}_n),
$$
where $\mbb{P}_n$ is the empirical probability distribution and $\hat{f}_n$, $\hat{f}_{n,-S}$ are estimators of population optimizers $f_0$ and $f_{0,-S}$ respectively and are often obtained by building the predictive model for $Y$ using all features $\mbf{X}$ or only those features in $\mbf{X}_{-S}$ respectively.

Let $\mcal{M}$ be a class of probability distributions and define the vector space of finite signed measures generated by $\mcal{M}$ as
$$
\mcal{S}=\{c(P_1-P_2): c\geq0, P_1, P_2\in\mcal{M}\}.
$$
For any $\nu=c(P_1-P_2)\in S$, let $\norm{\nu}_\infty=c\sup_{x}\abs{F_1(z)-F_2(z)}$ where $F_1$ and $F_2$ are the distribution functions with respect to $P_1$ and $P_2$ respectively. The main result in \citet{williamson2023general} shows that $\hat{\mbf{\psi}}_{n,S}$ is asymptotically normal under two sets of conditions, which can be classified as deterministic (D) and random (R) conditions in nature. 

\begin{itemize}
    \item [(D1)] (\textbf{Optimality}) There exists some constant $C>0$ such that for each sequence $f_1, f_2,\ldots, \in\mcal{F}$ such that $\norm{f_j-f_0}_{\mcal{F}}\to0$, $\abs{V(f_j, P_0)-V(f_0, P_))}\leq C\norm{f_j-f_0}_{\mcal{F}}^2$ for each $j$ large enough.

    \item[(D2)] (\textbf{Differentiability}) There exists some constant $\delta>0$ such that for each sequence $\eta_1,\eta_2,\cdots\in\mbb{R}$ and $H, H_1, H_2,\cdots\in\mcal{S}$ satisfying that $\eta_j\to0$ and $\norm{H-H_j}_\infty\to0$, it holds that
    $$
    \sup_{f\in\mcal{F}, \norm{f-f_0}_{\mcal{F}}<\delta}\abs{\frac{V(f, P_0+\eta_jH_j)-V(f, P_0)}{\eta_j}-\dot{V}(f, P_0;\eta_j)}\to0,
    $$
    where $\dot{V}(f, P_0; H)$ is the G\^ateaux derivative of $P\mapsto V(f,P)$ at $P_0$ along the direction $H\in\mcal{S}$.

    \item [(R1)] (\textbf{Minimum Rate of Convergence}) $\norm{\hat{f}_n-f_0}_{\mcal{F}}=o_p(n^{-1/4})$.

    \item [(R2)] (\textbf{Weak Consistency}) $\int[g_n(z)]^2dP_0(z)=o_p(1)$, where $g_n:z\mapsto\dot{V}(\hat{f}_n, P_0;\delta_z-P_0)-\dot{V}(f_0, P_0;\delta_z-P_0)$ with $\delta_z$ being the degenerate distribution on $\{z\}$.

    \item [(R3)] (\textbf{Limited Complexity}) There exists some $P_0$-Donsker class $\mcal{D}_0$ such that $P_0(g_n\in\mcal{D}_0)\to1$.
\end{itemize}

\begin{theorem}[Theorem 1 in \citet{williamson2023general}]\label{Thm: Thm 1 in williamson}
    If (D1)-(D2) and (R1)-(R3) hold, then $V(\hat{f}_n,\mbb{P}_n)$ is an asymptotically linear estimator of $V(f_0, P_0)$, that is,
    \begin{equation}\label{Eq: asymp linear estimator}
        V(\hat{f}_n, \mbb{P}_n)-V(f_0, P_0)=\frac{1}{n}\sum_{i=1}^n\dot{V}(f_0, P_0;\delta_{Z_i}-P_0)+o_p(n^{-1/2}).
    \end{equation}
    under sampling from $P_0$.
\end{theorem}

\section{Main Results}
\subsection{Problem Setup}\label{Sec: Problem Setup}
Suppose that $(\mbf{X},Y), (\mbf{X}_1, Y_1), \ldots, (\mbf{X}_n, Y_n)\sim\mrm{ i.i.d. }P_0$, where $\mbf{X}\in\mbb{R}^p$ satisfying $\norm{\mbf{X}}\leq\kappa$ for some $\kappa>0$, $Y\in\{0, 1\}$ and $P_0$ is some probability distribution. In classical statistical methods, logistic regression is the most widely used approach to model dichotomous response variables. Recall that in a logistic regression, the assumption on the conditional distribution of $Y|\mbf{X}$ is $P_{Y|\mbf{X}}=\mrm{Ber}\left(\sigma(\mbf{X}^T\mbf{\beta})\right)$. In other words, the logit of the conditional mean $\mbb{E}[Y|\mbf{X}]$ is a linear function with respect to the input variables. A natural way to generalize the linear assumption is to assume a more general function. In other words, it is reasonable to assume more generally that $P_{Y|\mbf{X}}=\mrm{Ber}(\pi(\mbf{X}))$. 

Now, let $\mcal{F}$ be a class of functions and
\begin{equation}\label{Eq: Def of f0}
f_0\in\mrm{argmin}_{f\in\mcal{F}}\mbb{E}_{P_0}\left[\ell(f(\mbf{X}), Y)\right],
\end{equation}
where
\begin{equation}\label{Eq: logistic loss}
    \ell(f(\mbf{X}), Y)=-Yf(\mbf{X})+\log\left(1+e^{f(\mbf{X})}\right).
\end{equation}
In other words, the function $V(f_, P_0)$ measuring the predictive performance in the VI framework is $V(f,P_0)=\mbb{E}_{P_0}\left[Yf(\mbf{X})-\log\left(1+e^{f(\mbf{X})}\right)\right]$. Our first observation, as shown in Proposition \ref{Prop: About f0}, is that the underlying function $\pi(\mbf{X})$ that generates the response variable $Y$ is the same as $\sigma(f_0(\mbf{X}))$ almost surely.

\begin{proposition}\label{Prop: About f0}
    Suppose that $P_{Y|\mbf{X}}=\mrm{Ber}(\pi(\mbf{X}))$ for some function $\pi(\mbf{X})$ satisfying $\mrm{logit}(\pi(\mbf{X}))\in\mcal{F}$ and $f_0$ is as defined in (\ref{Eq: Def of f0}) Then
    $$
    \pi(\mbf{X})=\sigma(f_0(\mbf{X})),\quad P_{\mbf{X}}-a.s.
    $$
\end{proposition}

Similar to Section \ref{Sec: VI Framework}, Let $S\subset[p]$ be the set of indices of variables to be tested for importance. Define
$$
f_{0, -S}\in\mrm{argmin}_{f\in\mcal{F}_{-S}}\mbb{E}_{P_0}\left[-Yf(\mbf{X})+\log\left(1+e^{f(\mbf{X})}\right)\right],
$$
where $\mcal{F}_{-S}$ is the same as defined in (\ref{Eq: F-S VI Framework}). Throughout the remaining of the paper, we consider the class of deep ReLU neural networks as in \citet{bartlett2017spectrally}
\begin{equation}\label{Eq: Class of DNN}
    \mcal{F}=\left\{h_{\mbf{\theta}_f}(\mbf{x})=\mbf{W}_L\varphi(\mbf{W}_{L-1}\varphi(\cdots\mbf{W}_2\varphi(\mbf{W}_1\mbf{x}))): \norm{\mbf{W}_i}_{op}\leq \kappa_i, \norm{\mbf{W}_i^T}_{2,1}\leq b_i,  \forall i\in[L]\right\},
\end{equation}
where $\varphi:\mbb{R}\to\mbb{R}$ is an activation function and $\mbf{W}_i\in\mbb{R}^{p_i\times p_{i-1}}$, $i=1,\ldots,L$ are weight matrices in the network with $p_0=p$ and $p_L=1$. In particular, we consider $\varphi$ to be the most popular rectified linear unit (ReLU) activation function \citep{nair2010rectified}. In addition, write $\mbf{W}_i=[\mbf{w}_{i,1}, \ldots, \mbf{w}_{i,p_{i-1}}]$. Then note that
\begin{align*}
        \norm{\mbf{W}_i}_{op} & =\sup_{\norm{\mbf{u}}= 1}\norm{\mbf{W}_i\mbf{u}}=\sup_{\norm{\mbf{u}}=1}\norm{\sum_{j=1}^{p_{i-1}}u_j\mbf{w}_{i,j}}\\
            & \leq\sup_{\norm{\mbf{u}}=1}\sum_{j=1}^{p_{i-1}}\abs{u_j}\norm{\mbf{w}_{i,j}}\\
            & \leq\sum_{j=1}^{p_{i-1}}\norm{\mbf{w}_{i,j}}=\norm{\mbf{W}_i^T}_{2,1}.
\end{align*}
So it is reasonable to assume $b_i\geq\kappa_i$, for all $i\in[L]$, and we will implicitly make this assumption in the sections to follow. 

Let $\mbf{\theta}\in\mbb{R}^W$ be the vector of all parameters ($W$ is the total number of parameters) in a deep neural network in $\mcal{F}$. The estimator of $f_0$ is given by the empirical risk minimizer:
\begin{equation}\label{Eq: Definition of theta_f}
\mbf{\theta}_f=\mrm{argmin}_{\mbf{\theta}\in\mbb{R}^W}\frac{1}{n}\sum_{i=1}^n\left[-Y_i\log h_{\mbf{\theta}}(\mbf{X}_i)+\log\left(1+e^{h_{\mbf{\theta}}(\mbf{X}_i)}\right)\right]
\end{equation}
Denote $\mbf{X}_{-S}$ be the vector of $\mbf{X}$ with the indices in $S$ replaced by their corresponding mean. The estimator of $f_{0, -S}$ can similarly by defined as
\begin{equation}\label{Eq: Definition of theta_f, S}
\mbf{\theta}_{f,-S}=\mrm{argmin}_{\mbf{\theta}\in\mbb{R}^W}\frac{1}{n}\sum_{i=1}^n\left[-Y_ih_{\mbf{\theta}}(\mbf{X}_{i,-S})+\log\left(1+e^{h_{\mbf{\theta}}(\mbf{X}_{i,-S})}\right)\right].
\end{equation}
To address the computational complexity of retraining a neural network, we considered the lazy regime by linearly approximating $h_\theta$ around $\mbf{\theta}_f$:
$$
h_{\mbf{\theta}} \approx h_{\mbf{\theta}_f}+[\nabla_{\mbf{\theta}} h_{\mbf{\theta}_f}]^T (\mbf{\theta}-\mbf{\theta}_f).
$$
Under such a framework, the estimator of $f_{0,-S}$ will be  $\tilde{h}_{\mbf{\theta}_f+\Delta\mbf{\theta}_S}:=h_{\mbf{\theta}_f}+[\nabla_{\mbf{\theta}} h_{\mbf{\theta}_f}]^T\Delta\mbf{\theta}_S$, where
\begin{align*}
    \Delta\mbf{\theta}_S=\mrm{argmin}_{\mbf{w}\in\mbb{R}^W}\frac{1}{n}\sum_{i=1}^n & \left[-Y_i (h_{\mbf{\theta}_f}(\mbf{X}_{i,-S})+[\nabla_\theta h_{\mbf{\theta}_f}(\mbf{X}_{i, -S})]^T\mbf{w})\right.\\
    \qquad & \left.+\log\left(1+e^{h_{\mbf{\theta}_f}(\mbf{X}_{i,-S})+[\nabla_\theta h_{\mbf{\theta}_f}(\mbf{X}_{i, -S})]^T\mbf{w})}\right)\right]+\frac{\lambda}{2}\norm{\mbf{w}}^2\numberthis\label{Eq: logistic regression lazy regime}
\end{align*}
The optimization problem (\ref{Eq: logistic regression lazy regime}) now becomes a logistic regression with a ridge penalty, and $\lambda>0$ is the regularization parameter. Let $L_\lambda(\mbf{w})$ be the target function. Then
\begin{align*}
    \frac{\partial}{\partial \mbf{w}}L_\lambda(\mbf{w}) & =-\frac{1}{n}\mbf{\Phi}_{-S}^T\left[\mbf{Y}-\sigma\left(\mbf{h}_{\mbf{\theta}_f, -S}+\mbf{\Phi}_{-S}^T\mbf{w}\right)\right]+\lambda\mbf{w}\\
    \frac{\partial^2}{\partial\mbf{w}\partial\mbf{w}^T}L_\lambda(\mbf{w}) & =\frac{1}{n}\mbf{\Phi}_{-S}^T\mbf{\Pi}_{-S}\mbf{\Phi}_{-S}+\lambda\mbf{I}_W,
\end{align*}
where
\begin{align*}
    \mbf{\Phi}_{-S}^T & =\begin{bmatrix} \nabla_{\mbf{\theta}}h_{\mbf{\theta}_f}(\mbf{X}_{1, -S}), \cdots, \nabla_{\mbf{\theta}}h_{\mbf{\theta}_f}(\mbf{X}_{n, -S}) \end{bmatrix}\in\mbb{R}^{W\times n},\\
    \mbf{\Pi}_{-S} & =\mrm{Diag}\left(
        \sigma\left(\mbf{h}_{\mbf{\theta}_f, -S}+\mbf{\Phi}_{-S}\mbf{w}\right)\left(1-\sigma\left(\mbf{h}_{\mbf{\theta}_f, -S}+\mbf{\Phi}_{-S}\mbf{w}\right)\right)
    \right),\\
    \mbf{h}_{\mbf{\theta}_f, -S} & =\begin{bmatrix}
        h_{\mbf{\theta}_f}(\mbf{X}_{1,-S}),\cdots, h_{\mbf{\theta}_f}(\mbf{X}_{n, -S})
    \end{bmatrix}^T.
\end{align*}
It is easy to see that $L_\lambda(\mbf{w})$ is a convex function in $\mbf{w}$. Because of this, the optimizer can be  obtained through the Newton-Raphson algorithm and the updating equation for $\mbf{w}$ in the Newton-Raphson algorithm is
\begin{align*}
    \mbf{w}^{(k+1)} & =\mbf{w}^{(k)}+\left[\mbf{\Phi}_{-S}^T\mbf{\Pi}_{-S}^{(k)}\mbf{\Phi}_{-S}+n\lambda\mbf{I}_W\right]^{-1}\left[\mbf{\Phi}_{-S}^T\left[\mbf{Y}-\sigma\left(\mbf{h}_{\mbf{\theta}_f, -S}+\mbf{\Phi}_{-S}^T\mbf{w}^{(k)}\right)\right]-n\lambda\mbf{w}^{(k)}\right]\\
        & =\left[\mbf{\Phi}_{-S}^T\mbf{\Pi}_{-S}^{(k)}\mbf{\Phi}_{-S}+n\lambda\mbf{I}_W\right]^{-1}\left[\mbf{\Phi}_{-S}^T\mbf{\Pi}_{-S}^{(k)}\mbf{\Phi}_{-S}\mbf{w}^{(k)}+\mbf{\Phi}_{-S}^T\left[\mbf{Y}-\sigma\left(\mbf{h}_{\mbf{\theta}_f, -S}+\mbf{\Phi}_{-S}^T\mbf{w}^{(k)}\right)\right]
        \right]\\
        & =\left[\mbf{\Phi}_{-S}^T\mbf{\Pi}_{-S}^{(k)}\mbf{\Phi}_{-S}+n\lambda\mbf{I}_W\right]^{-1}\mbf{\Phi}_{-S}^T\mbf{\Pi}_{-S}^{(k)}\left[\mbf{\Phi}_{-S}\mbf{w}^{(k)}+\mbf{\Pi}_{-S}^{(k)^{-1}}\left[\mbf{Y}-\sigma\left(\mbf{h}_{\mbf{\theta}_f, -S}+\mbf{\Phi}_{-S}^T\mbf{w}^{(k)}\right)\right]\right]\\
        & =\mbf{\Phi}_{-S}^T\mbf{\Pi}_{-S}^{(k)}\left[\mbf{\Phi}_{-S}\mbf{\Phi}_{-S}^T\mbf{\Pi}_{-S}^{(k)}+n\lambda\mbf{I}_n\right]^{-1}\left[\mbf{\Phi}_{-S}\mbf{w}^{(k)}+\mbf{\Pi}_{-S}^{(k)^{-1}}\left[\mbf{Y}-\sigma\left(\mbf{h}_{\mbf{\theta}_f, -S}+\mbf{\Phi}_{-S}^T\mbf{w}^{(k)}\right)\right]\right],
\end{align*}
where the last equation follows from the Sherman-Morrison-Woodbury identity. Consequently,
\begin{align*}
    \tilde{\mbf{h}}_{\mbf{\theta}_f+\Delta\mbf{\theta}_S}^{(k+1)} & \coloneq \mbf{h}_{\mbf{\theta}_f,-S}+\mbf{\Phi}_{-S}\mbf{w}^{(k+1)}\\
        & =\mbf{h}_{\mbf{\theta}_f,-S}+\mbf{\Phi}_{-S}\mbf{\Phi}_{-S}^T\mbf{\Pi}_{-S}^{(k)}\left[\mbf{\Phi}_{-S}\mbf{\Phi}_{-S}^T\mbf{\Pi}_{-S}^{(k)}+n\lambda\mbf{I}_n\right]^{-1}\\
        & \hspace*{5cm}\left[\mbf{\Phi}_{-S}\mbf{w}^{(k)}+\mbf{\Pi}_{-S}^{(k)^{-1}}\left[\mbf{Y}-\sigma\left(\mbf{h}_{\mbf{\theta}_f, -S}+\mbf{\Phi}_{-S}^T\mbf{w}^{(k)}\right)\right]\right]\\
        & =\mbf{h}_{\mbf{\theta}_f,-S}+\mbf{K}_{-S}\mbf{\Pi}_{-S}^{(k)}\left[\mbf{K}_{-S}\mbf{\Pi}_{-S}^{(k)}+n\lambda\mbf{I}_n\right]^{-1}\left[\mbf{\Phi}_{-S}\mbf{w}^{(k)}+\mbf{\Pi}_{-S}^{(k)^{-1}}\left[\mbf{Y}-\sigma\left(\mbf{h}_{\mbf{\theta}_f, -S}+\mbf{\Phi}_{-S}^T\mbf{w}^{(k)}\right)\right]\right]\\
        & =\mbf{h}_{\mbf{\theta}_f,-S}+\mbf{K}_{-S}\left[\mbf{K}_{-S}+n\lambda\mbf{\Pi}_{-S}^{(k)^{-1}}\right]^{-1}\left[\mbf{\Phi}_{-S}\mbf{w}^{(k)}+\mbf{\Pi}_{-S}^{(k)^{-1}}\left[\mbf{Y}-\sigma\left(\mbf{h}_{\mbf{\theta}_f, -S}+\mbf{\Phi}_{-S}^T\mbf{w}^{(k)}\right)\right]\right],
\end{align*}
where $\mbf{K}_{-S}:=\mbf{\Phi}_{-S}\mbf{\Phi}_{-S}^T$ is the neural tangent kernel (NTK) matrix \citep{jacot2018neural}.

% Overview of the main result.
To provide a systematic overview of the theoretical results, we summarize the main ideas in Figure \ref{Fig: Summary of the main result}. In the figure, the red point represents the underlying function $f_0$; the blue point represents the fitted deep ReLU neural network $h_{\mbf{\theta}f}$ using all input features; the green point represents $\tilde{h}_{\mbf{\theta}_f+\Delta\mbf{\theta}_S}$, where $\Delta\mbf{\theta}_S$ is obtained by fitting a logistic regression model using the neural tangent kernel (NTK); and the black point represents $h_{\mbf{\theta}_f+\Delta\mbf{\theta}_S}$, which is a deep ReLU neural network with weights $\mbf{\theta}_f+\Delta\mbf{\theta}S$. The latter model is used to evaluate predictive performance over $\mathcal{F}_{-S}$. There are three quantities of primary interest.
(I) The approximation error $\norm{f_0 - h_{\mbf{\theta}_f}}$, which measures how well the deep ReLU neural network fits the underlying function $f_0$. The convergence rate of the DNN within the class $\mathcal{F}$ is studied in Section~\ref{Sec: Convergence Rate of DNN}, which is depicted in the left panel of Figure \ref{Fig: Summary of the main result}.
(II) The distance between $\tilde{h}_{\mbf{\theta}_f+\Delta\mbf{\theta}S}$ and $h_{\mbf{\theta}_f+\Delta\mbf{\theta}_S}$, which quantifies the discrepancy between the lazy VI–fitted linearized deep ReLU neural network around the fitted parameters—an element of the RKHS induced by the NTK—and the corresponding deep ReLU neural network in $\mathcal{F}$ with parameters $\mbf{\theta}_f+\Delta\mbf{\theta}S$. This quantity is analyzed in Section \ref{Sec: Bounding the difference}.
(III) The estimation error of the lazy VI–fitted linearized deep ReLU neural network, given by $\norm{\tilde{h}_{\mbf{\theta}_f+\Delta\mbf{\theta}_S} - f_0}$, which is investigated in Section~\ref{Sec: Estimation Error of h tilde}. (II) and (III) will be combined to determine the order of $\norm{h_{\mbf{\theta}_f+\Delta\mbf{\theta}_S}-f_0}$.

\begin{figure}[htbp]
    \centering
    \includegraphics[width=0.8\linewidth]{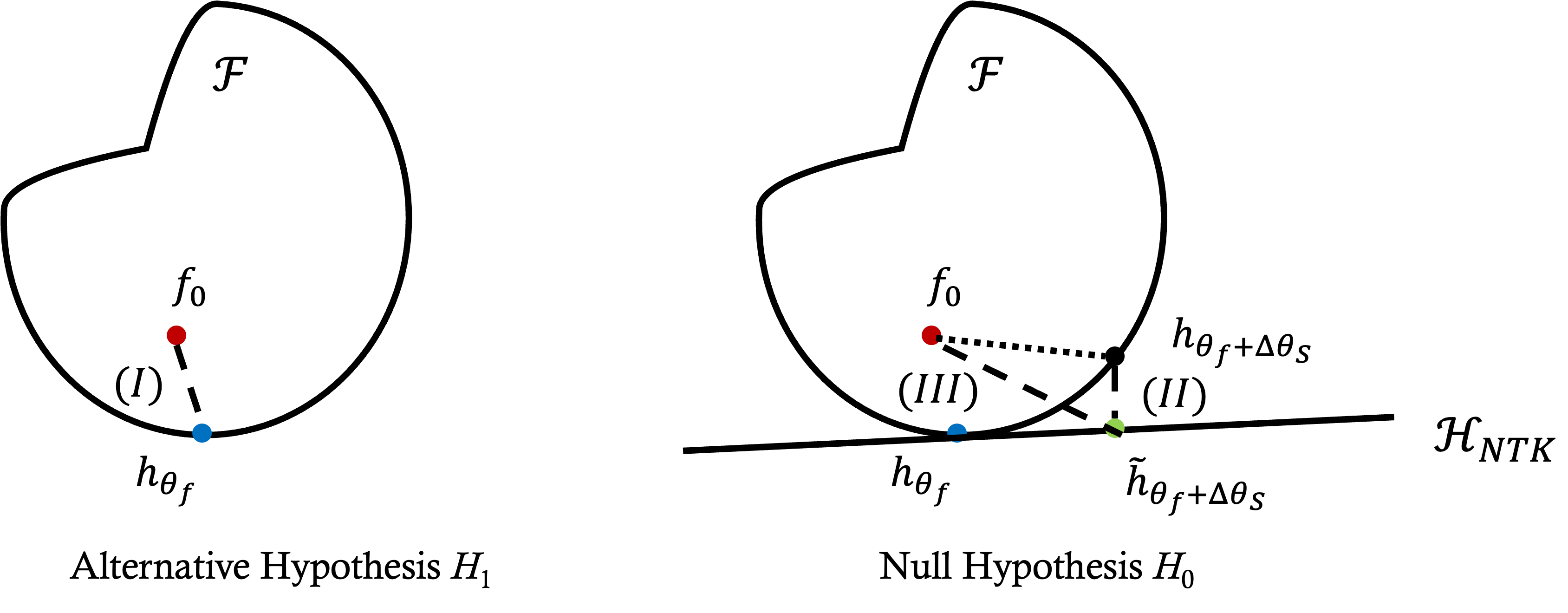}
    \caption{Graphical illustration of the three essential quantities in the main result. (I)$=\norm{f_0-h_{\mbf{\theta}_f}}$. (II)$=\norm{h_{\mbf{\theta}_f+\Delta\mbf{\theta}_S}-\tilde{h}_{\mbf{\theta}_f+\Delta\mbf{\theta}_S}}$. (III)$=\norm{h_{\tilde{h}_{\mbf{\theta}_f+\Delta\mbf{\theta}_S}}-f_0}$.}
    \label{Fig: Summary of the main result}
\end{figure}

% Empirical Norm and L2 Norm
To bound the quantities shown in Figure \ref{Fig: Summary of the main result}, we will frequently use the relationship between the $L_2(P)$ norm and the empirical $L_2(\mbb{P}_n)$ norm. So we first state a theorem that demonstrates their relationships. In short, with high probability, either norm can be bounded by the other with some additional factor depending on the complexity of the function class $\mcal{F}$.

\begin{theorem}\label{Thm: Empirical Norm and L2 Norm}
    Let $\mcal{F}$ be a class of functions with ranges in $[-M, M]$ and $f^*$ be a fixed function (not necessarily in $\mcal{F}$) with bounded range. Let $\tilde{M}=\sup_{f\in\mcal{F}}\sup_{\mbf{x}}\abs{f(\mbf{x})}+\sup_{\mbf{x}}\abs{f^*(\mbf{x})}$. In addition, suppose that $\hat{\psi}_n(r)$ is a sub-root function (possibly data-dependent) and let $\hat{r}^*$ be the fixed point of $\hat{\psi}_n(r)$ (i.e. $\hat{\psi}_n(r^*)=r^*$). Fix $\delta>0$ and assume that $\hat{\psi}_n$ satisfies for any $r\geq\hat{r}^*$,
    \begin{equation}\label{Eq: psi_n_hat}
        \hat{\psi}_n(r)\geq 20\mbb{E}_\xi\left[\left.\sup_{f\in\mrm{star}(\mcal{F},f^*), \norm{f-f^*}_n^2\leq2\tilde{M}^2r}\frac{1}{n}\sum_{i=1}^n\xi_i\frac{f(\mbf{X}_{i})}{\tilde{M}}\right|\mbf{X}_1,\ldots, \mbf{X}_n\right]+\frac{62}{3n}\log\frac{1}{\delta}.
    \end{equation}
    Then with probability at least $1-3\delta$,
    \begin{align*}
        \norm{f-f^*} & \leq \sqrt{2}\norm{f-f^*}_n+\sqrt{2\tilde{M}^2c_1\hat{r}^*}+\sqrt{\frac{\tilde{M}^2(11+2c_2)}{n}\log\frac{1}{\delta}},\quad\forall f\in\mcal{F},\numberthis\label{Eq: Upper Bound of L2 Norm}\\
        \norm{f-f^*}_n & \leq\sqrt{\frac{3}{2}}\norm{f-f^*}+\sqrt{2\tilde{M}^2c_1\hat{r}^*}+\sqrt{\frac{\tilde{M}^2(11+2c_2)}{n}\log\frac{1}{\delta}},\quad\forall f\in\mcal{F},\numberthis\label{Eq: Upper Bound of Empricial Norm}
    \end{align*}
    where $c_1$ and $c_2$ are universal constants.
\end{theorem}

\subsection{Assumptions}
\begin{itemize}
    \item[(A1)] \textbf{(Regularity Conditions on Kernel Functions)} Let $\mu_1\geq\mu_2\geq\cdots\geq\mu_n>0$ be the eigenvalues of $\frac{1}{n}\mbf{K}_{-S}$. There exists some constant $N_0>1$ such that $\mu_j\leq cj^{-\alpha}$ for all $j\geq N_0$ and some $\alpha>1$.
        %\begin{enumerate}
        %    \item $0<\sup_{x}K_{-S}(x,x)<\infty$.
        %    \item Let $\mu_1\geq\mu_2\geq\cdots\geq\mu_n>0$ be the eigenvalues of $\frac{1}{n}\mbf{K}_{-S}$. There exists some constant $N_0>1$ such that $\mu_j\leq cj^{-\alpha}$ for all $j\geq N_0$ and some $\alpha>1$.
           % \item $\mrm{tr}(\mbf{K}_{-S})=\mcal{O}_p(n)$
        %\end{enumerate}

    \item[(A2)] \textbf{(Regularity Conditions on Regularization Parameter $\lambda$)}
            %$\lambda=\left(1\vee\frac{\mu_1}{4}\right)[1-\varsigma_n]^{-1}$, where $\{\varsigma_n\}$ is a sequence such that $0<\varsigma_n<1$ and $\varsigma_n=\omega(n^{-\vartheta})$. 
            $\lambda=\left(1\vee\frac{\norm{\mbf{K}_{-S}}_{op}}{4}\right)\varsigma_n$ where $\varsigma_n=\omega\left(n^{-\frac{\alpha-1}{2(\alpha+1)}}\vee n^{-\frac{1}{4}}\right)$ is a deterministic sequence and $\alpha$ is the same as in Assumption (A1).

    %\item[(A3)] \textbf{(Regularity Conditions on Neural Network Estimators)} 
    %$\norm{\mbf{Y}-\sigma(\mbf{h}_{\mbf{\theta}_f,-S})}=\mcal{O}_p(1)$.
    %$\norm{\mbf{e}_{-S}}=\mcal{O}_p(n^{\frac{3-\alpha}{4(\alpha+1)}})$, where $\mbf{e}_{-S}=\mbf{h}_{\mbf{\theta}_f, -S} - \mbf{f}_{0, -S}$ and $\alpha$ is the same from Assumption (A1).
\end{itemize}

% Discussion on these assumptions
In \citet{gao2022lazy}, the Rademacher complexity of the RKHS spanned by neural tangent features is used to derive the convergence rate, which, from our perspective, is one reason why only an error rate of order $\mathcal{O}_p(n^{-1/2})$ is obtained. As discussed in Section \ref{Sec: Local Rademacher Complexity}, local Rademacher complexity can, in general, yield faster rate of convergence. As shown in Lemma \ref{Lm: Local Rad Comp of HB} in Section \ref{Sec: Estimation Error of h tilde}, the local Rademacher complexity depends on the eigenvalue decay rate of the NTK. Accordingly, Assumption (A1) is imposed to ensure that the NTK exhibits the desired spectral behavior, which enables a faster convergence rate.

We emphasize that Assumption (A1) is not overly restrictive. In fact, prior work has established that the eigenvalues of the NTK decay at a rate of $j^{-(p+1)/p}$ \citep{bietti2019inductive, bietti2020deep, li2024eigenvalue}. Existing results on eigenvalue decay are derived for the NTK evaluated at initialization, whereas in our setting, the NTK is evaluated at the fitted parameters. Intuitively, for large overparameterized neural networks, the parameters obtained via gradient descent remain close to their initializations \citep{du2018gradient, oymak2020toward}. Consequently, it is reasonable to expect that the NTK evaluated at the fitted parameters exhibits a similar eigenvalue decay behavior. Although we are currently unable to provide a theoretical proof that Assumption (A1) holds for the class of deep ReLU neural networks, we empirically assess this assumption through simulation studies in Section \ref{Sec: Decay Rate of NTK eigenvalues}.

Assumption~(A2), on the other hand, requires that the regularization parameter in the penalized logistic regression lies within an appropriate range to ensure the accuracy of the linear approximation to the fitted deep ReLU neural network.

\subsection{Convergence Rate of $h_{\mbf{\theta}_f}$}\label{Sec: Convergence Rate of DNN}
To begin with, we need to ensure that the neural network classifiers perform well by providing an upper bound for $\norm{h_{\mbf{\theta}_f}-f_0}$. For binary classification problems, there are many generalization bounds available for the loss 0-1 $\ell(a,y)=\mbb{I}_{\{a\neq y\}}$ or the logistic loss for large-margin classifiers $\ell(z)=\log(1+e^z)$ where $z=yf(\mbf{x})$ is known as the margin of a classifier taking values $\pm1$. Although the 0-1 logistic loss function defined in (\ref{Eq: logistic loss}) can be considered as a shifted version of the logistic loss, for completeness, we provide a detailed derivation of the generalization error bound based on the 0-1 logistic loss. Here are some basic facts about the 0-1 logistic loss function.

\begin{proposition}\label{Prop: Properties of the 0-1 logistic loss}
	Let $\ell(a,y)=-ay+\log(1+e^a)$ be the 0-1 logistic loss function. Then
	\begin{enumerate}
		\item $\ell$ is a convex function with respect to $a$.
		\item $\ell$ is a 1-Lipschitz function with respect to $a$, that is,
			$$
			\abs{\ell(a_1, y)-\ell(a_2,y)}\leq\abs{a_1-a_2},\quad \forall a_1, a_2\in\mbb{R}.
			$$
		\item Suppose that $a\in[-M, M]$ for some $M>0$, then a lower bound for the modulus of convexity for $\ell$ with respect to $a$ is given by
			\begin{equation}\label{Eq: Modulus of Convexity}
			\delta(\eta):=\inf_{\abs{a_1-a_2}\geq\eta}\frac{\ell(a_1, y)+\ell(a_2,y)}{2}-\ell\left(\frac{a_1+a_2}{2}, y\right)\geq\frac{1}{8}\sigma(M)[1-\sigma(M)]\eta^2.
			\end{equation}
	\end{enumerate}
\end{proposition}

As a result of the modulus of convexity of the 0-1 logistic function given in Proposition \ref{Prop: Properties of the 0-1 logistic loss}, the $L_2$ metric between any $f\in\mcal{F}$ and $f_0$ can be upper bounded via the risk function $R(f)=\mbb{E}[\ell(f(\mbf{X}), Y)]$ as demonstrated in the corollary below.

\begin{corollary}\label{Cor: Upper Bound L2 Metric by Risk}
	Suppose that $\sup_{\mbf{x}}\abs{f(\mbf{x})}\leq M$ for all $f\in\mcal{F}$, then
	$$
	\norm{f-f_0}^2\leq\frac{4(1+e^M)^2}{e^M}\left[R(f)-R(f_0)\right].
	$$
\end{corollary}

In view of Corollary \ref{Cor: Upper Bound L2 Metric by Risk}, to provide an upper bound for $\norm{h_{\mbf{\theta}_f}-f_0}$, it suffices to provide a good bound for $R(h_{\mbf{\theta}_f})-R(f_0)$. As have been mentioned in Section \ref{Sec: Local Rademacher Complexity}, local Rademacher complexity usually provides a better convergence rate for an estimator. So we utilized local Rademacher complexity and some techniques described in \citet{bartlett2005local} to provide upper bound for $R(h_{\mbf{\theta}_f})-R(f_0)$. The following theorem is a simple generalization to relax the assumption on the range of functions in Theorem 5.4 in \citet{bartlett2005local}.

\begin{theorem}\label{Thm: Extension of Bartlett Thm 5.4}
    Let $\mcal{F}$ be a class of functions with ranges in $[-M, M]$ and let $\ell(\cdot, \cdot)$ be a loss function satisfying the following conditions:
	\begin{enumerate}
		\item For every probability distribution $P$, there is an $f^*\in\mcal{F}$ satisfying $R(f^*)=\inf_{f\in\mcal{F}}R(f)$.
		\item There is a constant $L$ such that $\ell$ is $L$-Lipschitz in its first argument, that is, for all $y, \hat{y}_1, \hat{y}_2$,
			$$
			\abs{\ell(\hat{y}_1, y)-\ell(\hat{y}_2, y)}\leq L\abs{\hat{y}_1-\hat{y}_2}.
			$$
		\item There is a constant $B^*\geq1$ such that for every probability distribution and every $f\in\mcal{F}$,
			$$
			\norm{f-f^*}^2\leq B^*[R(f)-R(f^*)].
			$$
	\end{enumerate}
        Let $\hat{f}$ be any element of $\mcal{F}$ satisfying $R_n(\hat{f})=\inf_{f\in\mcal{F}}R_n(f)$. For any $\delta>0$, suppose that $\hat{\psi}_n(r)$ is a subroot function satisfying
        $$
        \hat{\psi}_n(r)\geq\tilde{c}_1\mbb{E}_\xi\left[\sup_{f\in\mrm{star}(\mcal{F},f^*),\norm{\hat{f}-f}_n^2\leq\tilde{c}_3r}\frac{1}{n}\sum_{i=1}^n\xi_if(\mbf{X}_i)\right]+\frac{\tilde{c}_2}{n}\log\frac{1}{\delta},
        $$
        where $\tilde{c}_1=40ML^2$, $\tilde{c}_2=44M^2L^2+2M\tilde{c}_1$ and $\tilde{c}_3=2/L^2$. Then for any $\delta>0$, with probability at least $1-3\delta$,
        $$
        R(\hat{f})-R(f^*)\leq\frac{c_1C}{B^*L^2}\hat{r}^*+\frac{11\bar{U}+c_2B^*L^2C}{n}\log\frac{1}{\delta},
        $$
        where $\bar{U}=\sup_{f\in\mcal{F}}\norm{\ell_f-\ell_{f^*}}_{\infty}$ and $\hat{r}^*$ is the fixed point of $\hat{\psi}_n(r)$.
\end{theorem}

We now focus on the properties of the class of deep ReLU neural networks described in (\ref{Eq: Class of DNN}). To begin with, we note that all functions in $\mcal{F}$ are uniformly bounded.

\begin{lemma}\label{Lm: Uniform Boundedness of DNN}
    Let $\mcal{F}$ be the class of deep neural networks as defined in (\ref{Eq: Class of DNN}). Then
	$$
	\sup_{h_{\mbf{\theta}_f}\in\mcal{F}}\sup_{\norm{\mbf{x}}\leq \kappa}\abs{h_{\mbf{\theta}_f}(\mbf{x})}\leq\kappa\left(\prod_{i=1}^L\kappa_i\right).
	$$
\end{lemma}

For simplicity, we denote $M=\kappa\prod_{i=1}^L\kappa_i$, $p_{max}=\max\{p_0, p_1, \ldots, p_L\}$.  According to Theorem \ref{Thm: Extension of Bartlett Thm 5.4}, we need a subroot function that upper bounds the local Rademacher complexity of the star hull of the class of deep neural networks in (\ref{Eq: Class of DNN}) and a choice of such subroot function is given by the following lemma.

\begin{lemma}[Local Rademacher Complexity of DNN]\label{Lm: Local Rademacher Complexity of DNN}
	$$
	\mbb{E}_{\xi}\left[\sup_{h_{\mbf{\theta}_f}\in\mrm{star}(\mcal{F},f_0), \norm{h_{\mbf{\theta}_f}-f_0}_n^2\leq r}\frac{1}{n}\sum_{i=1}^n\xi_ih_{\mbf{\theta}_f}(\mbf{X}_i)\right]\leq\frac{12\sqrt{2}}{\sqrt{n}}\sqrt{W}\sqrt{r}\left(1+\log^{1/2}\frac{5M\left(\sum_{i=1}^L\left(\frac{b_i}{\kappa_i}\right)^{1/2}\right)^2}{\sqrt{r}}\right),
	$$
        where $W$ is the total number of parameters in the deep neural network.
        Denote 
        $$
        \hat{\psi}_n(r)=\frac{12\sqrt{2}}{\sqrt{n}}\sqrt{W}\sqrt{r}\left(1+\log^{1/2}\frac{5M\left(\sum_{i=1}^L\left(\frac{b_i}{\kappa_i}\right)^{1/2}\right)^2}{\sqrt{r}}\right).
        $$ 
        Then $\hat{\psi}_n(r)$ is a sub-root function with fixed point
        $$
        \hat{r}^*\lesssim\frac{W}{n}\log\left[M\left(\sum_{i=1}^L\left(\frac{b_i}{\kappa_i}\right)^{1/2}\right)^2\sqrt{n}\right].
        $$
\end{lemma}

By combining all the above results together, we obtain the rate of convergence for $h_{\mbf{\theta}_f}$.

\begin{theorem}[Rate of Convergence of DNN]\label{Thm: RoC of DNN}
    For any $\delta>0$, with probability at least $1-3\delta$,
    $$
    R(h_{\mbf{\theta}_f})-R(f_0)\lesssim\frac{W}{n}\log\left[M\left(\sum_{i=1}^L\left(\frac{b_i}{\kappa_i}\right)^{1/2}\right)^2\sqrt{n}\right]+\frac{M+B^*}{n}\log\frac{1}{\delta},
    $$
   and 
   $$
   \norm{h_{\mbf{\theta}_f}-f_0}\lesssim\frac{\sqrt{W}}{\sqrt{n}}\log^{1/2}\left[M\left(\sum_{i=1}^L\left(\frac{b_i}{\kappa_i}\right)^{1/2}\right)^2\sqrt{n}\right]+\sqrt{\frac{M+B^*}{n}\log\frac{1}{\delta}},
   $$
   where $B^*=\frac{4(1+e^M)^2}{e^M}$.
\end{theorem}

As a corollary, we can also get the convergence rate of a deep ReLU neural network in terms of the empirical norm by applying Theorem \ref{Thm: Empirical Norm and L2 Norm}, which also results in a bound for $\norm{\mbf{e}}$, where $\mbf{e}=\mbf{h}_{\mbf{\theta}_f}-\mbf{f}_0=[h_{\mbf{\theta}_f}(\mbf{X}_1)-f_0(\mbf{X}_1), \ldots, h_{\mbf{\theta}_f}(\mbf{X}_n)-f_0(\mbf{X}_n)]^T$.

\begin{corollary}\label{Cor: Bound for e}
    For any $\delta>0$, with probability at least $1-6\delta$,
    $$
    \norm{h_{\mbf{\theta}_f}-f_0}_n \lesssim \frac{(1\vee M)\sqrt{W}}{\sqrt{n}}\log^{1/2}\left[(1\vee M)\left(\sum_{i=1}^L\left(\frac{b_i}{\kappa_i}\right)^{1/2}\right)^2\sqrt{n}\right]+(M+\sqrt{M+B^*})\sqrt{\frac{1}{n}\log\frac{1}{\delta}},
    $$
    and
    $$
    \norm{\mbf{e}} \lesssim (1\vee M)\sqrt{W}\log^{1/2}\left[(1\vee M)\left(\sum_{i=1}^L\left(\frac{b_i}{\kappa_i}\right)^{1/2}\right)^2\sqrt{n}\right]+(M+\sqrt{M+B^*})\sqrt{\log\frac{1}{\delta}}.
    $$
\end{corollary}

% Add conditions on $W$ so that the norm of $h_{\mbf{\theta}_f}-f_0$ is $o_p(n^{-1/4})$. Also mention that throughout the remainder of the manuscript, $\kappa, \kappa_l, b_l$, $l\in[L]$ will be assumed to be fixed.

For notation simplicity, throughout the remainder of the manuscript, we will assume $\kappa, \kappa_l, b_l$, $l=1,\ldots, L$ are fixed constants. As mentioned in Section \ref{Sec: VI Framework}, one of the assumptions needed for the VI framework is to ensure that the convergence rate of the estimated function is not too slow (see assumption (R1)). Based on the above results, we can know that when 
\begin{equation}\label{Eq: Condition on W}
W=o(\sqrt{n}/\log(n)),
\end{equation}
with probability at least $1-e^{-o(n^{1/2})}$, $\norm{h_{\mbf{\theta}_f}-f_0}=o(n^{-1/4})$. Similar conditions as (\ref{Eq: Condition on W}) also appeared in existing literature studying the consistency and rate of convergence of neural networks \citep{schmidt2020nonparametric, farrell2021deep,  shen2023asymptotic, shen2025consistency}.

\subsection{Bounding $\norm{\tilde{h}_{\mbf{\theta}_f+\Delta\mbf{\theta}_S} - h_{\mbf{\theta}_f+\Delta\mbf{\theta}_S}}$}\label{Sec: Bounding the difference}
In this section, we bound the $L_2$ norm between $\tilde{h}_{\mbf{\theta}_f+\Delta\mbf{\theta}_S}$ and $h_{\mbf{\theta}_f+\Delta\mbf{\theta}_S}$. Note that
\begin{align*}
        \abs{h_{\mbf{\theta}_f+\Delta\theta_S}(\mbf{x})-\tilde{h}_{\mbf{\theta}_f+\Delta\theta_S}(\mbf{x})} & =\abs{h_{\mbf{\theta}_f+\Delta\theta_S}(\mbf{x}) - h_{\mbf{\theta}_f}(\mbf{x})-[\nabla_{\mbf{\theta}}h_{\mbf{\theta}_f}(\mbf{x})]^T\Delta\theta_S}\\
        & =\abs{\int_0^1\left[\nabla_{\mbf{\theta}}h_{(1-t)\mbf{\theta}_f+t(\mbf{\theta}_f+\Delta\theta_S)}(\mbf{x})-\nabla_{\mbf{\theta}}h_{\mbf{\theta}_f}(\mbf{x})\right]^T\Delta\theta_Sdt}\\
        & =\abs{\int_0^1\left[\nabla_{\mbf{\theta}}h_{\mbf{\theta}_f+t\Delta\theta_S}(\mbf{x})-\nabla_{\mbf{\theta}}h_{\mbf{\theta}_f}(\mbf{x})\right]^T\Delta\theta_Sdt}\\
        & \leq\norm{\Delta\theta_S}\int_0^1\norm{\nabla_{\mbf{\theta}}h_{\mbf{\theta}_f+t\Delta\theta_S}(\mbf{x})-\nabla_{\mbf{\theta}}h_{\mbf{\theta}_f}(\mbf{x})}dt.\numberthis\label{Eq: Integral of norm of difference in gradient}
\end{align*}
where the second equality follows from applying Taylor's theorem, similar to equation (1.3.2) in \citet{misiakiewicz2024six}, and the last inequality follows from the Cauchy-Schwarz inequality and the triangle inequality. As we can see from (\ref{Eq: Integral of norm of difference in gradient}), it suffices to bound the norm of the difference in gradients. If the activation $\varphi$ is smooth, it is reasonable to assume that the gradient of the weights is Lipschitz continuous. However, since our focus is on the ReLU activation function, it is not differentiable. Instead, we bound the norm of the gradient vector immediately. Such an approach has been applied to study the convergence properties of shallow neural networks trained through (stochastic) gradient descent as in \citet{oymak2020toward}. To do so, we follow the idea in \citet{zou2018stochastic}. Note that given an input $\mbf{x}$, the output of the neural network after the $l$-th layer is
\begin{align*}
\mbf{z}_l & =\sigma\left(\mbf{W}_l\sigma(\mbf{W}_{l-1}\cdots\sigma(\mbf{W}_1\mbf{x}))\right)\\
    & =\left(\prod_{j=1}^l\mbf{\Sigma}_j\mbf{W}_j\right)\mbf{x},
\end{align*}
where 
$$
\prod_{j=l_1}^{l_2}\mbf{W}_l=\left\{\begin{array}{ll}
    \mbf{W}_{l_2}\mbf{W}_{l_2-1}\cdots\mbf{W}_{l_1}, &  \mrm{if }l_1\leq l_2\\
     \mbf{I}, & \mrm{otherwise} 
\end{array}
\right.,
$$
and
\begin{align*}
    \mbf{\Sigma}_1 & =\mrm{Diag}\left(\mbb{I}_{\left\{\mbf{w}_{1,1}^T\mbf{x}>0\right\}}, \ldots,\mbb{I}_{\left\{\mbf{w}_{1, p_1}^T\mbf{x}>0\right\}}\right),\\
    \mbf{\Sigma}_j & =\mrm{Diag}\left(\mbb{I}_{\left\{\mbf{w}_{j,1}^T\left(\prod_{i=1}^{j-1}\mbf{\Sigma}_i\mbf{W}_i\right)\mbf{x}>0\right\}}, \ldots, \mbb{I}_{\left\{\mbf{w}_{j,p_j}^T\left(\prod_{i=1}^{j-1}\mbf{\Sigma}_i\mbf{W}_i\right)\mbf{x}>0\right\}}\right), \quad j=1, \ldots, L-1.
\end{align*}
Then the gradients of $h_{\mbf{\theta}_f}(\mbf{x})$ with respect to $\mbf{W}_l$, $l=1,\ldots, L$ are
\begin{align*}
    \nabla_{\mbf{W}_L}h_{\mbf{\theta}_f}(\mbf{x}) & =\left[\left(\prod_{j=1}^{L-1}\mbf{\Sigma}_j\mbf{W}_j\right)\mbf{x}\right]^T\\
    \nabla_{\mbf{W}_l}h_{\mbf{\theta}_f}(\mbf{x}) & =\nabla_{\mbf{W}_l}\mbf{W}_L\left(\prod_{j=1}^{L-1}\mbf{\Sigma}_j\mbf{W}_j\right)\mbf{x}\\
    & =\nabla_{\mbf{W}_l}\mbf{W}_L\left(\prod_{j=l}^{L-1}\mbf{\Sigma}_j\mbf{W}_j\right)\mbf{z}_{l-1}\\
        & =\mbf{\Sigma}_l\left(\prod_{j=l+1}^{L-1}\mbf{\Sigma}_j\mbf{W}_j\right)^T\mbf{W}_L^T\mbf{z}_{l-1}^T,\quad l=1,\ldots, L-1,
\end{align*}
where $\mbf{z}_0=\mbf{x}$. 

\begin{lemma}\label{Lm: Bound Derivative of Deep ReLU Neural Networks}
    For any $h_{\mbf{\theta}_f}\in\mcal{F}$ and $\norm{\mbf{x}}\leq\kappa$,
    $$
    \norm{\nabla_{\mbf{\theta}}h_{\mbf{\theta}_f}(\mbf{x})}\leq \kappa\left(\prod_{j=1}^{L-1}\kappa_j\right)\left(1+b_L\sqrt{\sum_{l=1}^{L-1}\frac{1}{\kappa_l^2}}\right).
    $$
    Consequently,
    $$
    \mrm{tr}(\mbf{K}_{-S})\leq n\kappa^2\left(\prod_{j=1}^{L-1}\kappa_j^2\right)\left(1+b_L^2\sum_{l=1}^{L-1}\frac{1}{\kappa_l^2}\right)
    $$
\end{lemma}

\begin{proposition}\label{Prop: Bound on norm of Delta theta S}
    For $n$ sufficiently large and any $\delta>0$, with probability at least $1-7\delta$,
    \begin{align*}
        \max\left\{\norm{\Delta\mbf{\theta}_S}, \norm{\mbf{\Phi}_{-S}\Delta\mbf{\theta}_{S}}\right\} & = o(n^{-\frac{1}{4}})+o(n^{-\frac{\alpha+3}{2(\alpha+1)}})\left(\sqrt{W}\log^{1/2}n+\sqrt{\log\frac{1}{\delta}}\right).
    \end{align*}    
\end{proposition}

Consequently, we have the bound for the $L_2$ distance between $\tilde{h}_{\mbf{\theta}_f+\Delta\mbf{\theta}_S}$ and $h_{\mbf{\theta}_f+\Delta\mbf{\theta}_S}$.

\begin{theorem}\label{Thm: Distance between h tilde and h}
    For any $\mbf{x}$ satisfying $\norm{\mbf{x}}\leq\kappa$,
    \begin{align*}
        \abs{h_{\mbf{\theta}_f+\Delta\theta_S}(\mbf{x})-\tilde{h}_{\mbf{\theta}_f+\Delta\theta_S}(\mbf{x})}\leq2\norm{\Delta\theta_S}\kappa\left(\prod_{j=1}^{L-1}\kappa_j\right)\left(1+b_L\sqrt{\sum_{l=1}^{L-1}\frac{1}{\kappa_l^2}}\right).
    \end{align*}
    Consequently,
    \begin{equation}\label{Eq: L2 Distance between h tilde and h}
        \norm{h_{\mbf{\theta}_f+\Delta\mbf{\theta}_S}-\tilde{h}_{\mbf{\theta}_f+\Delta\mbf{\theta}_S}}\lesssim o(n^{-\frac{1}{4}})+o(n^{-\frac{\alpha+3}{2(\alpha+1)}})\sqrt{W}\log^{1/2}n.
    \end{equation}
\end{theorem}

\subsection{Estimation Error of $\tilde{h}_{\mbf{\theta}_f+\Delta\mbf{\theta}_S}$}\label{Sec: Estimation Error of h tilde}
Since $\tilde{h}_{\mbf{\theta}_f+\Delta\mbf{\theta}_S}$ is the estimated function under the lazy regime for logistic regression under the null hypothesis, it is natural to look at the estimation error between $\tilde{h}_{\mbf{\theta}_f+\Delta\mbf{\theta}_S}$ and the underlying truth $f_{0,-S}$. To begin with, we quantify the estimation error with respect to the $L_2(\mbb{P}_n)$-norm.

\begin{lemma}\label{Lm: Empirical norm Upper bounds}
    Under the assumptions (A1) and (A2), for any $\delta>0$ and $n$ suffciently large, there exists $\bar{K}>0$, such that with probability at least $1-7\delta$,
    \begin{equation}\label{Eq: Upper Bound for empirical L2 norm}
    \sqrt{n}\norm{\tilde{h}_{\mbf{\theta}_f+\Delta\mbf{\theta}_S}^{(k+1)}-f_{0,-S}}_n\lesssim\norm{\mbf{\Phi}_{-S}\mbf{w}^{(k)}}+\sqrt{W}\log^{1/2}n+\sqrt{\log\frac{1}{\delta}}+o(n^{\frac{\alpha-1}{4(\alpha+1)}}), \forall k\geq\bar{K}.
    \end{equation}
    In addition, denote $\tilde{h}_{\mbf{\theta}_f+\Delta\mbf{\theta}_S}=\lim_{k\to\infty}\tilde{h}_{\theta+\Delta\mbf{\theta}_S}^{(k)}$. Then under the assumptions (A1) and (A2), with probability at least $1-14\delta$,
    \begin{equation}\label{Eq: Upper Bound for empirical L2 norm at convergence}
        \sqrt{n}\norm{\tilde{h}_{\mbf{\theta}_f+\Delta\mbf{\theta}_S}-f_{0,-S}}_n\lesssim \sqrt{W}\log^{1/2}n+\sqrt{\log\frac{1}{\delta}}+o(n^{-\frac{1}{4}})+o_p(n^{\frac{\alpha-1}{4(\alpha+1)}}).
    \end{equation}
\end{lemma}

Lemma \ref{Lm: Empirical norm Upper bounds} shows that when the Newton-Raphson algorithm converges, the estimation error of the linearized neural network is $o_p(n^{-1/4})$ with high probability under the empirical $L_2$-norm. 

We now turn to bound the estimation error of the linearized neural network with respect to the $L_2$-norm. The idea is to use the local Rademacher complexity to bridge the estimation error in $L_2(\mbb{P}_n)$-norm and the estimation error in $L_2(P)$-norm. Note that

\begin{align*}
    \tilde{h}_{\mbf{\theta}_f+\Delta\mbf{\theta}_S}^{(k+1)}(\mbf{x}) & =h_{\mbf{\theta}_f,-S}(\mbf{x})+\left[\nabla_{\mbf{\theta}} h_{\mbf{\theta}_f}(\mbf{x})\right]^T\mbf{w}^{(k+1)}\\
    & =h_{\mbf{\theta}_f,-S}(\mbf{x})+\left[\nabla_{\mbf{\theta}} h_{\mbf{\theta}_f}(\mbf{x})\right]^T\mbf{\Phi}_{-S}^T\mbf{\Pi}_{-S}^{(k)}\left[\mbf{\Phi}_{-S}\mbf{\Phi}_{-S}^T\mbf{\Pi}_{-S}^{(k)}+n\lambda\mbf{I}_n\right]^{-1}\\
    & \qquad\qquad\qquad\qquad\qquad\left[\mbf{\Phi}_{-S}\mbf{w}^{(k)}+\mbf{\Pi}_{-S}^{(k)^{-1}}\left[\mbf{Y}-\sigma\left(\tilde{\mbf{h}}_{\mbf{\theta}_f+\Delta\mbf{\theta}_S}^{(k)}\right)\right]\right]\\
    & =h_{\mbf{\theta}_f,-S}(\mbf{x})+K_{-S}(\mbf{x},\mbf{X}_{-S})^T\mbf{\alpha}_{-S}^{(k)},
\end{align*}
where 
\begin{align*}
     K_{-S}(\mbf{x}, \mbf{X}_{-S})^T & =\left[\nabla_{\mbf{\theta}} h_{\mbf{\theta}_f}(\mbf{x})\right]^T\mbf{\Phi}_{-S}^T\in\mbb{R}^{1\times n}\\
    \mbf{\alpha}_{-S}^{(k)}& =\mbf{\Pi}_{-S}^{(k)}\left[\mbf{\Phi}_{-S}\mbf{\Phi}_{-S}^T\mbf{\Pi}_{-S}^{(k)}+n\lambda\mbf{I}_n\right]^{-1}\left[\mbf{\Phi}_{-S}\mbf{w}^{(k)}+\mbf{\Pi}_{-S}^{(k)^{-1}}\left[\mbf{Y}-\sigma\left(\tilde{\mbf{h}}_{\mbf{\theta}_f+\Delta\mbf{\theta}_S}^{(k)}\right)\right]\right]\in\mbb{R}^n\\
        & =\left[\mbf{K}_{-S}+n\lambda\mbf{\Pi}_{-S}^{(k)^{-1}}\right]^{-1}\left[\mbf{\Phi}_{-S}\mbf{w}^{(k)}+\mbf{\Pi}_{-S}^{(k)^{-1}}\left[\mbf{Y}-\sigma\left(\tilde{\mbf{h}}_{\mbf{\theta}_f+\Delta\mbf{\theta}_S}^{(k)}\right)\right]\right]
\end{align*}
As a result, $\tilde{h}_{\mbf{\theta}_f+\Delta\mbf{\theta}_S}^{(k+1)}-h_{\mbf{\theta}_f,-S}$ is a linear combination of neural tangent kernel. Define
$$
\mcal{H}_B=\left\{f(\mbf{x})=\mbf{\alpha}^TK_{-S}(\mbf{x},\mbf{X}_{-S}):\norm{f}_{\mcal{H}}=\sqrt{\mbf{\alpha}^T\mbf{K}_{-S}\mbf{\alpha}}\leq B\right\}.
$$
Lemma \ref{Lm: RKHS norm upper bound} shows that $\tilde{h}^{(k+1)}_{\mbf{\theta}_f+\Delta\mbf{\theta}_S}-h_{\mbf{\theta}_f,-S}\in\mcal{H}_B$. 

\begin{lemma}\label{Lm: RKHS norm upper bound}
    Under the Assumptions (A1) and (A2), for any $\delta>0$, with probability at least $1-7\delta$,
    \begin{equation}\label{Eq: Upper bound for the RKHS norm}
        \norm{\tilde{h}^{(k+1)}_{\mbf{\theta}_f+\Delta\mbf{\theta}_S}-h_{\mbf{\theta}_f,-S}}_\mathcal{H}\lesssim o\left(n^{-\frac{\alpha+3}{4(\alpha+1)}}\right)\left(\sqrt{n}+\sqrt{W}\log^{1/2}n+\norm{\mbf{\Phi}_{-S}\mbf{w}^{(k)}}+\sqrt{\log\frac{1}{\delta}}\right).
    \end{equation}
    In addition, with probability at least $1-14\delta$,
     \begin{equation}\label{Eq: Upper bound for the RKHS norm at convergence}
        \norm{\tilde{h}_{\mbf{\theta}_f+\Delta\mbf{\theta}_S}-h_{\mbf{\theta}_f,-S}}_\mathcal{H}\lesssim o\left(n^{-\frac{\alpha+3}{4(\alpha+1)}}\right)\left(\sqrt{n}+\sqrt{W}\log^{1/2}n+\sqrt{\log\frac{1}{\delta}}\right).
    \end{equation}
\end{lemma}

In \citet{gao2022lazy}, the estimation error of the linearized deep ReLU network was obtained by using the upper bound for the Rademacher complexity of $\mcal{H}_B$, which we believe is the main reason that $\mcal{O}_p(n^{-1/2})$ error rate can only be obtained. Instead, we use the upper bound for the local Rademacher complexity for $\mcal{H}_B$ (see Lemma \ref{Lm: Local Rad Comp of HB} in Appendix \ref{appendix: Auxiliary results}). Additionally, it is also important to note that functions in $\mcal{H}_B$ are uniformly bounded provided the kernel function is uniformly bounded (see Lemma \ref{Lm: Uniform Boundedness of HB} in Appendix \ref{appendix: Auxiliary results}).

% Additionally, since our goal is to provide an upper bound for $\norm{\tilde{h}_{\mbf{\theta}_f+\Delta\mbf{\theta}_S}^{(k+1)}-f_{0,-S}}$, we also provide the upper bound for the local Rademacher complexity for $\tilde{\mcal{H}}_B:=\{f-(f_{0,-S}-h_{\mbf{\theta}_f,-S}):f\in\mcal{H}_B\}$.

% \begin{corollary}\label{Cor: Local Rad Comp of HB tilde}
%     Let $\rho_n=\norm{f_{0,-S}-h_{\mbf{\theta}_f,-S}}_n^2$. Then
%     $$
%     \hat{\mcal{R}}_n(r;\tilde{\mcal{H}}_B)\leq(1\vee B)\frac{\sqrt{2}}{n}\sqrt{\sum_{j=1}^n\min\{2(r+\rho_n),\mu_j\}}.
%     $$
% \end{corollary}

Since $f_{0,-S}$ and $h_{\mbf{\theta}_f,-S}$ are uniformly bounded based on the definition of $\mcal{F}$, then under the kernel regularity conditions in the assumption (A1), we can assume that there exists $\tilde{M}>0$, such that $\sup_{\mbf{x}}\abs{f(\mbf{x})}+\sup_{\mbf{x}}\abs{f_{0,-S}(\mbf{x})-h_{\mbf{\theta}_f,-S}(\mbf{x})}\leq \tilde{M}$ for all $f\in\mcal{H}_B$. We also define 
\begin{equation}\label{Eq: FB bar}
    \bar{\mcal{H}}_B=\left\{\frac{f}{\tilde{M}}:f\in\mcal{H}_B\right\}
\end{equation}

\begin{lemma}[Upper Bound for $L_2$-norm of Functions in $\tilde{\mcal{H}}_B$]\label{Lm: Upper Bound for L2 Norm of f in HB}
     Let $\rho=\norm{f_{0,-S}-h_{\mbf{\theta}_f,-S}}^2/\tilde{M}^2$ and let $\hat{\psi}_n(r)$ be a (possibly data-dependent) sub-root function and let $\hat{r}^*$ be the fixed point of $\hat{\psi}_n$ (i.e., $\hat{\psi}_n(\hat{r}^*)=\hat{r}^*$). Fix $\delta>0$ and assume that $\hat{\psi}_n$ satisfies for any $r\geq\hat{r}^*$,
    \begin{equation}\label{Eq: psi_n_hat}
        \hat{\psi}_n(r)\geq 20\mbb{E}_\xi\left[\left.\sup_{h\in\bar{\mcal{H}}_B, \norm{h}_n^2\leq4(r+\rho)}\frac{1}{n}\sum_{i=1}^n\xi_ih(\mbf{X}_{i,-S})\right|\mbf{X}_1,\ldots, \mbf{X}_n\right]+\frac{62}{3n}\log\frac{1}{\delta}.
    \end{equation}
    Then with probability at least $1-3\delta$,
    \begin{equation}\label{Eq: Upper Bound of L2 Norm}
        \norm{f-(f_{0,-S}-h_{\mbf{\theta}_f,-S})}\leq\sqrt{2}\norm{f-(f_{0,-S}-h_{\mbf{\theta}_f,-S})}_n+\sqrt{2\tilde{M}^2c_1\hat{r}^*}+\sqrt{\frac{\tilde{M}^2(11+2c_2)}{n}\log\frac{1}{\delta}},\quad\forall f\in\mcal{H}_B
    \end{equation}
    where $c_1$ and $c_2$ are universal constants.
\end{lemma}

As one can tell from (\ref{Eq: Upper Bound of L2 Norm}) in Lemma \ref{Lm: Upper Bound for L2 Norm of f in HB}, bounding the $L_2(P)$-norm via the $L_2(\mbb{P}_n)$-norm relies heavily on the fixed point of the sub-root function $\hat{\psi}(r)$. Lemma \ref{Lm: Fixed point of sub-root function for HB} below provides the rate of $r^*$, the fixed point of a sub-root function for the local Rademacher complexity of $\mcal{H}_B$.

\begin{lemma}\label{Lm: Fixed point of sub-root function for HB}
    Let $\rho=\norm{f_{0,-S}-h_{\mbf{\theta}_f,-S}}^2/\tilde{M}^2$ and $\hat{\psi}_n(r) :=20\sqrt{\frac{2}{n}}\sqrt{c_\alpha(8(r\vee\rho))^{1-\frac{1}{\alpha}}}+\frac{62}{3n}\log\frac{1}{\delta}$. Then $\hat{\psi}_n(r)$ is a sub-root function and it satisfies (\ref{Eq: psi_n_hat}) in Lemma \ref{Lm: Upper Bound for L2 Norm of f in HB}. Moreover, under the kernel regularity assumptions (A1),
    $$
    \hat{r}^*\leq\max\left\{\tilde{c}_\alpha n^{-1/2}\rho^{\frac{1}{2}\left(1-\frac{1}{\alpha}\right)}, \tilde{c}_\alpha^{\frac{2\alpha}{\alpha+1}}n^{-\frac{\alpha}{\alpha+1}}\right\}+\frac{124}{3n}\log\frac{1}{\delta},
    $$
    where $\hat{r}^*$ is the fixed point of the sub-root function $\hat{\psi}_n(r)$ and $\tilde{c}_\alpha=20\sqrt{2c_\alpha\cdot 8^{1-\frac{1}{\alpha}}}$.
\end{lemma}

Combining Lemma \ref{Lm: Empirical norm Upper bounds}, Lemma \ref{Lm: RKHS norm upper bound}, Lemma \ref{Lm: Upper Bound for L2 Norm of f in HB} and Lemma \ref{Lm: Fixed point of sub-root function for HB}, we obtain the following main result on the error rate of the estimated function.
\begin{theorem}\label{Thm: Estimation error of h tilde}
    Under assumptions (A1) and (A2), with probability at least $1-17e^{-n^{\frac{1}{\alpha+1}}}$,
    \begin{align*}
        \norm{h_{\mbf{\theta}_f+\Delta\mbf{\theta}_S}^{(k+1)}-f_{0,-S}} & \lesssim_\alpha\frac{1}{\sqrt{n}}\norm{\mbf{\Phi}_{-S}\mbf{w}^{(k)}}+n^{-1/2}\sqrt{W}\log^{1/2}n+o(n^{-1/4})+\\
        & \qquad\qquad \left[o\left(n^{-\frac{1}{4}\left(1-\frac{\alpha-1}{2\alpha}\right)}\right)+o\left(n^{-\frac{1}{4}\left(1+\frac{\alpha+3}{2\alpha}\right)}\right)\left[\sqrt{W}\log^{1/2}n\right]^{\frac{\alpha+1}{2\alpha}}\right]\norm{f_{0,-S}-h_{\mbf{\theta}_f,-S}}^{\frac{\alpha-1}{2\alpha}}.\numberthis\label{Eq: L2 norm of hk}
    \end{align*}
    At convergence, with probability at least $1-31e^{-n^{\frac{1}{\alpha+1}}}$,
    \begin{align*}
        \norm{\tilde{h}_{\mbf{\theta}_f+\Delta\mbf{\theta}_S}-f_{0,-S}} & \lesssim_\alpha n^{-1/2}\sqrt{W}\log^{1/2}n+o(n^{-1/4})+\\
        & \qquad\qquad \left[o\left(n^{-\frac{1}{4}\left(1-\frac{\alpha-1}{2\alpha}\right)}\right)+o\left(n^{-\frac{1}{4}\left(1+\frac{\alpha+3}{2\alpha}\right)}\right)\left[\sqrt{W}\log^{1/2}n\right]^{\frac{\alpha+1}{2\alpha}}\right]\norm{f_{0,-S}-h_{\mbf{\theta}_f,-S}}^{\frac{\alpha-1}{2\alpha}}.\numberthis\label{Eq: L2 norm of h*}
    \end{align*}
\end{theorem}

\subsection{Lazy VI for Binary Classification}
Now we are going to formulate the variable importance as a hypothesis testing problem. If features with indices in $S$ are not important, whether including them to train a model should not change the value of $\psi_{0,S}$. Therefore, the null and alternative hypotheses are
\begin{equation}\label{Eq: hypothesis using vi}
    H_0:\psi_{0,S}=0\mrm{ vs }H_1:\psi_{0,S}>0.
\end{equation}

Under our setting, the predictiveness measure can be defined as
$$
V(f, P_0)=-\mbb{E}_{P_0}\left[-Yf(\mbf{X})+\log\left(1+e^{f(\mbf{X})}\right)\right].
$$
Therefore, its G\^ateaux derivative at $P_0$ along the direction $H\in\mcal{S}$ can be calculated directly: Let $\ell(f(\mbf{X}), Y)=-Yf(\mbf{X})+\log\left(1+e^{f(\mbf{X})}\right)$, then for any $f\in \mathcal{F}$ and $H\in\mcal{S}$,
\begin{align*}
    \frac{V(f, P_0 + \eta H) - V(f, P_0)}{\eta} 
        &= -\frac{\mathbb{E}_{P_0 + \eta H_j}[\ell(f(\mbf{X}), Y)] - \mathbb{E}_{P_0}[\ell(f(\mbf{X}), Y)]}{\eta} \\
        &= -\frac{\mathbb{E}_{P_0}[\ell(f(\mbf{X}), Y)] + \eta \mathbb{E}_{H}[\ell(f(\mbf{X}), Y)] - \mathbb{E}_{P_0}[\ell(f(\mbf{X}), Y)]}{\eta} \\
        &= -\mathbb{E}_{H}[\ell(f(\mbf{X}), Y)].\numberthis\label{Eq: Gateaux derivative}
\end{align*}
Therefore,
\begin{align*}
    g_n(\mbf{Z}) & =\dot{V}(\hat{f}_n, P_0;\delta_{\mbf{Z}}-P_0)-\dot{V}(f_0,P_0;\delta_{\mbf{Z}}-P_0)\\
    & =-\left(\mbb{E}_{\delta_{\mbf{Z}}-P_0}\left[\ell(\hat{f}_n(\mbf{X}), Y)\right]-\mbb{E}_{\delta_{\mbf{Z}}-P_0}\left[\ell(f_0(\mbf{X}), Y)\right]\right)\\
    & =\mbb{E}_{P_0}\left[\ell(\hat{f}_n(\mbf{X}), Y)-\ell(f_0(\mbf{X}), Y)\right]-[\ell(\hat{f}_n(\mbf{X}), Y)-\ell(f_0(\mbf{X}), Y)]\numberthis\label{Eq: g_n in binary classification}
\end{align*}
We are now ready to state the main theorem to conduct hypothesis testing based on the Lazy VI.

\begin{theorem}\label{Thm: Lazy VI HT}
    Suppose that 
    $$
    W=o(\sqrt{n}/\log n).
    $$
    Then for any $S\subset[p]$, under $H_0$, with probability at least $1-31e^{-n^{\frac{1}{\alpha+1}}}$,
    \begin{align*}
        \hat{\psi}_{n,S}-\psi_{0,S} &=V(h_{\mbf{\theta}_f},\mbb{P}_n) -V(h_{\mbf{\theta}_f+\Delta\mbf{\theta}_S}, \mbb{P}_n)-[V(f_0, P_0)-V(f_{0,-S}, P_0)]\\
        & =\frac{1}{n}\sum_{i=1}^n\left[\dot{V}(f_0, P_0;\delta_{\mbf{Z}_i-P_0})-\dot{V}(f_{0,-S}, P_0;\delta_{\mbf{Z}_i}-P_0)\right]+o_p(n^{-1/2}).
    \end{align*}
\end{theorem}

An important consequence of Theorem \ref{Thm: Lazy VI HT} is that under $H_0$, $\hat{\psi}_{n,S}$ follows an asymptotic normal distribution with meann 0 and variance $\mrm{Var}[\dot{V}(f_0, P_0;\delta_{\mbf{Z}-P_0})-\dot{V}(f_{0,-S}, P_0;\delta_{\mbf{Z}}-P_0)]$, which can be estimated by
$$
\tau_{n,-S}^2=\frac{1}{n}\sum_{i=1}^n\left[t_{i,-S}-\bar{t}\right]^2,
$$
where $t_{i,-S}=\dot{V}(f_0, \mbb{P}_n;\delta_{\mbf{Z}_i-\mbb{P}_n})-\dot{V}(f_{0,-S}, \mbb{P}_n;\delta_{\mbf{Z}_i}-\mbb{P}_n)$. This suggests that testing the importance of a set of features is the same as performing a $Z$-test in classical statistical inference. 

Following the idea in \citet{williamson2023general}, we divide the entire dataset into two parts (say training and test sets), estimating $f_0$ using the training data, and then evaluating the predictiveness measure on the test data. Algorithm \ref{alg:Lazy Vi for binary classification} provides the lazy VI framework for testing the importance of each feature in the dataset, i.e. $S=\{j\}$, $j=1,\ldots,p$.
\begin{algorithm}
    \caption{Lazy VI for binary classification}\label{alg:Lazy Vi for binary classification}
    \begin{algorithmic}
        \Require: Data: $\{(\mbf{X}_i, Y_i)\}_{i=1}^n$; $\lambda>0$; training size: $0<n_1<n$; $n_2\gets n-n_1$; DNN structure: $\mbf{\theta}\in\mbb{R}^W\mapsto h_\theta(\cdot)$; initial value $\mbf{w}^{(0)}$ for Newton-Raphson algorithm
        \State $\mbf{\theta}_f\gets\mrm{argmin}_{\theta\in\mbb{R}^W}\frac{1}{n_1}\sum_{i=1}^{n_1}\left[-Y_ih_{\theta}(\mbf{X}_i)+\log\left(1+e^{h_\theta(\mbf{X}_i)}\right)\right]$
        \State $V(h_{\mbf{\theta}_f}, \mbb{P}_n)\gets-\frac{1}{n_2}\sum_{i=n_1+1}^n\left[-Y_ih_{\theta}(\mbf{X}_i)+\log\left(1+e^{h_\theta(\mbf{X}_i)}\right)\right]$
        \For{$j\in[p]$}
            \State $\mbf{X}_{i,-j}\gets\mbf{X}_i$; $[\mbf{X}_{i,-j}]_j\gets\frac{1}{n_1}\sum_{i=1}^n[\mbf{X}_{i}]_j$
            \State $[\mbf{\Phi}_{-j}]_{:,i} \gets\left.\nabla_{\mbf{\theta}}h_{\mbf{\theta}}(\mbf{X}_{i,-j})\right|_{\mbf{\theta}=\mbf{\theta}_f}, i=1,\ldots, n_1$
            \State \scalebox{0.9}{$\Delta\theta_j\gets\mrm{argmin}_{\mbf{w}\in\mbb{R}^W}\frac{1}{n}\sum_{i=1}^n\left[-Y_i (h_{\mbf{\theta}_f}(\mbf{X}_{i,-j})+[\mbf{\Phi}_{-j}]_{:,i}^T\mbf{w})+\log\left(1+e^{h_{\mbf{\theta}_f}(\mbf{X}_{i,-j})+[\mbf{\Phi}_{-j}]_{:,i}^T\mbf{w})}\right)\right]+\frac{\lambda}{2}\norm{\mbf{w}}^2$}
            \State \scalebox{0.9}{$V(h_{\mbf{\theta}_f+\Delta\theta_j},\mbb{P}_n)\gets-\frac{1}{n_2}\sum_{i=n_1+1}^n\left[-Y_ih_{\mbf{\theta}_f+\Delta\theta_j}(\mbf{X}_{i,-j})+\log\left(1+e^{h_{\mbf{\theta}_f+\Delta\theta_j}(\mbf{X}_{i,-j})}\right)\right]$}
            \State \scalebox{0.9}{$\hat{\psi}_{n,j}\gets V(h_{\mbf{\theta}_f},\mbb{P}_n)-V(h_{\mbf{\theta}_f+\Delta\theta_j},\mbb{P}_n)$}
            \State \scalebox{0.9}{$t_{i,j}\gets \left[-Y_ih_{\theta+\Delta\theta_j}(\mbf{X}_{i,-j})+\log\left(1+e^{h_{\theta+\Delta\theta_j}(\mbf{X}_{i,-j})}\right)\right]-\left[-Y_ih_{\mbf{\theta}_f}(\mbf{X}_{i})+\log\left(1+e^{h_{\mbf{\theta}_f}(\mbf{X}_{i})}\right)\right]$}
            \State $\hat{\tau}_{j}\gets\frac{1}{n_2}\sum_{i=1}^{n_2}(t_{i,j}-\bar{t}_j)^2$
        \EndFor
        \Ensure: $\hat{\psi}_{n,j}$, $j=1,\ldots, p$.
    \end{algorithmic}
\end{algorithm}

\section{Simulations and Experiments}
To evaluate the proposed LazyVI method for binary classification, we conducted extensive simulations and empirical experiments. In particular, the first simulation study investigated the empirical power of the proposed method and its ability to control the Type I error rate. The second simulation study examined whether the assumed eigenvalue decay rate of the neural tangent kernel holds for the simulated data generated in the first study. Finally, we applied our method to identify important subregions in images from the Modified National Institute of Standards and Technology (MNIST) database \citep{lecun2002gradient}.

\subsection{Empirical Type I Error and Empirical Power}\label{Sec: Type I Error and Power}
\subsubsection{Data Generation}
We conducted simulations to evaluate the performance of the Lazy Variable Importance framework in controlling the empirical Type I error rate and achieving high empirical power for detecting important variables. To begin, we generated $p = 10$ feature variables independently from the standard normal distribution,
\[
X_1, ..., X_{10} \stackrel{\text{i.i.d.}}{\sim} \mathcal{N}(0,1).
\]
The binary response variable $Y \in \{0,1\}$ was generated according to
\[
Y \sim \text{Ber}\big(\sigma(f(X_1,..., X_9))\big),
\]
where $\sigma(x)$ denotes the sigmoid function. That is, $Y$ follows a Bernoulli distribution with success probability $\sigma\big(f(X_1, \ldots, X_9)\big)$. In this simulation study, we consider two types of signal functions $f$: a linear function and a nonlinear function, defined below:
\begin{align*}
    \textrm{(Linear Signal)} & \quad f(X_1, \ldots, X_9) = 2.2 X_1 - 1.9 X_2 + 1.8 X_3 + 2.3 X_4 \\
        & \qquad\qquad\qquad\qquad- 1.9 X_5 + 2.4 X_6 - 1.7 X_7 + 1.9 X_8 + 2.1 X_9\\
    \textrm{(Nonlinear Signal)} & \quad f(X_1,..,X_9) = 4.6 X_1 - 2.1 X_2^2 + 3.3 X_3 + 1.3 X_4^3 - 6.5 \sin(X_5) \\
& \qquad\qquad\qquad\qquad + 2.6 \exp(\abs{X_6}) - 3.7 X_7 + 5.8 \cos(X_8) + 2.1 X_7 X_8 + 3.1 X_9^2.
\end{align*}

\noindent In both settings, we generated a sample size of $n=$ 5,000. 

\subsubsection{Implementation}
The purpose of this simulation study is to determine which of the ten variables are important, or statistically significantly associated with the response variable $Y$. To this end, we conduct a sequence of hypothesis tests. The null and alternative hypotheses are stated as follows:
\begin{align*}
H_{0i}: & \ \text{The variable } X_i \text{ is not significantly related to } Y, \\
H_{1i}: & \ \text{The variable } X_i \text{ is significantly related to } Y, \quad i = 1, \ldots, 10.
\end{align*}
It is important to note that the variable $X_{10}$ is not involved in generating the response variable $Y$; therefore, it is used to assess whether the proposed LazyVI method controls the empirical Type I error rate at the nominal significance level $\alpha = 0.05$. The remaining nine variables, $X_1, \ldots, X_9$, are used to generate $Y$, and each is truly associated with the response. Testing the hypotheses $H_{0i}$ versus $H_{1i}$ for $i = 1, \ldots, 9$ allows us to evaluate the empirical power of the proposed framework.

For the linear signal case, we implemented a feedforward deep neural network with two hidden layers, each consisting of 50 neurons with ReLU activation functions, and a sigmoid output layer for binary classification. The model was trained using the binary cross-entropy loss and the Limited-memory BFGS optimizer with a learning rate of $\eta = 0.0005$. Training employed early stopping with a patience of 10 epochs and a minimum improvement threshold of $10^{-4}$ to prevent overfitting. The ridge penalty parameter $\lambda$ was selected using 3-fold cross-validation from a logarithmically spaced grid of 10 values ranging from $10^{-4}$ to $10^{-2.25}$, choosing the value that minimized the validation loss.

Similarly, in the nonlinear signal case, we trained a feedforward deep neural network with two hidden layers of 50 neurons each using ReLU activation functions and a sigmoid output layer. The model was trained using the binary cross-entropy loss and the Limited-memory BFGS optimizer with a learning rate of $\eta = 0.0005$. Training employed early stopping with a patience of 10 epochs and a minimum improvement threshold of $10^{-4}$ to prevent overfitting. The ridge penalty parameter $\lambda$ was selected using 3-fold cross-validation from a logarithmically spaced grid of 10 values ranging from $10^{-4}$ to $10^{-2.25}$, with the optimal value chosen according to the validation loss.

After training the neural network, we applied the LazyVI framework to assess the importance of each feature variable $X_1, \ldots, X_{10}$. For each feature, we computed a test statistic and its corresponding standard error under the LazyVI framework, and then calculated a one-sided $p$-value using the standard normal cumulative distribution function (CDF):
\[
PV_i = 1 - \Phi\left(\frac{\text{VI}_i}{\text{SE}_i}\right),
\]
where $\Phi(\cdot)$ is the standard normal CDF and $\text{VI}_i$ and $\text{SE}_i$ denote the variable importance estimate and its standard error for feature $X_i$, respectively. For each feature $X_i$, we tested the null hypothesis $H_{0i}$ that the variable is not important, and the null hypothesis $H_{0i}$ was rejected if its corresponding p-value $PV_i < \alpha$ for our chosen significance level $\alpha=0.05$. For $X_i \in \{X_1,..,X_9\},$ the empirical power is the proportion of the time that our algorithm claims that $X_i$ is important. Since all of these features are important in determining our label $Y$, these should be close to $1.$ For $X_{10}$, the empirical Type I error is the proportion of the time that our algorithm incorrectly claims that $X_{10}$ is important. 

\subsubsection{Results}

To benchmark the performance of LazyVI against existing methods for inference in deep neural networks, we compare our results with the \texttt{dnn-inference} framework developed by \citet{dai2022significance}. The \texttt{dnn-inference} package implements statistical inference for deep neural networks using asymptotic approximations and influence-function-based methods designed for variable importance and hypothesis testing in high-dimensional settings.

In our experiments, we apply \texttt{dnn-inference} to the same simulated datasets used in the LazyVI experiments under both the linear and nonlinear signals. We then compare empirical power and Type~I error rates across all variables. While \texttt{dnn-inference} performs comparatively well in the linear setting, its empirical power decreases substantially for several variables in the nonlinear setting. In contrast, LazyVI maintains consistently high power while preserving appropriate Type~I error control.

We also compare computational efficiency. The \texttt{dnn-inference} procedure required 83,324 seconds in the linear setting and 166,551 seconds in the nonlinear setting, whereas LazyVInonlinear required approximately 21,534 seconds in the linear setting and 33,766 seconds in the linear setting. These results suggest that LazyVI provides a favorable balance between statistical performance and computational scalability for variable importance inference in deep neural networks.

\begin{table}[htbp]
\centering
\caption{Empirical rejection rates for LazyVI, Logistic Regression, and DNN-Inference under linear and nonlinear signal settings.}
\label{tab:combined_results}
\resizebox{\textwidth}{!}{
\begin{tabular}{l|ccc|ccc}
\toprule
& \multicolumn{3}{c|}{\bfseries Linear Signal} 
& \multicolumn{3}{c}{\bfseries Nonlinear Signal} \\
\textbf{Variable} 
& \textbf{LazyVI} 
& \textbf{Logistic Reg.} 
& \textbf{DNN-Inf.}
& \textbf{LazyVI} 
& \textbf{Logistic Reg.} 
& \textbf{DNN-Inf.} \\
\midrule
$X_1$  & 1.000 & 1.000 & 0.882 & 1.000 & 1.000 & 0.846 \\
$X_2$  & 1.000 & 1.000 & 0.864 & 0.960 & 0.039 & 0.103 \\
$X_3$  & 1.000 & 1.000 & 0.865 & 0.996 & 0.999 & 0.420 \\
$X_4$  & 1.000 & 1.000 & 0.920 & 0.996 & 0.996 & 0.426 \\
$X_5$  & 1.000 & 1.000 & 0.979 & 1.000 & 1.000 & 0.780 \\
$X_6$  & 1.000 & 1.000 & 0.941 & 1.000 & 0.032 & 0.183 \\
$X_7$  & 1.000 & 1.000 & 0.740 & 1.000 & 0.999 & 0.603 \\
$X_8$  & 1.000 & 1.000 & 0.882 & 1.000 & 0.999 & 0.024 \\
$X_9$  & 1.000 & 1.000 & 0.906 & 0.990 & 0.044 & 0.141 \\
$X_{10}$ & 0.059 & 0.041 & 0.020 & 0.047 & 0.037 & 0.020 \\
\bottomrule
\end{tabular}
}
\end{table}

\subsection{Decay Rate of NTK Eigenvalues}\label{Sec: Decay Rate of NTK eigenvalues}
To demonstrate the empirical validity of Assumption (A1), we use the same neural network architecture, training procedure, and nonlinear data generated in \ref{Sec: Type I Error and Power}. We use a dataset of size $n=5000$ observations, with 1650 of these held out as the test set to compute the NTK matrix. Assumption (A1) requires that the eigenvalues satisfy the power-law decay $\mu_j \leq c j^{-\alpha}$ for $\alpha > 1$. Taking the natural logarithm implies $\log(\mu_j) \leq \log(c) - \alpha \log(j)$. 

\begin{figure}[htbp]
    \centering
    \includegraphics[width=0.8\linewidth]{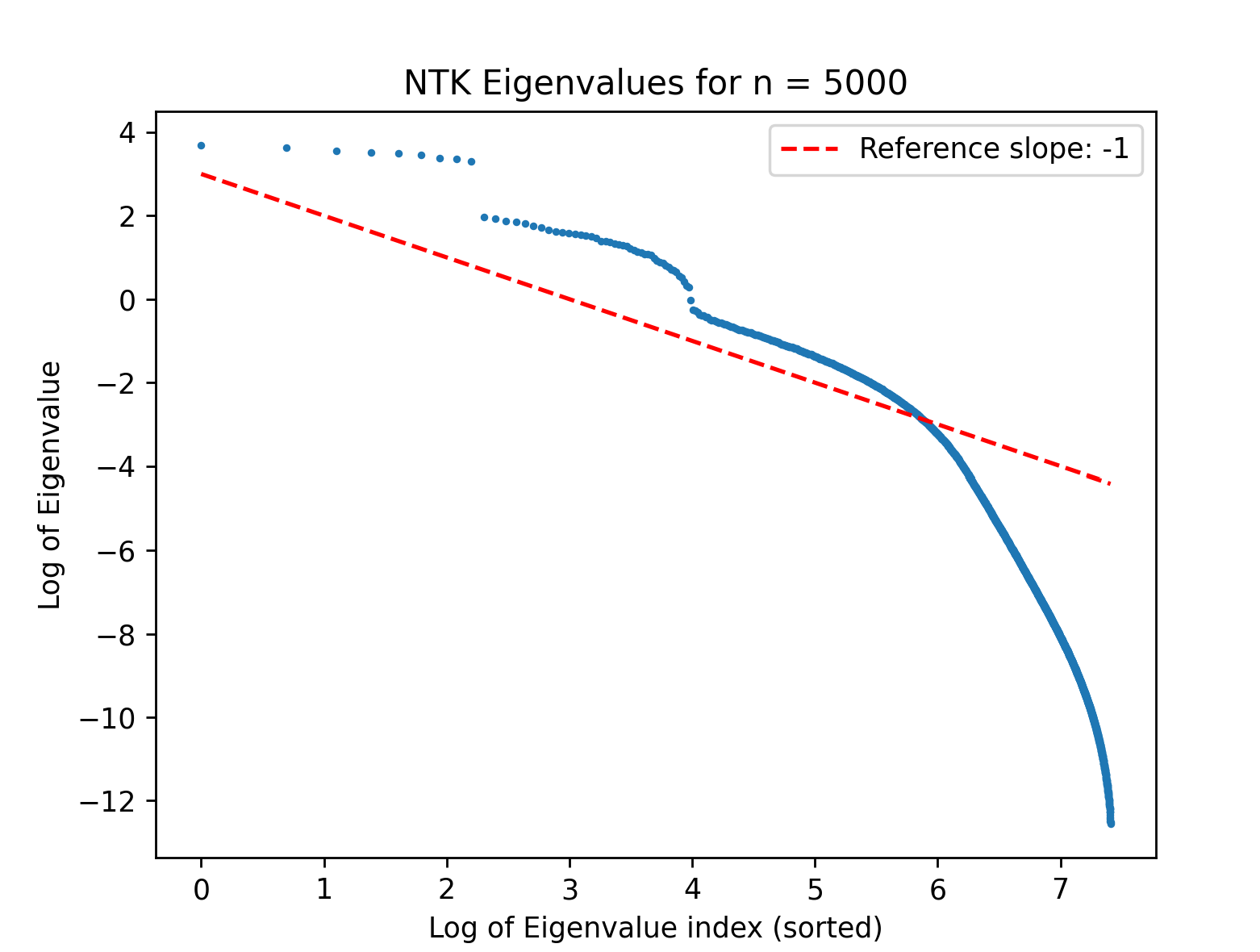}
    \caption{Log-log plot of the sorted NTK eigenvalues. The spectrum decays faster than the red reference line with slope $-1$, supporting the regularity condition $\alpha > 1$.}
    \label{Fig: Decay Rate of NTK Eigenvalues}
\end{figure}

In Figure \ref{Fig: Decay Rate of NTK Eigenvalues}, we plot the natural logarithm of the sorted eigenvalues against the natural logarithm of their indices. The results show that the eigenvalues decay significantly faster than a reference line with slope $-1$, confirming that the decay rate $\alpha$ is strictly greater than 1.

\subsection{Classification on MNIST}\label{Sec: Classification on MNIST}

\subsubsection{Problem Formulation and Region Definition}

We evaluated the efficacy of Mean Imputation (Dropout), Lazy Training, and Retraining methods for Variable Importance (VI) estimation using the MNIST dataset. The classification task was restricted to the digits `8' and `9', resulting in a training set size of approximately 11,800 images, with $20\%$ of this used as validation, and a test set of approximately 1,983 images. As illustrated in Figure \ref{Fig: Sample 8 and 9 from MNIST}, the primary structural distinction between these digits lies in the central-bottom region: the digit `8' contains a closed loop crossing the midline, whereas the `9' features a straight stroke or a curve that typically remains open in the center. Consequently, we hypothesized that the localized $7\times7$ regions in the central-bottom area of the image (regions 10, 11, 14, and 15) would exhibit the highest variable importance. Conversely, we expected pixels in the upper regions (shared loop feature) and the far-left and far-right edges (typically empty background padding) to demonstrate negligible importance.

To mitigate the issue of high pixel-to-pixel correlation, where a single pixel's value is highly predictive of its neighbors, we computed VI for groups of pixels rather than individual inputs. We analyzed the $28\times28$ pixel images at three levels of granularity:
\begin{enumerate}
    \item \textbf{Halves:} Top vs. Bottom ($14\times28$ pixels).
    \item \textbf{Quadrants:} Four disjoint regions ($14\times14$ pixels).
    \item \textbf{Fine-Grained Grid:} Sixteen disjoint regions ($7\times7$ pixels).
\end{enumerate}

\begin{figure}[htbp]
    \centering

    \includegraphics[width=0.7\linewidth]{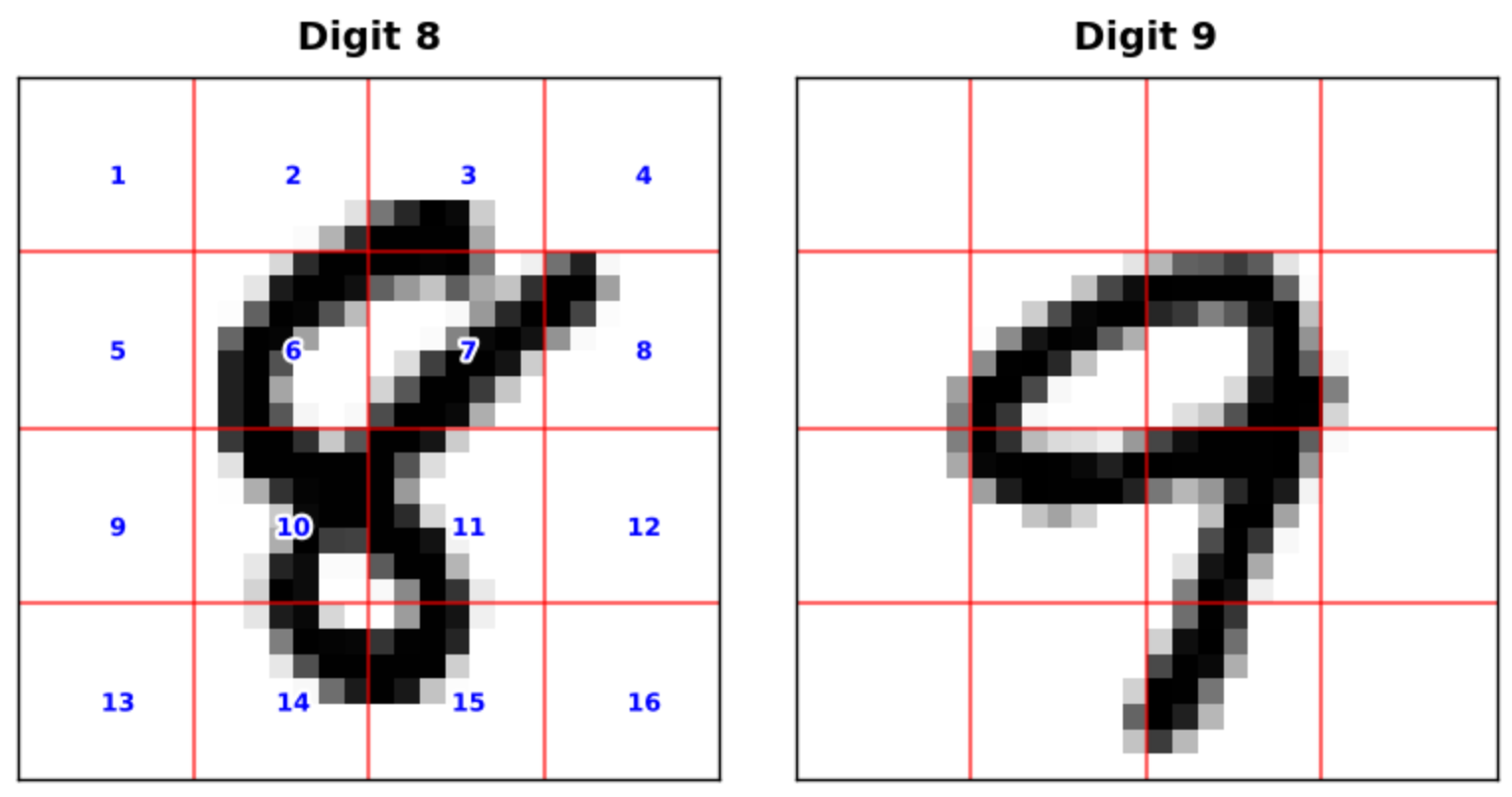}
    \caption{Representative samples of digits `8' and `9' from the test set. The overlay illustrates the $7\times7$ pixel grid boundaries used for the fine-grained analysis. Regions are labeled from 1 to 16.}
    \label{Fig: Sample 8 and 9 from MNIST}
\end{figure}

\subsubsection{Model Architecture and Implementation}
We reused the same setup and architecture that controlled the Type I error in the Empirical Power study above. We implemented a feedforward neural network with two hidden layers, each containing 50 neurons with ReLU activation functions, and a final output layer producing logits for binary classification. The model was trained using the L-BFGS optimizer with a learning rate of $\eta = 5\cdot 10^{-4}$. We used early stopping with a patience of 10 epochs to prevent overfitting, and the ridge penalty parameter $\lambda$ was selected from a logarithmic path of 10 values ranging from $10^{-4}$ to $10^{-2.25}$ via 3-fold cross-validation. 

For the VI comparisons, we defined the ``Dropout'' baseline as mean imputation, where pixels in a target region are replaced by their global mean intensity across the training set. 

After training, we applied the LazyVI framework to assess the importance of each pixel region. We computed a one-sided $p$-value for each region $i$ using the standard normal cumulative distribution function ($\Phi$):
\[
PV_i = 1 - \Phi\left(\frac{\text{VI}_i}{\text{SE}_i}\right),
\]
where $\text{VI}_i$ and $\text{SE}_i$ denote the variable importance estimate and its standard error, respectively.

\subsubsection{Results and Analysis}
The comparative variable results are visualized in Figure \ref{Fig: Heatmap of VI}. It is important to note that the color scale in the figure is normalized throughout the columns of a given row (granularity level) to facilitate comparison between methods, but not normalized across distinct rows. 

Figure \ref{Fig: Heatmap of p-values} shows a heatmap of the Lazy Training p-values for each region, adjusted with the Bonferroni correction. Each p-value was multiplied by 2, 4, or 16 for regions in the Top vs Bottom, Quadrant, and 16-Region granularities respectively, with the maximum adjusted p-value capped to 1. The plotted values are computed as $-log_{10}(\max\{.0001, \text{p-value}\})$. Note all p-values for the remainder of this section are reported with respect to the Bonferroni correction.

\begin{figure}[htbp]
    \centering

    \includegraphics[width=0.9\linewidth]{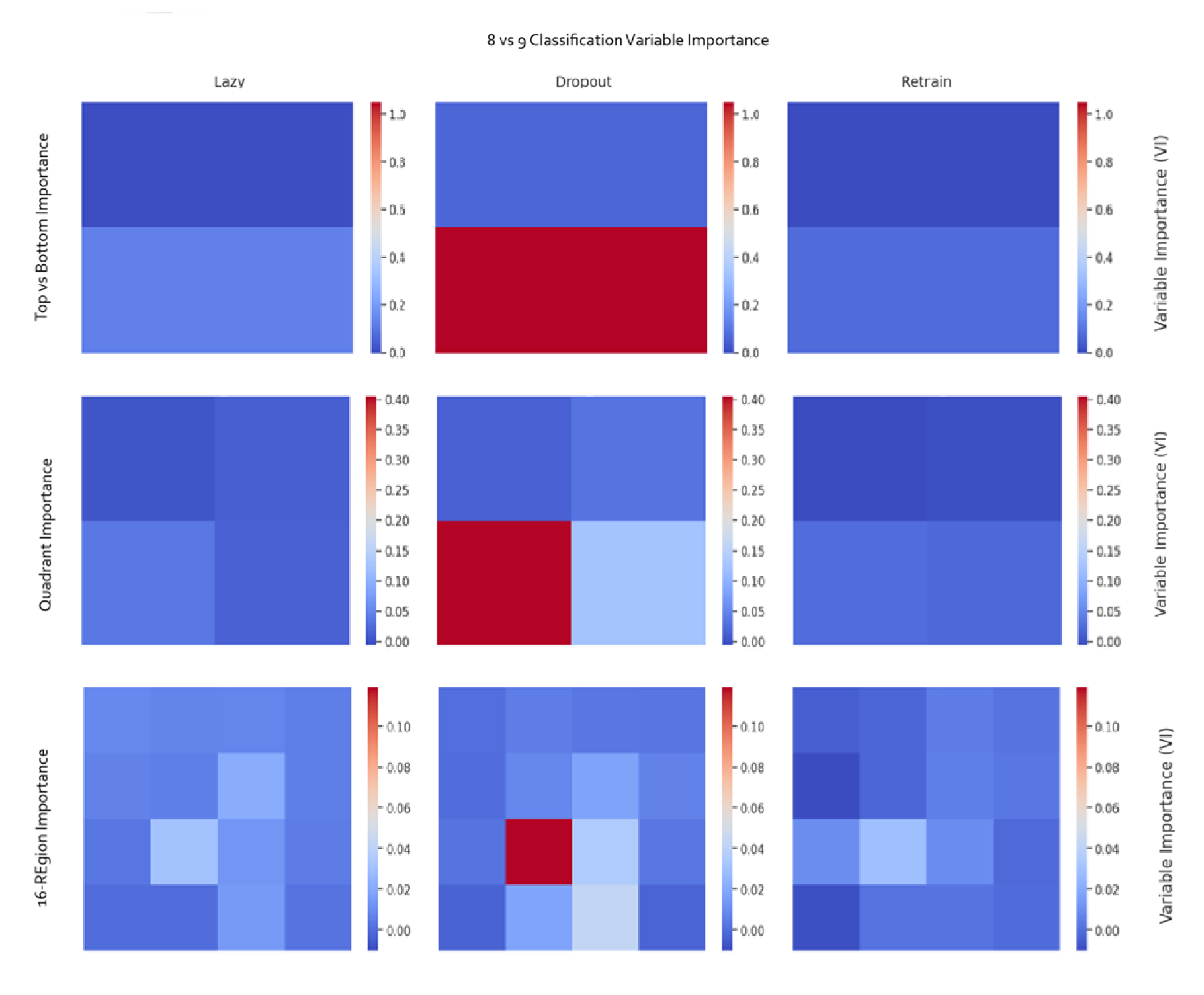}
    \caption{Heatmaps of Variable Importance (VI) across three granularity levels. Columns compare the estimation methods: Lazy Training, Dropout, and Retraining. Warmer colors indicate higher importance.}
    \label{Fig: Heatmap of VI}
\end{figure}

\begin{figure}[htbp]
    \centering

    \includegraphics[width=1.0\linewidth]{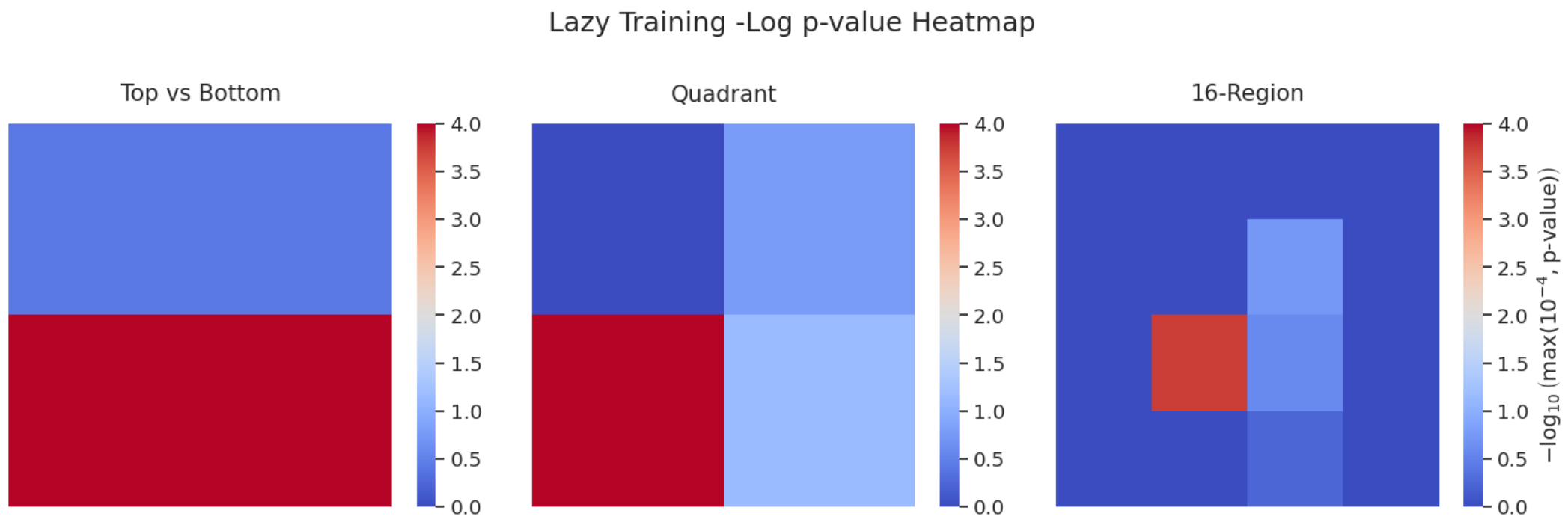}
   \caption{Heatmaps of Lazy Training $p$-values across three levels of granularity. The $p$-values are multiplied by 2, 4, or 16 for the Top vs Bottom, Quadrant, and 16-Region respectively for the Bonferroni correction with values capped at 1, then are transformed using $-\log_{10}(\max\{0.0001, p\})$ to enhance visual contrast and highlight significance. A score at approximately $ 1.3$ and above corresponds to $p\leq.05$.}
    \label{Fig: Heatmap of p-values}
\end{figure}

The results strongly support our central-bottom hypothesis and illustrate the value of fine-grained analysis. In the coarse-grained experiment, both halves were significant with $p < 0.001$, though the bottom half of the image was overwhelmingly dominant with a Lazy VI score of 0.1291 compared to 0.0338 for the top half. Once we moved to quadrants, we found the the three regions other than the Top-Left (TL) were significant, showing the significance of the top half was likely attributed to the Top-Right corner. In fact, this region had the largest VI estimate of the quadrants, unlike in Dropout and Retraining which had the Bottom-Left (BL) as the largest, which we attribute to randomness. The fine-grained analysis of the $7\times7$ pixels revealed that the most important region was region 10 ($p < 0.001$), which physically corresponds to the closure of the bottom loop of the `8' and its connection to the top loop, supporting our hypothesis.

The results strongly confirm our central-bottom hypothesis and illustrate the value of fine-grained analysis. In the coarse-grained experiments, the bottom half of the image was overwhelmingly dominant, yielding a Lazy VI score of 0.1291 ($p < 0.001$) compared to 0.0046 ($p \approx 0.388$) for the top half. Similarly, the quadrant analysis identified the Bottom-Left (BL) quadrant as the most critical, with a Lazy VI of 0.0360 ($p < 0.001$) while the other quadrants were not statistically significant.

Finally, the fine-grained $7\times7$ analysis revealed that this high "Bottom-Left" importance is likely driven almost entirely by the digit's internal structure, as Region 10 was had variable importance of .0307 ($p < 0.001$) while each other region in the Bottom-Left quadrant had p-values of 1. Physically, this region corresponds to the closure of the bottom loop of the `8' and its connection to the top loop, which is absent in the typical `9'. Furthermore, all regions on an edge demonstrated $p$-values of $1$, with the exception of region 15 ($p= .567$), correctly identifying the empty background padding as uninformative.
% \begin{itemize}
%     \item \textbf{The Critical Core (Regions 10 \& 11):} The highest importance across all regions in the 16-Region granularity was found in Region 10 (Lazy VI = 0.0549, $p < 0.001$), followed by Region 11 (Lazy VI = 0.0206, $p = 0.0128$). Physically, these correspond to the closure of the bottom loop of the `8' and its connection to the top loop. These were also the only regions with $p < .05$.
%     \item \textbf{The Empty Corner (Region 13):} In contrast, Region 13—the actual bottom-left corner of the image—yielded a negligible VI of -0.0019 ($p =1$). Regions 9 and 14 from the bottom left quadrant were similarly found to be unsignificant with $p =1$. Despite the high importance of the bottom-left quadrant, the only significant region in the 16-region analysis of the quadrant was found to be region 10.
%     \item \textbf{Image Edges:} Furthermore, the regions on an edge generally demonstrated $p$-values $> 0.9$, with the exception of regions 8 ($p= .1696$) and 15 ($p= .0768$), correctly identifying the empty background padding as uninformative.
% \end{itemize}
% This distinction explains why the $14\times14$ BL quadrant score was high (it contains Region 10) but imprecise. The fine-grained analysis demonstrated the ability of the neural networks to isolate the topology of the digit from the empty background padding.

Comparing the three estimation methodologies, we observed two key trends regarding magnitude and consistency:
\begin{enumerate}
    \item \textbf{Magnitude Estimation and Overestimation:} While the LazyVI framework closely matched the relative regional rankings predicted by Mean Imputation (Dropout), the magnitudes of LazyVI tracked the ``ground truth'' established by Retraining much more accurately. Dropout consistently overestimated Variable Importance, particularly when the dropped region was important. In the ``Top vs Bottom'' experiment, Dropout estimated the Bottom region's importance at 1.050, which is approximately 14 times the ``ground-truth'' Retraining value of 0.075. Meanwhile, the Lazy Training estimate was 0.129, only 1.72 times the Retraining baseline. For the smaller, less important Region 10 subset of the bottom half, we observed an overestimation factor of times 4.10 for Dropout and 1.07 for Lazy Training of Retraining's VI estimate of .029. In general, we noted that both Lazy Training and Dropout's overestimation increases for more significant regions, but inflates disproportionately for Dropout while Lazy Training is minimally affected.
    \item \textbf{Consistency Across Granularity:} As the subset size decreased (from halves to $7\times7$ squares), the estimated VI magnitudes decreased across all methods, as expected. However, the relative ranking between the regions remained consistent, and LazyVI successfully identified the main discriminative features with high fidelity—achieving results comparable to Retraining.
\end{enumerate}

\subsubsection{Runtime}

In Table \ref{tab:time_results}, we show the runtimes of training the full model, as well as Retraining and Lazy Training (not including the full model training time)  for each region. Most of the time taken by Lazy Training is from the 30 trainings needed to choose a ridge penalty value, so for a more fair comparison, we showed the time taken for Lazy Training to run the final fitting. We found that the final fitting time for Lazy Training is consistently around 2 to 4 times faster than Retraining.
% Note for the total runtime of Lazy Training, it increase with larger regions and with significant regions, which is to be expected as the change in network parameters would be larger when we remove larger and more significant regions, leading the optimizer to take longer to converge.
\\\\Because the runtime is dependent on the optimization techniques used, we expect the total runtime of Lazy Training could be optimized further. Warm starts for the final fitting using the solution from the cross-validation of the selected ridge penalty could speed up the final fitting time, but was excluded to make the comparison with Retraining fair. Warm starts during the cross-validation section and different optimizers such as Stochastic L-BFGS would also likely significantly improve runtime for the penalty selection process.

% ~\ref{tab:lazyvi_results}
\begin{table}[htbp]
    \centering
    \small
    \caption{Runtime Comparisons For Training Methods}
    \label{tab:time_results}
    \begin{tabular}{llrrrr}
        \toprule
        \textbf{Configuration} & \textbf{Region} &  \textbf{\begin{tabular}{@{}r@{}}Final Fitting \\ Lazy Time (s)\end{tabular}} & \textbf{Retrain Time (s)} & \textbf{Full Model Time (s)} \\
        \midrule
        
        \multirow{16}{*}{\textbf{16 Regions}} 
        & 1   & 4.78 & 11.41 & \multirow{16}{*}{13.10} \\
        & 2   & 4.48 & 14.82 & \\
        & 3   & 4.56 & 15.10 & \\
        & 4   & 4.48 & 14.97 & \\
        & 5   & 4.55 & 12.03 & \\
        & 6   & 4.58 & 17.32 & \\
        & 7   & 4.55 & 15.31 & \\
        & 8   & 4.24 & 15.83 & \\
        & 9   & 4.54 & 12.38 & \\
        & 10  & 4.48 & 15.64 & \\
        & 11  & 4.55 & 14.82 & \\
        & 12  & 4.50 & 13.28 & \\
        & 13  & 4.53 & 12.01 & \\
        & 14  & 4.50 & 14.78 & \\
        & 15  & 4.51 & 17.23 & \\
        & 16  & 4.51 & 14.00 & \\
        \midrule
        
        \multirow{4}{*}{\textbf{4 Regions}} 
        & TL &  4.64 & 18.84 & \multirow{4}{*}{12.81} \\
        & TR &  4.56 & 17.43 & \\
        & BL &  4.51 & 14.33 & \\
        & BR &  4.56 & 25.32 & \\
        \midrule
        
        \multirow{2}{*}{\textbf{2 Regions}} 
        & Top     & 4.62 & 19.37 & \multirow{2}{*}{13.08} \\
        & Bottom  & 4.52 & 21.57 & \\
        \bottomrule
    \end{tabular}
\end{table}

\subsection{Application to Detect Genes Associated with Alzheimer's Disease (AD)}
Alzheimer's disease (AD) is one of the most common neurodegenerative diseases, and it is influenced heavily through genetic components \citep{karch2014alzheimer, sims2020multiplex}. Therefore, it is essential to detect genes significantly related to AD for targeted treatments. As an application of the proposed method, we performed a genetic association study based on the gene expression data from the Alzheimer's Disease Neuroimaging Initiative (ADNI). 

The disease status in the ADNI data has three categories: cognitive model, mild cognitive impairment, and Alzheimer's disease. To apply our LazyVI for binary classification to the ADNI dataset, we combined mild cognitive impairment and Alzheimer's disease into one group so that the classifier can detect genes that potentially relate to any potential neural degeneration. We then merged data from individuals having both gene expression information and disease status. A total of 521 individuals and 15,837 gene expressions were obtained. 

We then performed a variable importance test for each of the 15,837 gene expression variables. In other words, we conducted 15,837 hypothesis tests, where the null hypothesis states that the gene is not important and the alternative hypothesis states that the gene is important. To fit the model under the alternative hypothesis, we included age, gender, years of education, the number of \textit{APOE4} alleles, and the expression level of the gene of interest as input features. Under the null hypothesis, we replaced the gene expression level with its sample mean and applied our LazyVI algorithm. Following the simulation setup, we fitted the data using a deep ReLU network (two hidden layers, 50 neurons per layer). The regularization parameter $\lambda$ and the hyperparameters in the L-BFGS algorithm were kept identical to those used in the simulation studies. After obtaining the $p$-values for all genes, we ranked them from smallest to largest. Table \ref{tab: top 10 significant genes} summarizes the 10 most significant genes identified by logistic regression and the LazyVI algorithm.

\begin{table}[htbp]
    \centering
    \caption{The top 10 significant genes detected from the logistic regression and LazyVI for choosing the regularization parameter $\lambda$ and hyperparameters in the L-BFGS algorithm as described in the simulation studies}
    \label{tab: top 10 significant genes}
    \begin{tabular}{c|c}
        \toprule
       \textbf{Logistic Regression}  & \textbf{LazyVI} \\
       \midrule
        \textit{ORC6} &	\textit{MT1H}\\
        \textit{SPATA7} & \textit{COMMD6}\\
        \textit{GPAT2}	& \textit{LRFN3}\\
        \textit{OR52B2} &   \textit{GLRX5} \\
        \textit{KRTAP6-3} &    \textit{SGCB}\\
        \textit{TAS2R10}    &	\textit{TCF19}\\
        \textit{SLITRK6}    & 	\textit{NKX2-6}	\\
        \textit{ZNF503} &	\textit{PIGC}\\
        \textit{KCTD8}  &	\textit{PSTPIP1} \\
        \textit{OR52A5} &	\textit{MAPK11} \\
        \bottomrule
    \end{tabular}
\end{table}

It is not surprising that deep neural network-based methods identify different genes from those identified by classical statistical methods. For instance, \citet{shen2024exploration} applied a goodness-of-fit test based on deep ReLU neural networks to detect genes associated with quantitative traits related to AD, and the genes identified by the deep neural network-based methods differed substantially from those identified by linear models. On the other hand, it is worth noting that the most significant genes identified by each method have biological relevance supported by previous studies. The origin recognition complex (ORC), which controls the initiation of DNA replication, consists of six subunits, one of which is \textit{ORC6}. \citet{arendt2007linking} showed that ORC subunits are involved in AD pathology. Among the significant genes identified by the LazyVI algorithm, \textit{MT1H}, which belongs to the metallothionein family, has been uncovered as a hub gene through a network analysis and is suspected to be related to AD development for long time \citep{liang2018application}. In addition, \textit{MAPK11} is a member of the mitogen-activated protein kinase (MAPK) family. As mentioned in \citet{zhao2002map}, the MAPK family regulates phosphorylation of the microtubule-associated protein tau and processing of the amyloid protein $\beta$, and both events are critical to the pathophysiology of AD. These findings suggest that our method successfully identifies genes involved in biological pathways known to contribute to AD pathogenesis while also revealing potentially novel candidate genes for future investigation.

\section{Discussions and Conclusions}
% Summary of the paper
In this paper, we propose a framework for detecting important input features using lazy-trained deep neural network features in binary classification problems. We rigorously establish the asymptotic normality of the proposed test statistics. Through simulation studies and experiments on the MNIST dataset, the LazyVI framework successfully identifies important features. Moreover, compared with dropout-based methods, LazyVI generally does not overestimate variable importance scores. At the same time, the variable importance scores estimated by LazyVI are approximately the same as those obtained from retraining-based methods. Nevertheless, LazyVI is computationally more efficient than retraining-based approaches.

% Possible extensions
The universal approximation property of neural networks \citep{hornik1989multilayer, cybenko1989approximation, yarotsky2017error, yarotsky2020phase, schmidt2020nonparametric} provides a powerful alternative for function estimation in nonparametric statistical models (e.g., nonparametric regression). However, conducting statistical inference based on a fitted neural network for the purpose of detecting important input features remains a challenging problem. Although the proposed LazyVI framework primarily focuses on binary classification, we believe it has the potential to be extended to other types of outcome variables. A natural generalization is to replace the binary cross-entropy loss with the multiclass cross-entropy loss, thereby accommodating multiclass classification problems. Indeed, we formulate the problem using a Bernoulli distribution in Section~\ref{Sec: Problem Setup} to connect our framework with the classical formulation of generalized linear models (GLMs) \citep{mccullagh2019generalized}, where the conditional distribution of $Y \mid \mbf{X}$ is typically assumed to belong to the exponential dispersion family:
$$
Y|\mbf{X}\sim\exp\left\{\frac{y\theta-b(\theta)}{a(\phi)}+c(y,\phi)\right\},
$$
where $\phi$ is the dispersion parameter and $a(\cdot)$, $b(\cdot)$, and $c(\cdot,\cdot)$ are known functions. The relationship between $Y$ and $\mbf{X}$ is specified through a link function $g$ such that
\begin{equation}\label{Eq: link function}
g\big(b'(\theta)\big) = \mbf{X}^\top \mbf{\beta}.
\end{equation}
The linear predictor $\mbf{X}^\top \mbf{\beta}$ can be naturally generalized to a nonlinear function $f(\mbf{X})$ by replacing it in (\ref{Eq: link function}). Consequently, within the framework described in Section~\ref{Sec: Problem Setup}, we may define the loss function as the negative log-likelihood
$$
\ell(\cdot,y)=\frac{1}{a(\phi)}\left[-y(g\circ b')^{-1}(\cdot)+b((g\circ b')^{-1}(\cdot))+c(y,\phi)\right].
$$
If this loss function satisfies Conditions 1–3 in Proposition \ref{Prop: Properties of the 0-1 logistic loss}, then the LazyVI framework can be applied in this more general setting.

% Limitations
% 1. The current setting does not apply to the overparameterized neural networks
% 2. Only considered deep neural networks.
% 3. Applications to large-scale dataset (e.g. UK biobank) may be a challenge.
We would also like to highlight several limitations of the current method. (1) In this paper, we focus primarily on the function class of fully connected deep neural networks. It would be worthwhile to investigate whether the framework can be extended to more sophisticated architectures, including convolutional neural networks \citep{pfeifer1989generalization}, long short-term memory networks \citep{hochreiter1997long}, and transformers with attention mechanisms \citep{vaswani2017attention}. (2) As indicated by Assumption~(\ref{Eq: Condition on W}), our analysis primarily considers the regime in which the number of network parameters does not grow too rapidly relative to the sample size. However, due to the empirical success of large-scale deep learning models, overparameterized neural networks have attracted substantial attention in recent years. It is therefore important to investigate whether the LazyVI framework can be extended to the overparameterized setting. For the framework developed in this paper, the main technical bottleneck underlying these two limitations concerns whether a sufficiently tight upper bound on the entropy numbers of the function class can be established and how this bound scales with the covering radius $\varepsilon$. The volume-based argument we employ yields a growth rate of order $\mathcal{O}\left(\log \frac{1}{\varepsilon}\right)$, but at the cost of dependence on the total number of network parameters. This dependence ultimately affects the convergence rate derived via local Rademacher complexity arguments. In contrast, applying Maurey’s sparsification lemma \citep{pisier1981remarques} yields size-independent upper bounds on entropy numbers; however, the resulting growth rate is of order $\mathcal{O}\left(\varepsilon^{-2}\right)$, as shown in \citet{zhang2004statistical} and \citet{bartlett2017spectrally}. Moreover, \citet{golowich2018size} demonstrate that the Rademacher complexity of neural networks with parameter matrices of bounded Schatten norm is lower bounded by $\Omega(n^{-1/2})$, which makes it challenging to derive faster convergence rates even when using local Rademacher complexity. (3) Although, in theory, Newton’s method can be used to compute the estimators arising from the lazy training procedure, this becomes computationally prohibitive for very large sample sizes, since the neural tangent kernel matrix has dimension $n \times n$. For example, the UK Biobank dataset \citep{bycroft2018uk} contains over 500{,}000 participants. Storing such large matrices requires substantial memory, and inverting the corresponding Hessian matrices further increases the computational burden. In our implementation, we therefore adopt a quasi-Newton method to reduce computational cost. Nevertheless, applying the LazyVI framework to large-scale datasets such as UK Biobank requires the development of more scalable algorithms. Addressing these limitations will be the focus of future work.

To conclude, our work demonstrates that statistically valid inference for feature importance can be carried out in deep neural networks by leveraging the lazy training regime and likelihood-based formulations. By connecting neural tangent kernel approximations with classical tools from empirical process theory and generalized linear models, our framework provides a principled bridge between deep neural networks and traditional statistical inference. We hope this perspective encourages further research on scalable and theoretically grounded inference procedures for deep learning models beyond predictive performance alone.

\section*{Software and Appendices}
The proposed LazyVI framework for detecting important features for binary classification was implemented using Python packages. The codes and simulated data are available at \url{https://github.com/SxxMichael/DNN-LazyVI-Binary-Classification}. Additional technical details are available in the Appendices. 

\section*{Acknowledgements}
This research is supported in part by NSF Grant DMS-2447229.

\bibliographystyle{plainnat}
\bibliography{reference}

@article{hsu2012tail,
  title={A tail inequality for quadratic forms of subgaussian random vectors},
  author={Hsu, Daniel and Kakade, Sham and Zhang, Tong},
  journal={Electronic Communications in Probability},
  volume={17},
  pages={1-6},
  year={2012}
}

@article{hoeffding1994probability,
  title={Probability inequalities for sums of bounded random variables},
  author={Hoeffding, Wassily},
  journal={The collected works of Wassily Hoeffding},
  pages={409--426},
  year={1994},
  publisher={Springer}
}

@article{mirsky1975trace,
  title={A trace inequality of John von Neumann},
  author={Mirsky, Leon},
  journal={Monatshefte f{\"u}r mathematik},
  volume={79},
  number={4},
  pages={303--306},
  year={1975},
  publisher={Springer}
}

@inproceedings{gao2022lazy,
  title={Lazy estimation of variable importance for large neural networks},
  author={Gao, Yue and Stevens, Abby and Raskutti, Garvesh and Willett, Rebecca},
  booktitle={International Conference on Machine Learning},
  pages={7122--7143},
  year={2022},
  organization={PMLR}
}

@article{bartlett2005local,
author = {Peter L. Bartlett and Olivier Bousquet and Shahar Mendelson},
title = {{Local Rademacher complexities}},
volume = {33},
journal = {The Annals of Statistics},
number = {4},
pages = {1497 -- 1537},
year = {2005}
}

@article{aronszajn1950theory,
  title={Theory of reproducing kernels},
  author={Aronszajn, Nachman},
  journal={Transactions of the American mathematical society},
  volume={68},
  number={3},
  pages={337--404},
  year={1950}
}

@book{mohri2018foundation,
  title={Foundations of machine learning},
  author={Mehryar Mohri and Afshin Rostamizadeh and Ameet Talwalkar},
  year={2018},
  publisher={MIT press}
}

@book{ledoux2013probability,
  title={Probability in Banach Spaces: isoperimetry and processes},
  author={Ledoux, Michel and Talagrand, Michel},
  year={2013},
  publisher={Springer Science \& Business Media}
}

@book{nocedal2006optimization,
    title = {Numerical Optimization},
    author = {Nocedal, Jorge and Wright, Stephen J.},
    year = {2006},
    publisher = {Springer Series in Operations Research}
}

@article{bartlett2002rademacher,
  title={Rademacher and gaussian complexities: Risk bounds and structural results},
  author={Bartlett, Peter L and Mendelson, Shahar},
  journal={Journal of machine learning research},
  volume={3},
  number={Nov},
  pages={463--482},
  year={2002}
}

@article{williamson2023general,
  title={A general framework for inference on algorithm-agnostic variable importance},
  author={Williamson, Brian D and Gilbert, Peter B and Simon, Noah R and Carone, Marco},
  journal={Journal of the American Statistical Association},
  volume={118},
  number={543},
  pages={1645--1658},
  year={2023},
  publisher={Taylor \& Francis}
}

@article{bartlett2017spectrally,
  title={Spectrally-normalized margin bounds for neural networks},
  author={Bartlett, Peter L and Foster, Dylan J and Telgarsky, Matus J},
  journal={Advances in neural information processing systems},
  volume={30},
  year={2017}
}

@inproceedings{nair2010rectified,
  title={Rectified linear units improve restricted boltzmann machines},
  author={Nair, Vinod and Hinton, Geoffrey E},
  booktitle={Proceedings of the 27th international conference on machine learning (ICML-10)},
  pages={807--814},
  year={2010}
}

@article{farrell2021deep,
  title={Deep neural networks for estimation and inference},
  author={Farrell, Max H and Liang, Tengyuan and Misra, Sanjog},
  journal={Econometrica},
  volume={89},
  number={1},
  pages={181--213},
  year={2021},
  publisher={Wiley Online Library}
}

@article{mendelson2002improving,
  title={Improving the sample complexity using global data},
  author={Mendelson, Shahar},
  journal={IEEE transactions on Information Theory},
  volume={48},
  number={7},
  pages={1977--1991},
  year={2002},
  publisher={IEEE}
}

@article{shen2025consistency,
  title={Consistency and Rate of Convergence for Deep ReLU Neural Networks},
  author={Shen, Xiaoxi and Espinoza, Jesus},
  journal={Journal of Statistical Theory and Practice},
  volume={19},
  number={2},
  pages={33},
  year={2025},
  publisher={Springer}
}

@misc{ma2022lecture,
  title={Lecture notes for machine learning theory (CS229M/STATS214)},
  author={Ma, Tengyu},
  year={2022}
}

@book{Vershynin_2018, 
    title={High-Dimensional Probability: An Introduction with Applications in Data Science}, 
    publisher={Cambridge University Press}, 
    author={Vershynin, Roman}, 
    year={2018}
}

@article{misiakiewicz2024six,
  title={Six lectures on linearized neural networks},
  author={Misiakiewicz, Theodor and Montanari, Andrea},
  journal={Journal of Statistical Mechanics: Theory and Experiment},
  volume={2024},
  number={10},
  pages={104006},
  year={2024},
  publisher={IOP Publishing}
}

@article{oymak2020toward,
  title={Toward moderate overparameterization: Global convergence guarantees for training shallow neural networks},
  author={Oymak, Samet and Soltanolkotabi, Mahdi},
  journal={IEEE Journal on Selected Areas in Information Theory},
  volume={1},
  number={1},
  pages={84--105},
  year={2020},
  publisher={IEEE}
}

@article{zou2018stochastic,
  title={Stochastic gradient descent optimizes over-parameterized deep relu networks},
  author={Zou, Difan and Cao, Yuan and Zhou, Dongruo and Gu, Quanquan},
  journal={arXiv preprint arXiv:1811.08888},
  year={2018}
}

@article{shen2023asymptotic,
  title={Asymptotic properties of neural network sieve estimators},
  author={Shen, Xiaoxi and Jiang, Chang and Sakhanenko, Lyudmila and Lu, Qing},
  journal={Journal of nonparametric statistics},
  volume={35},
  number={4},
  pages={839--868},
  year={2023},
  publisher={Taylor \& Francis}
}

@article{horel2020significance,
  title={Significance tests for neural networks},
  author={Horel, Enguerrand and Giesecke, Kay},
  journal={Journal of Machine Learning Research},
  volume={21},
  number={227},
  pages={1--29},
  year={2020}
}

@article{schmidt2020nonparametric,
    author = {Schmidt-Hieber, Johannes},
    title = {Nonparametric regression using deep neural networks with ReLU activation function},
    journal = {The Annals of Statistics},
    volume = {48},
    number = {4},
    year = {2020}
}

@book{van1996weak,
  title={Weak convergence},
  author={Van Der Vaart, Aad W and Wellner, Jon A},
  title={Weak convergence and empirical processes: with applications to statistics},
  year={1996},
  publisher={Springer}
}

@article{ossiander1987central,
  title={A central limit theorem under metric entropy with L 2 bracketing},
  author={Ossiander, Mina},
  journal={The Annals of Probability},
  pages={897--919},
  year={1987},
  publisher={JSTOR}
}

@misc{zhang2025unified,
      title={Towards Unified Attribution in Explainable AI, Data-Centric AI, and Mechanistic Interpretability}, 
      author={Shichang Zhang and Tessa Han and Usha Bhalla and Himabindu Lakkaraju},
      year={2025},
      eprint={2501.18887},
      archivePrefix={arXiv},
      primaryClass={cs.LG},
      url={https://arxiv.org/abs/2501.18887}, 
}

@article{petsiuk2018rise,
  title={Rise: Randomized input sampling for explanation of black-box models},
  author={Petsiuk, Vitali and Das, Abir and Saenko, Kate},
  journal={arXiv preprint arXiv:1806.07421},
  year={2018}
}

@article{lundberg2017unified,
  title={A unified approach to interpreting model predictions},
  author={Lundberg, Scott M and Lee, Su-In},
  journal={Advances in neural information processing systems},
  volume={30},
  year={2017}
}

@article{smilkov2017smoothgrad,
  title={Smoothgrad: removing noise by adding noise},
  author={Smilkov, Daniel and Thorat, Nikhil and Kim, Been and Vi{\'e}gas, Fernanda and Wattenberg, Martin},
  journal={arXiv preprint arXiv:1706.03825},
  year={2017}
}

@article{simonyan2013deep,
  title={Deep inside convolutional networks: Visualising image classification models and saliency maps},
  author={Simonyan, Karen and Vedaldi, Andrea and Zisserman, Andrew},
  journal={arXiv preprint arXiv:1312.6034},
  year={2013}
}

@article{shen2021goodness,
  title={A goodness-of-fit test based on neural network sieve estimators},
  author={Shen, Xiaoxi and Jiang, Chang and Sakhanenko, Lyudmila and Lu, Qing},
  journal={Statistics \& probability letters},
  volume={174},
  pages={109100},
  year={2021},
  publisher={Elsevier}
}

@article{shen2024exploration,
  title={An exploration of testing genetic associations using goodness-of-fit statistics based on deep ReLU neural networks},
  author={Shen, Xiaoxi and Wang, Xiaoming},
  journal={Frontiers in Systems Biology},
  volume={4},
  pages={1460369},
  year={2024},
  publisher={Frontiers Media SA}
}

@article{chizat2019lazy,
  title={On lazy training in differentiable programming},
  author={Chizat, Lenaic and Oyallon, Edouard and Bach, Francis},
  journal={Advances in neural information processing systems},
  volume={32},
  year={2019}
}

@inproceedings{mandel2024permutation,
  title={Permutation-based hypothesis testing for neural networks},
  author={Mandel, Francesca and Barnett, Ian},
  booktitle={Proceedings of the AAAI Conference on Artificial Intelligence},
  volume={38},
  number={13},
  pages={14306--14314},
  year={2024}
}

@article{shen2022sieve,
  title={A sieve quasi-likelihood ratio test for neural networks with applications to genetic association studies},
  author={Shen, Xiaoxi and Jiang, Chang and Sakhanenko, Lyudmila and Lu, Qing},
  journal={arXiv preprint arXiv:2212.08255},
  year={2022}
}

@article{dai2022significance,
  title={Significance tests of feature relevance for a black-box learner},
  author={Dai, Ben and Shen, Xiaotong and Pan, Wei},
  journal={IEEE transactions on neural networks and learning systems},
  volume={35},
  number={2},
  pages={1898--1911},
  year={2022},
  publisher={IEEE}
}

@book{mccullagh2019generalized,
  title={Generalized linear models},
  author={McCullagh, Peter and Nelder, John A.},
  year={1999},
  publisher={Chapman \& Hall/CRC}
}

@article{bietti2019inductive,
  title={On the inductive bias of neural tangent kernels},
  author={Bietti, Alberto and Mairal, Julien},
  journal={Advances in Neural Information Processing Systems},
  volume={32},
  year={2019}
}

@article{bietti2020deep,
  title={Deep equals shallow for ReLU networks in kernel regimes},
  author={Bietti, Alberto and Bach, Francis},
  journal={arXiv preprint arXiv:2009.14397},
  year={2020}
}

@article{li2024eigenvalue,
  title={On the eigenvalue decay rates of a class of neural-network related kernel functions defined on general domains},
  author={Li, Yicheng and Yu, Zixiong and Chen, Guhan and Lin, Qian},
  journal={Journal of Machine Learning Research},
  volume={25},
  number={82},
  pages={1--47},
  year={2024}
}

@article{du2018gradient,
  title={Gradient descent provably optimizes over-parameterized neural networks},
  author={Du, Simon S and Zhai, Xiyu and Poczos, Barnabas and Singh, Aarti},
  journal={arXiv preprint arXiv:1810.02054},
  year={2018}
}

@article{lecun2002gradient,
  title={Gradient-based learning applied to document recognition},
  author={LeCun, Yann and Bottou, L{\'e}on and Bengio, Yoshua and Haffner, Patrick},
  journal={Proceedings of the IEEE},
  volume={86},
  number={11},
  pages={2278--2324},
  year={2002},
  publisher={Ieee}
}

@article{hornik1989multilayer,
  title={Multilayer feedforward networks are universal approximators},
  author={Hornik, Kurt and Stinchcombe, Maxwell and White, Halbert},
  journal={Neural networks},
  volume={2},
  number={5},
  pages={359--366},
  year={1989},
  publisher={Elsevier}
}

@article{cybenko1989approximation,
  title={Approximation by superpositions of a sigmoidal function},
  author={Cybenko, George},
  journal={Mathematics of control, signals and systems},
  volume={2},
  number={4},
  pages={303--314},
  year={1989},
  publisher={Springer}
}

@article{yarotsky2017error,
  title={Error bounds for approximations with deep ReLU networks},
  author={Yarotsky, Dmitry},
  journal={Neural networks},
  volume={94},
  pages={103--114},
  year={2017},
  publisher={Elsevier}
}

@article{yarotsky2020phase,
  title={The phase diagram of approximation rates for deep neural networks},
  author={Yarotsky, Dmitry and Zhevnerchuk, Anton},
  journal={Advances in neural information processing systems},
  volume={33},
  pages={13005--13015},
  year={2020}
}

@article{vaswani2017attention,
  title={Attention is all you need},
  author={Vaswani, Ashish and Shazeer, Noam and Parmar, Niki and Uszkoreit, Jakob and Jones, Llion and Gomez, Aidan N and Kaiser, {\L}ukasz and Polosukhin, Illia},
  journal={Advances in neural information processing systems},
  volume={30},
  year={2017}
}

@article{pfeifer1989generalization,
  title={Generalization and network design strategies},
  author={LeCun, Yann},
  journal={Connectionism in perspective},
  pages={143--155},
  year={1989},
  publisher={Elsevier}
}

@article{hochreiter1997long,
  title={Long short-term memory},
  author={Hochreiter, Sepp and Schmidhuber, J{\"u}rgen},
  journal={Neural computation},
  volume={9},
  number={8},
  pages={1735--1780},
  year={1997},
  publisher={MIT press}
}

@article{pisier1981remarques,
  title={Remarques sur un r{\'e}sultat non publi{\'e} de B. Maurey},
  author={Pisier, Gilles},
  journal={S{\'e}minaire d'Analyse fonctionnelle (dit" Maurey-Schwartz")},
  pages={1--12},
  year={1981}
}

@article{zhang2004statistical,
  title={Statistical analysis of some multi-category large margin classification methods},
  author={Zhang, Tong},
  journal={Journal of Machine Learning Research},
  volume={5},
  number={Oct},
  pages={1225--1251},
  year={2004}
}

@inproceedings{golowich2018size,
  title={Size-independent sample complexity of neural networks},
  author={Golowich, Noah and Rakhlin, Alexander and Shamir, Ohad},
  booktitle={Conference on learning theory},
  pages={297--299},
  year={2018},
  organization={PMLR}
}

@article{bycroft2018uk,
  title={The UK Biobank resource with deep phenotyping and genomic data},
  author={Bycroft, Clare and Freeman, Colin and Petkova, Desislava and Band, Gavin and Elliott, Lloyd T and Sharp, Kevin and Motyer, Allan and Vukcevic, Damjan and Delaneau, Olivier and O’Connell, Jared and others},
  journal={Nature},
  volume={562},
  number={7726},
  pages={203--209},
  year={2018},
  publisher={Nature Publishing Group UK London}
}

@article{karch2014alzheimer,
  title={Alzheimer’s disease genetics: from the bench to the clinic},
  author={Karch, Celeste M and Cruchaga, Carlos and Goate, Alison M},
  journal={Neuron},
  volume={83},
  number={1},
  pages={11--26},
  year={2014},
  publisher={Elsevier}
}

@article{sims2020multiplex,
  title={The multiplex model of the genetics of Alzheimer’s disease},
  author={Sims, Rebecca and Hill, Matthew and Williams, Julie},
  journal={Nature neuroscience},
  volume={23},
  number={3},
  pages={311--322},
  year={2020},
  publisher={Nature Publishing Group US New York}
}

@article{arendt2007linking,
  title={Linking cell-cycle dysfunction in Alzheimer's disease to a failure of synaptic plasticity},
  author={Arendt, Thomas and Br{\"u}ckner, Martina K},
  journal={Biochimica et Biophysica Acta (BBA)-Molecular Basis of Disease},
  volume={1772},
  number={4},
  pages={413--421},
  year={2007},
  publisher={Elsevier}
}

@article{jacot2018neural,
  title={Neural tangent kernel: Convergence and generalization in neural networks},
  author={Jacot, Arthur and Gabriel, Franck and Hongler, Cl{\'e}ment},
  journal={Advances in neural information processing systems},
  volume={31},
  year={2018}
}

@article{zhao2002map,
  title={MAP kinase signaling cascade dysfunction specific to Alzheimer's disease in fibroblasts},
  author={Zhao, Wei-Qin and Ravindranath, Lakshmi and Mohamed, Ali S and Zohar, Ofer and Chen, Gina H and Lyketsos, Constantine G and Etcheberrigaray, Ren{\'e} and Alkon, Daniel L},
  journal={Neurobiology of disease},
  volume={11},
  number={1},
  pages={166--183},
  year={2002},
  publisher={Elsevier}
}

@article{liang2018application,
  title={Application of weighted gene co-expression network analysis to explore the key genes in Alzheimer’s disease},
  author={Liang, Jia-Wei and Fang, Zheng-Yu and Huang, Yong and Liuyang, Zhen-yu and Zhang, Xiao-Lin and Wang, Jing-Lin and Wei, Hui and Wang, Jian-Zhi and Wang, Xiao-Chuan and Zeng, Ji and others},
  journal={Journal of Alzheimer’s Disease},
  volume={65},
  number={4},
  pages={1353--1364},
  year={2018},
  publisher={SAGE Publications Sage UK: London, England}
}

\newpage
\appendix
\section*{Appendix}
In this appendix, we provide detailed proofs of the results in the main text. We organize the appendix as follows:
\begin{itemize}
    \item Appendix \ref{appendix: logistic loss} contains the proofs of the results related to the 0-1 logistic loss.

    \item  Appendix \ref{appendix: DNN} includes all the theoretical results, mainly in Section \ref{Sec: Convergence Rate of DNN}, regarding the class of deep ReLU neural networks. 

    \item Appendix \ref{appendix: distance between networks} consists of the technical details about the distance between the linearized deep ReLU network $\tilde{h}_{\mbf{\theta}_f+\Delta\mbf{\theta}_S}$ and the deep ReLU network $h_{\mbf{\theta}_f+\Delta\mbf{\theta}_S}$ as discussed in Section \ref{Sec: Bounding the difference}. 

    \item Appendix \ref{appendix: Convergence Rate of Linearized ReLU Network} includes the detailed proofs of the results in Section \ref{Sec: Estimation Error of h tilde} regarding the convergence rate of linearized ReLU network $\tilde{h}_{\mbf{\theta}_f+\Delta\mbf{\theta}_S}$.

    \item Appendix \ref{appendix: Auxiliary results} contains the proofs of some auxiliary results that are used in the theorems and lemmas of the main text.
\end{itemize}

\section{Properties for 0-1 Logistic Loss}\label{appendix: logistic loss}
\subsection{Proof of Proposition \ref{Prop: About f0}}
\begin{proof}
   Let $P_{Y|\mbf{X}}=\mrm{Ber}(\pi(\mbf{X}))$ and $\hat{P}_{Y|\mbf{X}}=\mrm{Ber}(\sigma(f(\mbf{X})))$ with $\sigma(\cdot)$ being the sigmoid function. Here $P_{Y|\mbf{X}}$ represents the underlying conditional distribution of $Y|\mbf{X}$ and $\hat{P}_{Y|\mbf{X}}$ represents the estimated conditional distribution of $Y|\mbf{X}$ based on our model. Then the KL divergence between these two distributions is
    \begin{align*}
        KL(P_{Y|\mbf{X}}\|\hat{P}_{Y|\mbf{X}}) & =\pi(\mbf{X})\log\frac{\pi(\mbf{X})}{\sigma(f(\mbf{X}))}+(1-\pi(\mbf{X}))\log\frac{1-\pi(\mbf{X})}{1-\sigma(f(\mbf{X}))}\\
        & =-\pi(\mbf{X})\log\frac{\sigma(f(\mbf{X}))}{1-\sigma(f(\mbf{X}))}+\log\frac{1}{1-\sigma(f(\mbf{X}))}+S(\pi(\mbf{X}))\\
        & =\mbb{E}\left[\left.-Yf(\mbf{X})+\log\left(1+e^{f(\mbf{X})}\right)\right|\mbf{X}\right]+S(\pi(\mbf{X})),
    \end{align*}
    where $S(\pi(\mbf{X}))=\pi(\mbf{X})\log\pi(\mbf{X})+(1-\pi(\mbf{X}))\log(1-\pi(\mbf{X}))$ does not depend on $f$. Therefore,
    \begin{align*}
        f_0 & \in\mrm{argmin}_{f\in\mcal{F}}\mbb{E}_{P_0}\left[-Yf(\mbf{X})+\log\left(1+e^{f(\mbf{X})}\right)\right]\\
            & =\mrm{argmin}_{f\in\mcal{F}}\mbb{E}_{\mbf{X}}\mbb{E}\left[\left.-Yf(\mbf{X})+\log\left(1+e^{f(\mbf{X})}\right)\right|\mbf{X}\right]\\
            & =\mrm{argmin}_{f\in\mcal{F}}\mbb{E}_{\mbf{X}}\left[KL(P_{Y|\mbf{X}}\|\hat{P}_{Y|\mbf{X}})\right].
    \end{align*}
    Since $\mbb{E}_{\mbf{X}}\left[KL(P_{Y|\mbf{X}}\|\hat{P}_{Y|\mbf{X}})\right]\geq0$, then under the assumption that $\mrm{logit}(\pi(\mbf{X}))\in\mcal{F}$, the minimum is attained when $KL(P_{Y|\mbf{X}}\|\hat{P}_{Y|\mbf{X}})=0$, $P_{\mbf{X}}$-almost surely. Consequently, the minimum is attained when $P_{Y|\mbf{X}}=\hat{P}_{Y|\mbf{X}}$ and hence if $f_0$ is a minimizer, then $\pi(\mbf{X})=\mbb{E}_{P_{Y|\mbf{X}}}[Y|\mbf{X}]=\mbb{E}_{\hat{P}_{Y|\mbf{X}}}[Y|\mbf{X}]=\sigma(f_0(\mbf{X}))$, $P_{\mbf{X}}$-almost surely.
\end{proof}

\subsection{Proof of Proposition \ref{Prop: Properties of the 0-1 logistic loss}}
\begin{proof}
	\begin{enumerate}
		\item Note that
			\begin{align*}
				\frac{\partial\ell}{\partial a} & =-y+\sigma(a)\\
				\frac{\partial^2\ell}{\partial a^2} & =\sigma'(a)=\sigma(a)[1-\sigma(a)],
			\end{align*}
			and since $\frac{\partial^2\ell}{\partial a^2}$ is non-negative for all $a\in\mbb{R}$, it follows that $\ell$ is convex with respect to $a$.
			
		\item Since $y\in\{0,1\}$ and $\sigma(a)\in[0,1]$ for all $a\in\mbb{R}$, it then follows that $\frac{\partial\ell(a,y)}{\partial a}=-y+\sigma(a)\in[-1,1]$ for all $a\in\mbb{R}$. Based on the mean value theorem, for any $a_1, a_2\in\mbb{R}$, there exists $\nu$ between $a_1$ and $a_2$ such that
		$$
		\ell(a_1, y)-\ell(a_2,y)=\frac{\partial\ell(\nu,y)}{\partial a}(a_1-a_2).
		$$
		As a result, for any $a_1, a_2\in\mbb{R}$,
		$$
		\abs{\ell(a_1, y)-\ell(a_2,y)}=\abs{\frac{\partial\ell(\nu,y)}{\partial a}}\cdot\abs{a_1-a_2}\leq\abs{a_1-a_2}.
		$$
		
		\item Let $a_1, a_2\in[-M, M]$ be arbitrary such that $\abs{a_1-a_2}\geq\eta$. Now consider the Taylor expansion of $\ell(a_1, y)$ and $\ell(a_2, y)$ around the mid-point $(a_1+a_2)/2$.
			\begin{align*}
				\ell(a_1, y) & =\ell\left(\frac{a_1+a_2}{2}, y\right)+\frac{\partial\ell\left(\frac{a_1+a_2}{2}, y\right)}{\partial a}\left(a_1-\frac{a_1+a_2}{2}\right)+\frac{1}{2}\frac{\partial^2\ell\left(\nu_1, y\right)}{\partial a^2}\left(a_1-\frac{a_1+a_2}{2}\right)^2\\
				\ell(a_2, y) & =\ell\left(\frac{a_1+a_2}{2}, y\right)+\frac{\partial\ell\left(\frac{a_1+a_2}{2}, y\right)}{\partial a}\left(a_2-\frac{a_1+a_2}{2}\right)+\frac{1}{2}\frac{\partial^2\ell\left(\nu_2, y\right)}{\partial a^2}\left(a_2-\frac{a_1+a_2}{2}\right)^2,
			\end{align*}
			where $\nu_1$ is between $a_1$ and $(a_1+a_2)/2$ and $\nu_2$ is between $a_2$ and $(a_1+a_2)/2$. Then
			\begin{align*}
				\frac{\ell(a_1,y)+\ell(a_2,y)}{2}-\ell\left(\frac{a_1+a_2}{2}, y\right) & =\frac{1}{4}\left[\frac{\partial^2\ell\left(\nu_1, y\right)}{\partial a^2}+\frac{\partial^2\ell\left(\nu_2, y\right)}{\partial a^2}\right]\frac{(a_1-a_2)^2}{4}\\
				& \geq\frac{\eta^2}{8}\min_{a\in[-M, M]}\frac{\partial^2\ell\left(a, y\right)}{\partial a^2}.
			\end{align*}
			Since $a\in[-M, M]$ and $\frac{\partial^2\ell\left(a, y\right)}{\partial a^2}=\sigma(a)[1-\sigma(a)]$ is an even function that is monotonically decreasing when $a\geq 0$, it then follows that 
			$$
			\min_{a\in[-M, M]}\frac{\partial^2\ell\left(a, y\right)}{\partial a^2}=\frac{\partial^2\ell\left(-M, y\right)}{\partial a^2}=\frac{\partial^2\ell\left(M, y\right)}{\partial a^2}=\sigma(M)[1-\sigma(M)].
			$$
			Hence
			\begin{align*}
				\delta(\eta) & =\inf_{\abs{a_1-a_2}\geq\eta}\frac{\ell(a_1, y)+\ell(a_2,y)}{2}-\ell\left(\frac{a_1+a_2}{2}, y\right)\\
					& \geq\frac{\eta^2}{8}\min_{a\in[-M, M]}\frac{\partial^2\ell\left(a, y\right)}{\partial a^2}\\
					& =\frac{\eta^2}{8}\sigma(M)[1-\sigma(M)].
			\end{align*}
	\end{enumerate}
\end{proof}

\subsection{Proof of Corollary \ref{Cor: Upper Bound L2 Metric by Risk}}
\begin{proof}
	For any $\eta>0$ and $f_1, f_2\in\mcal{F}$ satisfying $\norm{f_1-f_2}^2\geq\eta^2$, we have
	\begin{align*}
		& \frac{R(f_1)+R(f_2)}{2}-R\left(\frac{f_1+f_2}{2}\right)\\
		= & \mbb{E}\left[\frac{\ell(f_1(\mbf{X}), Y)+\ell(f_2(\mbf{X}), Y)}{2}-\ell\left(\frac{f_1(\mbf{X})+f_2(\mbf{X})}{2}, Y\right)\right]\\
		\overset{(*)}{\geq} & \mbb{E}[\delta(\abs{f_1(\mbf{X})-f_2(\mbf{X})})]\\
		\overset{(**)}{\geq} & \frac{e^M}{8(1+e^M)^2}\norm{f_1-f_2}^2\\
		\geq & \frac{e^M}{8(1+e^M)^2}\eta^2,
	\end{align*}
	where (*) and (**) follow from the definition and the lower bound of the modulus of convexity in (\ref{Eq: Modulus of Convexity}). Therefore, we obtain a lower bound for the modulus of convexity for $R(f)$:
	$$
	\tilde{\delta}(\eta):=\inf_{\norm{f_1-f_2}^2\geq\eta^2}\frac{R(f_1)+R(f_2)}{2}-R\left(\frac{f_1+f_2}{2}\right)\geq\frac{e^M}{8(1+e^M)^2}\eta^2,
	$$
	which implies that
	\begin{align*}
		\tilde{\delta}(\norm{f-f_0}) & \leq\frac{R(f)+R(f_0)}{2}-R\left(\frac{f+f_0}{2}\right)\\
			& \leq\frac{R(f)+R(f_0)}{2}-R(f_0)\\
			& =\frac{R(f)-R(f_0)}{2},
	\end{align*}
	where the 2nd inequality follows from the fact that $f_0$ minimizes $R(f)$. Hence,
	$$
	R(f)-R(f_0)\geq2\tilde{\delta}(\norm{f-f_0})\geq\frac{e^M}{4(1+e^M)^2}\norm{f-f_0}^2.
	$$
	Rearranging the above display yields the desired result.
\end{proof}

\section{Technical Results for Deep ReLU Neural Networks}\label{appendix: DNN}
\subsection{Proof of Lemma \ref{Lm: Uniform Boundedness of DNN}}
\begin{proof}
	It is easy to see that $\varphi$ is a 1-Lipschitz function so that for any $m\in\mbb{N}$ and $\mbf{u}, \mbf{v}\in\mbb{R}^m$,
	\begin{align*}
	\norm{\varphi(\mbf{u})-\varphi(\mbf{v})}=\sqrt{\sum_{i=1}^m[\varphi(u_i)-\varphi(v_i)]^2} \leq\sqrt{\sum_{i=1}^m[u_i-v_i]^2}=\norm{\mbf{u}-\mbf{v}}.
	\end{align*}
	Consequently, for any $h_{\mbf{\theta}_f}\in\mcal{F}$, 
	\begin{align*}
		\sup_{\mbf{x}\in[-1,1]^p, \norm{\mbf{x}}\leq \kappa}\abs{h_{\mbf{\theta}_f}(\mbf{x})} & =\sup_{\mbf{x}\in[-1,1]^p, \norm{\mbf{x}}\leq \kappa}\abs{\mbf{W}_L\varphi(\mbf{W}_{L-1}\varphi(\cdots\mbf{W}_2\varphi(\mbf{W}_1\mbf{x})))}\\
		& \leq\sup_{\mbf{x}\in[-1,1]^p, \norm{\mbf{x}}\leq \kappa}\norm{\mbf{W}_L}_{op}\norm{\varphi(\mbf{W}_{L-1}\varphi(\cdots\mbf{W}_2\varphi(\mbf{W}_1\mbf{x})))-\varphi(\mbf{0})}\\
		& \leq\kappa_L\sup_{\mbf{x}\in[-1,1]^p, \norm{\mbf{x}}\leq \kappa}\norm{\mbf{W}_{L-1}\varphi(\cdots\mbf{W}_2\varphi(\mbf{W}_1\mbf{x}))}\\
		& \leq\cdots\\
		& \leq\kappa_L\cdots\kappa_2\sup_{\mbf{x}\in[-1,1]^p, \norm{\mbf{x}}\leq \kappa}\norm{\varphi(\mbf{W}_1\mbf{x})-\varphi(\mbf{0})}\\
		& \leq\kappa_L\cdots\kappa_2\sup_{\mbf{x}\in[-1,1]^p, \norm{\mbf{x}}\leq \kappa}\norm{\mbf{W}_1\mbf{x}}\\
		& \leq\kappa_L\cdots\kappa_2\sup_{\mbf{x}\in[-1,1]^p, \norm{\mbf{x}}\leq \kappa}\norm{\mbf{W}_1}_{op}\norm{x}\\
		& \leq\kappa\prod_{i=1}^L\kappa_i.
	\end{align*}
\end{proof}

\subsection{Proof of Lemma \ref{Lm: Local Rademacher Complexity of DNN}}
The proof relies on bounding the covering number of $\mcal{F}$. To do this, we applied the following result (Lemma 5.23) from \citet{ma2022lecture}.

\begin{lemma}\label{Lm: 5.23 in Tengyu Ma}
    Let $\mcal{F}$ be a function class such that 
    $$
    \mcal{F}=\mcal{F}_L\circ\mcal{F}_{L-1}\circ\cdots\circ\mcal{F}_1=\{f_L\circ f_{L-1}\circ\cdots\circ f_1:f_i\in\mcal{F}_i\}.
    $$
    Assume
    \begin{enumerate}
        \item $f_i\in\mcal{F}_i$ is $\kappa_i$-Lipschitz, that is
        $$
        \norm{f_i(\mbf{x})-f_i(\mbf{y})}\leq\kappa_i\norm{\mbf{x}-\mbf{y}}.
        $$

        \item The Euclidean norm of every input to $f_i$ is bounded by $c_{i-1}$.

        \item The metric entropy of $\mcal{F}_i$ is bounded as follow:
            $$
            \log N(\varepsilon_i,\mcal{F}_i,\norm{\cdot}_n)\leq g(\varepsilon_i,c_{i-1}).
            $$
    \end{enumerate}
    Then there exists an $\varepsilon$-cover $\mcal{C}$ of $\mcal{F}_r\circ\mcal{F}_{r-1}\circ\cdots\circ\mcal{F}_1$ for $\varepsilon=\varepsilon_L+\kappa_L\varepsilon_{L-1}+\cdots+\kappa_L\kappa_{L-1}\cdots\kappa_2\varepsilon_1$ such that
    $$
    \log\abs{\mcal{C}}\leq\sum_{i=1}^Lg(\varepsilon_i, c_{i-1}).
    $$
\end{lemma}

For the class of deep ReLU neural network as described in (\ref{Eq: Class of DNN}), we can set
\begin{align*}
    \mcal{F}_1 & =\{\varphi(\mbf{W}_1\mbf{x}):\norm{\mbf{W}_1}_{op}\leq\kappa_1, \norm{\mbf{W}_1^T}_{2,1}\leq b_1\}\\
    \mcal{F}_i & =\{\varphi(\mbf{W}_i\mbf{z}_{i-1}): \mbf{z}_{i-1}\in\mcal{F}_{i-1}, \norm{\mbf{W}_i}_{op}\leq\kappa_i, \norm{\mbf{W}_i^T}_{2,1}\leq b_i\},\quad i=2,\ldots, L-1\\
    \mcal{F}_L & =\{\mbf{W}_L\mbf{z}_{L-1}:\mbf{z}_{L-1}\in\mcal{F}_{L-1},\norm{\mbf{W}_i}_{op}\leq\kappa_L, \norm{\mbf{W}_i^T}_{2,1}\leq b_L\}
\end{align*}
It is obvious that each $f_i\in\mcal{F}_i$, $i\in[L]$ is a $\kappa_i$-Lipschitz function. In terms of the Euclidean norm of the inputs, we have
\begin{equation}\label{Eq: Eulidian Norm Bound of Layer Input}
    \norm{\mbf{z}_{i-1}}=\norm{\varphi(\mbf{W}_{i-1}\varphi(\mbf{W}_{i-2}\cdots\varphi(\mbf{W}_1\mbf{x})))}\leq\kappa\prod_{l=1}^{i-1}\kappa_l,\quad i=1,\ldots,L,
\end{equation}
where $\mbf{z}_0=\mbf{x}$. In addition, let $\mcal{W}_i=\{\mbf{W}_i\in\mbb{R}^{p_i\times p_{i-1}}:\norm{\mbf{W}_i}_{op}\leq\kappa_i, \norm{\mbf{W}_i^T}_{2,1}\leq b_i\}$. Suppose that $\{\mbf{W}_i^{(1)},\ldots, \mbf{W}_i^{(N)}\}$ is a minimal $\varepsilon_i$ cover of $\mcal{W}_i$ with respect to $\norm{\cdot}_F$. Then for any $f_i\in\mcal{F}_i$, we have $f_i(\mbf{z})=\sigma(\mbf{W}_i\mbf{z})$ for some $\mbf{W}_i\in\mcal{W}_i$ and there exists $j\in[N]$ such that
$$
\norm{\mbf{W}_i-\mbf{W}_i^{(j)}}_F<\varepsilon_i.
$$
Then set $f_i^{(j)}(\mbf{z})=\sigma(\mbf{W}_i^{(j)}\mbf{z})$, we have
\begin{align*}
    \norm{f_i-f_i^{(j)}}_n & =\sqrt{\frac{1}{n}\sum_{k=1}^n\norm{f_i(\mbf{z}_{i-1}^{(k)})-f_i^{(j)}(\mbf{z}_{i-1}^{(k)})}^2}\\
        & =\sqrt{\frac{1}{n}\sum_{k=1}^n\norm{\sigma(\mbf{W}_i\mbf{z}_{i-1}^{(k)})-\sigma(\mbf{W}_i^{(j)}\mbf{z}_{i-1}^{(k)})}^2}\\
        & \leq\sqrt{\frac{1}{n}\sum_{k=1}^n\norm{(\mbf{W}_i-\mbf{W}_i^{(j)})\mbf{z}_{i-1}^{(k)}}^2}\\
        & \leq\norm{\mbf{W}_i-\mbf{W}_i^{(j)}}_F\sqrt{\frac{1}{n}\sum_{k=1}^n\norm{\mbf{z}_{i-1}^{(k)}}^2}\\
        & \leq\kappa\left(\prod_{l=1}^{i-1}\kappa_l\right)\varepsilon_i,
\end{align*}
where $\mbf{z}_{i-1}^{(k)}$ is the input of the $i$the layer in the network for the $k$th sample. Then
\begin{align*}
    N(\varepsilon_i,\mcal{F}_i,\norm{\cdot}_n) & \leq N\left(\frac{\varepsilon_i}{\kappa\prod_{l=1}^{i-1}\kappa_l},\mcal{W}_i,\norm{\cdot}_F\right)
\end{align*}
Now note that for any $\mbf{W}_i\in\mcal{W}_i$,
\begin{align*}
    \norm{\mbf{W}_i}_F^2 & =\norm{\mbf{W}_i^T}_F^2=\sum_{j}\norm{[\mbf{W}_i^T]_{j,:}}^2\\
        & \leq\left(\sum_j\norm{[\mbf{W}_i^T]_{j,:}}\right)^2\\
        & =\norm{\mbf{W}_i^T}_{2,1}^2\leq b_i^2.
\end{align*}
Then by Proposition 4.2.12 in \citet{Vershynin_2018}, we have
$$
N\left(\frac{\varepsilon_i}{\kappa\prod_{l=1}^{i-1}\kappa_l},\mcal{W}_i,\norm{\cdot}_F\right)\leq\left(1+\frac{2b_i\kappa\prod_{l=1}^{i-1}\kappa_l}{\varepsilon_i}\right)^{p_ip_{i-1}},
$$
and hence
\begin{equation}\label{Eq: Entropy Number of ith layer}
\log N(\varepsilon_i,\mcal{F}_i,\norm{\cdot}_n)\leq p_ip_{i-1}\log\left(1+\frac{2b_i\kappa\prod_{l=1}^{i-1}\kappa_l}{\varepsilon_i}\right)
\end{equation}
Based on these observations, we are now ready to prove Lemma \ref{Lm: Local Rademacher Complexity of DNN}.

\begin{proof}
    Denote $c_{i-1}=\kappa\prod_{l=1}^{i-1}\kappa_l$ and take
    \begin{equation}\label{Eq: choice of ei}
    \varepsilon_i=\left(\frac{c_{i-1}b_i}{\kappa_{i+1}\cdots\kappa_L}\right)^{1/2}\frac{\varepsilon}{\kappa^{1/2}\sum_{l=1}^L\left(\frac{b_l}{\kappa_l}\right)^{1/2}\prod_{l=1}^L\kappa_l^{1/2}}.
    \end{equation}
    Then
    \begin{align*}
        \sum_{i=1}^L\varepsilon_i\kappa_{i+1}\cdots\kappa_L & =\left(\sum_{i=1}^Lc_{i-1}^{1/2}b_i^{1/2}\kappa_{i+1}^{1/2}\cdots\kappa_L^{1/2}\right)\frac{\varepsilon}{\kappa^{1/2}\sum_{l=1}^L\left(\frac{b_l}{\kappa_l}\right)^{1/2}\prod_{l=1}^L\kappa_l^{1/2}}\\
            & =\left(\sum_{i=1}^L\kappa^{1/2}b_i^{1/2}\kappa_1^{1/2}\cdots\kappa_{i-1}^{1/2}\kappa_{i+1}^{1/2}\cdots\kappa_L^{1/2}\right)\frac{\varepsilon}{\kappa^{1/2}\sum_{l=1}^L\left(\frac{b_l}{\kappa_l}\right)^{1/2}\prod_{l=1}^L\kappa_l^{1/2}}\\
            & =\left(\kappa^{1/2}\sum_{i=1}^L\left(\frac{b_i}{\kappa_i}\right)^{1/2}\prod_{i=1}^L\kappa_i^{1/2}\right)\frac{\varepsilon}{\kappa^{1/2}\sum_{l=1}^L\left(\frac{b_l}{\kappa_l}\right)^{1/2}\prod_{l=1}^L\kappa_l^{1/2}}\\
            & =\varepsilon.
    \end{align*}
    In addition, 
    \begin{align*}
        \frac{c_{i-1}b_i}{\varepsilon_i} & =\frac{\sum_{l=1}^L\left(\frac{b_l}{\kappa_l}\right)^{1/2}\prod_{l=1}^L\kappa_l^{1/2}}{\varepsilon}\frac{\kappa\kappa_1\cdots\kappa_{i-1}b_i}{(\kappa_1\cdots\kappa_{i-1}b_i)^{1/2}}(\kappa_{i+1}\cdots\kappa_L)^{1/2}\\
        & =\frac{\kappa}{\varepsilon}\left[\sum_{l=1}^L\left(\frac{b_l}{\kappa_l}\right)^{1/2}\prod_{l=1}^L\kappa_l^{1/2}\right]^2.
    \end{align*}
    It then follows from Lemma \ref{Lm: 5.23 in Tengyu Ma} and (\ref{Eq: Entropy Number of ith layer}) that
    \begin{align*}
        \log N(\varepsilon,\mcal{F},\norm{\cdot}_n) & \leq\sum_{i=1}^Lp_ip_{i-1}\log\left(1+\frac{2\kappa}{\varepsilon}\left[\sum_{l=1}^L\left(\frac{b_l}{\kappa_l}\right)^{1/2}\prod_{l=1}^L\kappa_l^{1/2}\right]^2\right)\\
        & \leq W\log\left(1+\frac{2\kappa}{\varepsilon}\left[\sum_{l=1}^L\left(\frac{b_l}{\kappa_l}\right)^{1/2}\prod_{l=1}^L\kappa_l^{1/2}\right]^2\right).
    \end{align*}
    Recall that $M=\kappa\prod_{i=1}^L\kappa_i$. Now based on Lemma \ref{Lm: Covering Number of Star-hull}, it follows that
    \begin{align*}
        \log N(\varepsilon,\mrm{star}(\mcal{F}, f_0), \norm{\cdot}_n) & \leq\log\frac{4M}{\varepsilon}+\log N\left(\frac{\varepsilon}{2},\mcal{F},\norm{\cdot}_n\right)\\
        & \leq\log\frac{4M}{\varepsilon}+W\log\left(1+\frac{4M}{\varepsilon}\left(\sum_{i=1}^L\left(\frac{b_i}{\kappa_i}\right)^{1/2}\right)^2\right).
    \end{align*}
    For any $f\in\mrm{star}(\mcal{F},f_0)$, we have $f=(1-\gamma)f_0+\gamma h_{\mbf{\theta}_f}$ for some $\gamma\in[0,1]$ and $h_{\mbf{\theta}_f}\in\mcal{F}$, it then follows that
    \begin{align*}
        \sup_{\norm{\mbf{x}}\leq\kappa}\abs{f(\mbf{x})} & \leq(1-\gamma)\sup_{\norm{\mbf{x}}\leq\kappa}\abs{f_0(\mbf{x})}+\gamma\sup_{\norm{\mbf{x}}\leq\kappa}\abs{h_{\mbf{\theta}_f}(\mbf{x})}\leq M.
    \end{align*}
    Therefore, it suffices to consider $\varepsilon\leq M$. In this case, since $b_i\geq\kappa_i$ for all $i\in[L]$, we have
    $$
    \frac{4M}{\varepsilon}\left(\sum_{i=1}^L\left(\frac{b_i}{\kappa_i}\right)^{1/2}\right)^2\geq4\left(\sum_{i=1}^L\left(\frac{b_i}{\kappa_i}\right)^{1/2}\right)^2\geq 4,
    $$
    which implies that
    \begin{align*}
    \log N(\varepsilon,\mrm{star}(\mcal{F}, f_0), \norm{\cdot}_n) & \leq \log\frac{4M}{\varepsilon}+M\log\frac{5M\left(\sum_{i=1}^L\left(\frac{b_i}{\kappa_i}\right)^{1/2}\right)^2}{\varepsilon}\\
    & \leq2M\log\frac{5M\left(\sum_{i=1}^L\left(\frac{b_i}{\kappa_i}\right)^{1/2}\right)^2}{\varepsilon}.\numberthis\label{Eq: Entropy Number of DNN-1}
    \end{align*}
    Next, by the Dudley's chaining technique (see for example, Lemma 3 in \citet{farrell2021deep}) that
    \begin{align*}
		& \mbb{E}_{\xi}\left[\sup_{h_{\mbf{\theta}_f}\in\mrm{star}(\mcal{F},f_0), \norm{h_{\mbf{\theta}_f}-f_0}_n^2\leq r}\frac{1}{n}\sum_{i=1}^n\xi_ih_{\mbf{\theta}_f}(\mbf{X}_i)\right]\\
        = & \mbb{E}_{\xi}\left[\sup_{h_{\mbf{\theta}_f}\in\mrm{star}(\mcal{F},f_0), \norm{h_{\mbf{\theta}_f}-f_0}_n^2\leq r}\frac{1}{n}\sum_{i=1}^n\xi_i(h_{\mbf{\theta}_f}(\mbf{X}_i)-f_0(\mbf{X}_i))\right] \\
        \leq & \frac{12}{\sqrt{n}}\int_0^{\sqrt{r}}\sqrt{\log N(\varepsilon, \mrm{star}(\mcal{F},f_0)-\{f_0\},\norm{\cdot}_n)}d\varepsilon\\
        \leq & \frac{12}{\sqrt{n}}\int_0^{\sqrt{r}}\sqrt{\log N(\varepsilon, \mrm{star}(\mcal{F},f_0),\norm{\cdot}_n)}d\varepsilon
    \end{align*}
    For notation simplicity, denote $\breve{M}=5M\left(\sum_{i=1}^L\left(\frac{b_i}{\kappa_i}\right)^{1/2}\right)^2$. According to (\ref{Eq: Entropy Number of DNN-1}), we have
    \begin{align*}
        \int_0^{\sqrt{r}}\sqrt{\log N(\varepsilon, \mrm{star}(\mcal{F},f_0),\norm{\cdot}_n)}d\varepsilon & \leq \sqrt{2W}\int_0^{\sqrt{r}}\sqrt{\log\frac{\breve{M}}{\varepsilon}}d\varepsilon\\
        & \leq \sqrt{2W}\sqrt{r}\left(1+\log^{1/2}\frac{\breve{M}}{\sqrt{r}}\right),
    \end{align*}
    where the last inequality follows since $\sqrt{r}\leq 2M$ so that $\log\frac{\breve{M}}{\sqrt{r}}\geq\log\frac{5M\left(\sum_{i=1}^L\left(\frac{b_i}{\kappa_i}\right)^{1/2}\right)^2}{2M}\geq\log(5/2)$.
    \begin{align*}
        \int_0^{\sqrt{r}}\sqrt{\log\frac{\breve{M}}{\varepsilon}}d\varepsilon & =\left.\varepsilon\log^{1/2}\frac{\breve{M}}{\varepsilon}\right|_0^{\sqrt{r}}+\int_0^{\sqrt{r}}\frac{1}{2}\log^{-1/2}\frac{\breve{M}}{\varepsilon}d\varepsilon\\
        & \leq\sqrt{r}\log^{1/2}\frac{\breve{M}}{\sqrt{r}}+\frac{\sqrt{r}}{2}\log^{-1/2}\frac{\breve{M}}{\sqrt{r}}\\
        & \leq\sqrt{r}\log^{1/2}\frac{\breve{M}}{\sqrt{r}}+\frac{\log^{-1/2}(5/2)}{2}\sqrt{r}\\
        & \leq\sqrt{r}\left(1+\log^{1/2}\frac{\breve{M}}{\sqrt{r}}\right).
    \end{align*}
    As a result,
    $$
    \mbb{E}_{\xi}\left[\sup_{h_{\mbf{\theta}_f}\in\mrm{star}(\mcal{F},f_0), \norm{h_{\mbf{\theta}_f}-f_0}_n^2\leq r}\frac{1}{n}\sum_{i=1}^n\xi_ih_{\mbf{\theta}_f}(\mbf{X}_i)\right]\leq\frac{12\sqrt{2}}{\sqrt{n}}\sqrt{W}\sqrt{r}\left(1+\log^{1/2}\frac{\breve{M}}{\sqrt{r}}\right).
    $$
    We now show that $\hat{\psi}_n(r)$ is a sub-root function. To do so, we check the three conditions required for a function to be sub-root.
    \begin{enumerate}
        \item First note that $\sqrt{r}\leq 2M$, then $\log\frac{\breve{M}}{\sqrt{r}}\geq\log(5/2)$, which immediately shows that $\hat{\psi}_n(r)$ is nonnegative.

        \item Note that
        \begin{align*}
            \hat{\psi}'_n(r) & =\frac{12\sqrt{2}}{\sqrt{n}}\sqrt{W}\left[\frac{1}{2\sqrt{r}}\left(1+\log^{1/2}\frac{\breve{M}}{\sqrt{r}}\right)-\frac{1}{4\sqrt{r}\log^{1/2}\frac{\breve{M}}{\sqrt{r}}}\right]\\
            & =\frac{6\sqrt{2}}{\sqrt{n}\sqrt{r}}\sqrt{W}\left[1+\log^{1/2}\frac{\breve{M}}{\sqrt{r}}-\frac{1}{2\log^{1/2}\frac{\tilde{M}}{\sqrt{r}}}\right]\\
            & \geq0,
        \end{align*}
        where the last inequality follows since $\log\frac{\breve{M}}{\sqrt{r}}\geq\log(5/2)$ and hence $1-\frac{1}{2\log^{1/2}\frac{\breve{M}}{\sqrt{r}}}\geq1-\frac{1}{2\log^{1/2}(5/2)}\geq0$. Hence, $\hat{\psi}_n(r)$ is a nondecreasing function.

        \item Since $\hat{\psi}_n(r)/\sqrt{r}=\frac{12\sqrt{2}}{\sqrt{n}}\sqrt{W}\left(1+\log^{1/2}\frac{\breve{M}}{\sqrt{r}}\right)$, it is obvious that $\hat{\psi}_n(r)/\sqrt{r}$ is nonincreasing.
    \end{enumerate}
    We now provide an upper bound for the fixed point of $\hat{\psi}_n(r)$. Note that
    \begin{align*}
        \hat{\psi}_n(\hat{r}^*)=\hat{r}^* & \Leftrightarrow \frac{12\sqrt{2}}{\sqrt{n}}\sqrt{W}\sqrt{\hat{r}^*}\left(1+\log^{1/2}\frac{\breve{M}}{\sqrt{\hat{r}^*}}\right)=\hat{r}^*\\
            & \Leftrightarrow\frac{12\sqrt{2}}{\sqrt{n}}\sqrt{W}\left(1+\log^{1/2}\frac{\breve{M}}{\sqrt{\hat{r}^*}}\right)=\sqrt{\hat{r}^*},
    \end{align*}
    which implies that $\sqrt{\hat{r}^*}\geq\frac{1}{\sqrt{n}}$ and hence,
    $$
    \sqrt{\hat{r}^*}\leq\frac{12\sqrt{2}}{\sqrt{n}}\sqrt{W}\left(1+\log^{1/2}(\breve{M}\sqrt{n})\right).
    $$
    Therefore,
    $$
    \hat{r}^*\lesssim\frac{W}{n}\left(1+\log^{1/2}(\breve{M}\sqrt{n})\right)^2\lesssim\frac{W}{n}\log(\breve{M}\sqrt{n}).
    $$
\end{proof}

\subsection{Proof of Theorem \ref{Thm: RoC of DNN}}
\begin{proof}
    To start with, we check the conditions required in Theorem \ref{Thm: Extension of Bartlett Thm 5.4}. Condition 1 holds immediately based on our assumption $f_0\in\mrm{argmin}_{f\in\mcal{F}}R(f)$. Conditions 2 and 3 follow from Proposition \ref{Prop: Properties of the 0-1 logistic loss} with $L=1$ and $B^*=\frac{4(1+e^M)^2}{e^M}$. In addition, note that
   \begin{align*}
   \sup_{f\in\mcal{F}}\norm{\ell_f-\ell_{f_0}}_\infty & \leq\sup_{f\in\mcal{F}}\sup_{\norm{\mbf{x}}\leq\kappa}\abs{\ell(f(\mbf{x}), y)-\ell(f_0(\mbf{x}),y)}\\
    & \leq\sup_{f\in\mcal{F}}\sup_{\norm{\mbf{x}}\leq\kappa}\abs{f(\mbf{x})-f_0(\mbf{x})}\\
    & \leq 2M.
   \end{align*}
   Therefore, it follows from Theorem \ref{Thm: Extension of Bartlett Thm 5.4} by taking $C=2$ that with probability at least $1-3\delta$,
   \begin{align*}
   R(h_{\mbf{\theta}_f})-R(f_0) & \lesssim \frac{W}{B^*n}\log\left[M\left(\sum_{i=1}^L\left(\frac{b_i}{\kappa_i}\right)^{1/2}\right)^2\sqrt{n}\right]+\frac{M+B^*}{n}\log\frac{1}{\delta}\\
    & \lesssim\frac{W}{n}\log\left[M\left(\sum_{i=1}^L\left(\frac{b_i}{\kappa_i}\right)^{1/2}\right)^2\sqrt{n}\right]+\frac{M+B^*}{n}\log\frac{1}{\delta},
   \end{align*}
   where the 2nd inequality follows since $1/B^*=\sigma(M)[1-\sigma(M)]\leq\frac{1}{4}$ and 
   $$
   \norm{h_{\mbf{\theta}_f}-f_0}\lesssim\frac{\sqrt{W}}{\sqrt{n}}\log^{1/2}\left[M\left(\sum_{i=1}^L\left(\frac{b_i}{\kappa_i}\right)^{1/2}\right)^2\sqrt{n}\right]+\sqrt{\frac{M+B^*}{n}\log\frac{1}{\delta}}
   $$
\end{proof}

\subsection{Proof of Corollary \ref{Cor: Bound for e}}
\begin{proof}
    Using the notations in Theorem \ref{Thm: Empirical Norm and L2 Norm}, it is easy to see that $\tilde{M}=2M$. It then follows from Lemma \ref{Lm: Local Rademacher Complexity of DNN} that
    \begin{align*}
         & \mbb{E}_\xi\left[\left.\sup_{f\in\mrm{star}(\mcal{F},f_0), \norm{f-f_0}_n^2\leq 2\tilde{M}^2r}\frac{1}{n}\sum_{i=1}^n\xi_i\frac{f(\mbf{X}_i)}{\tilde{M}}\right|\mbf{X}_1,\ldots,\mbf{X}_n\right]\\
        \leq & \frac{24}{\tilde{M}}\frac{\sqrt{W}}{\sqrt{n}}\sqrt{\tilde{M}^2r}\left(1+\log^{1/2}\frac{5M\left(\sum_{i=1}^L\left(\frac{b_i}{\kappa_i}\right)^{1/2}\right)^2}{\sqrt{2\tilde{M}^2r}}\right)\\
        = & \frac{24\sqrt{W}}{\sqrt{n}}\sqrt{r}\left(1+\log^{1/2}\frac{5\left(\sum_{i=1}^L\left(\frac{b_i}{\kappa_i}\right)^{1/2}\right)^2}{2\sqrt{2r}}\right).
    \end{align*}
    Let $\tilde{\psi}_n(r)=\frac{24\sqrt{W}}{\sqrt{n}}\sqrt{r}\left(1+\frac{5\left(\sum_{i=1}^L\left(\frac{b_i}{\kappa_i}\right)^{1/2}\right)^2}{2\sqrt{2r}}\right)$ and $\hat{\psi}_n(r)=\tilde{\psi}_n(r)+\frac{62}{3n}\log\frac{1}{\delta}$. Note that for any $f\in\mcal{F}$, $\norm{f-f_0}_n^2\leq 2\norm{f}_n^2+2\norm{f_0}_n^2\leq4M^2=\tilde{M}^2$, so without loss of generality, we may assume $r\leq 1/2$ and it follows from the same reasoning as in the proof of Lemma \ref{Lm: Local Rademacher Complexity of DNN} that $\tilde{\psi}_n(r)$ is a sub-root function with fixed point
    $$
    \tilde{r}^*\lesssim\frac{W}{n}\log\left[\left(\sum_{i=1}^L\left(\frac{b_i}{\kappa_i}\right)^{1/2}\right)^2\sqrt{n}\right].
    $$
    In view of Lemma \ref{Lm: adding a constant to sub-root function}, we can know that the fixed point of $\hat{\psi}_n(r)$ is bounded by
    $$
    \hat{r}^*\lesssim \tilde{r}^*+\frac{1}{n}\log\frac{1}{\delta}\lesssim \frac{W}{n}\log\left[\left(\sum_{i=1}^L\left(\frac{b_i}{\kappa_i}\right)^{1/2}\right)^2\sqrt{n}\right]+\frac{1}{n}\log\frac{1}{\delta}.
    $$
    Therefore, by Theorem \ref{Thm: Empirical Norm and L2 Norm}, with probability at least $1-6\delta$,
    \begin{align*}
        \norm{h_{\mbf{\theta}_f}-f_0}_n & \lesssim\norm{h_{\mbf{\theta}_f}-f_0}+M\left(\sqrt{\hat{r}^*}+\sqrt{\frac{1}{n}\log\frac{1}{\delta}}\right)\\
            & \lesssim \frac{\sqrt{W}}{\sqrt{n}}\log^{1/2}\left[M\left(\sum_{i=1}^L\left(\frac{b_i}{\kappa_i}\right)^{1/2}\right)^2\sqrt{n}\right]+\sqrt{\frac{M+B^*}{n}\log\frac{1}{\delta}}+\\
            & \qquad\qquad M\left(\frac{\sqrt{W}}{\sqrt{n}}\log^{1/2}\left[\left(\sum_{i=1}^L\left(\frac{b_i}{\kappa_i}\right)^{1/2}\right)^2\sqrt{n}\right]+\sqrt{\frac{1}{n}\log\frac{1}{\delta}}\right)\\
            & \lesssim \frac{(1\vee M)\sqrt{W}}{\sqrt{n}}\log^{1/2}\left[(1\vee M)\left(\sum_{i=1}^L\left(\frac{b_i}{\kappa_i}\right)^{1/2}\right)^2\sqrt{n}\right]+(M+\sqrt{M+B^*})\sqrt{\frac{1}{n}\log\frac{1}{\delta}}.
    \end{align*}
    Consequently,
    \begin{align*}
        \norm{\mbf{e}} & =\norm{\mbf{h}_{\mbf{\theta}_f}-\mbf{f}_0}=\sqrt{n}\norm{h_{\mbf{\theta}_f}-f_0}_n\\
            & \lesssim \sqrt{n}\left\{\frac{(1\vee M)\sqrt{W}}{\sqrt{n}}\log^{1/2}\left[(1\vee M)\left(\sum_{i=1}^L\left(\frac{b_i}{\kappa_i}\right)^{1/2}\right)^2\sqrt{n}\right]+(M+\sqrt{M+B^*})\sqrt{\frac{1}{n}\log\frac{1}{\delta}}\right\}\\
            & \lesssim (1\vee M)\sqrt{W}\log^{1/2}\left[(1\vee M)\left(\sum_{i=1}^L\left(\frac{b_i}{\kappa_i}\right)^{1/2}\right)^2\sqrt{n}\right]+(M+\sqrt{M+B^*})\sqrt{\log\frac{1}{\delta}}.
    \end{align*}
\end{proof}

\section{Distance Between Linearized Deep ReLU Network $\tilde{h}_{\mbf{\theta}_f+\Delta\mbf{\theta}_S}$ and the Deep ReLU Network $h_{\mbf{\theta}_f+\Delta\mbf{\theta}_S}$}\label{appendix: distance between networks}
\subsection{Proof of Lemma \ref{Lm: Bound Derivative of Deep ReLU Neural Networks}}
\begin{proof}
    Note that
    \begin{align*}
        \norm{\nabla_{\mbf{W}_L}h_{\mbf{\theta}_f}(\mbf{x})}_F^2 & =\norm{\left[\left(\prod_{j=1}^{L-1}\mbf{\Sigma}_j\mbf{W}_j\right)\mbf{x}\right]^T}_F^2=\mrm{tr}\left(\left[\left(\prod_{j=1}^{L-1}\mbf{\Sigma}_j\mbf{W}_j\right)\mbf{x}\right]^T\left[\left(\prod_{j=1}^{L-1}\mbf{\Sigma}_j\mbf{W}_j\right)\mbf{x}\right]\right)\\
        & =\mbf{x}^T\left[\left(\prod_{j=1}^{L-1}\mbf{\Sigma}_j\mbf{W}_j\right)\right]^T\left[\left(\prod_{j=1}^{L-1}\mbf{\Sigma}_j\mbf{W}_j\right)\right]\mbf{x}\\
        & \leq\norm{\left[\left(\prod_{j=1}^{L-1}\mbf{\Sigma}_j\mbf{W}_j\right)\right]^T\left[\left(\prod_{j=1}^{L-1}\mbf{\Sigma}_j\mbf{W}_j\right)\right]}_{op}\norm{\mbf{x}}^2\\
        & \leq\kappa^2\prod_{j=1}^{L-1}\norm{\mbf{\Sigma}_j}_{op}^2\norm{\mbf{W}_j}_{op}^2\\
        & \leq\kappa^2\prod_{j=1}^{L-1}\kappa_j^2,
    \end{align*}
    where the last inequality follows since $\norm{\mbf{\Sigma}_j}_{op}\leq1$ based on the definition of $\mbf{\Sigma}_j$. Similarly, for $l=1,\ldots, L-1$,
    \begin{align*}
      \norm{\mbf{\Sigma}_l\left(\prod_{j=l+1}^{L-1}\mbf{\Sigma}_j\mbf{W}_j\right)^T\mbf{W}_L^T\mbf{z}_{l-1}^T}_F^2 
      = & \norm{\mbf{z}_{l-1}}^2\mrm{tr}\left[\mbf{W}_L\left(\prod_{j=l+1}^{L-1}\mbf{\Sigma}_j\mbf{W}_j\right)\mbf{\Sigma}_l^T\mbf{\Sigma}_l\left(\prod_{j=l+1}^{L-1}\mbf{\Sigma}_j\mbf{W}_j\right)^T\mbf{W}_L^T\right]\\
      \leq & \norm{\mbf{z}_{l-1}}^2\norm{\left(\prod_{j=l+1}^{L-1}\mbf{\Sigma}_j\mbf{W}_j\right)\mbf{\Sigma}_l^T\mbf{\Sigma}_l\left(\prod_{j=l+1}^{L-1}\mbf{\Sigma}_j\mbf{W}_j\right)^T}_{op}\mrm{tr}\left(\mbf{W}_L^T\mbf{W}_L\right)\\
      \leq & \norm{\mbf{z}_{l-1}}^2\norm{\mbf{W}_L}_F^2\prod_{j=l+1}^{L-1}\norm{\mbf{\Sigma}_j}^2\norm{\mbf{W}_j}_{op}^2\norm{\mbf{\Sigma}_l^T}_{op}^2\\
     \overset{(*)}{\leq} & \left(\kappa\prod_{j=1}^{l-1}\kappa_j\right)^2b_L^2\left(\prod_{j=l+1}^{L-1}\kappa_j^2\right)\\
      = & \frac{\kappa^2b_L^2}{\kappa_l^2}\prod_{j=1}^{L-1}\kappa_j^2,
    \end{align*}
    where the inequality (*) follows from (\ref{Eq: Eulidian Norm Bound of Layer Input}) and $\norm{\mbf{W}_L}_F=\sqrt{\sum_j\norm{[\mbf{W}_L]_{j,:}}^2}\leq\sum_j\norm{[\mbf{W}_L]_{j,:}}=\norm{\mbf{W}_L}_{2,1}$. Consequently,
    \begin{align*}
        \norm{\nabla_{\mbf{\theta}}h_{\mbf{\theta}_f}(\mbf{x})}^2 & = \sup_{\norm{\mbf{x}}\leq\kappa}\sum_{l=1}^L\norm{\nabla_{\mbf{W}_l}h_{\mbf{\theta}_f}}_F^2\\
        & =\norm{\left[\left(\prod_{j=1}^{L-1}\mbf{\Sigma}_j\mbf{W}_j\right)\mbf{x}\right]^T}_F^2+\sum_{l=1}^{L-1}\norm{\mbf{\Sigma}_l\left(\prod_{j=l+1}^{L-1}\mbf{\Sigma}_j\mbf{W}_j\right)^T\mbf{W}_L^T\mbf{z}_l}_F^2\\
        & \leq\kappa^2\prod_{j=1}^{L-1}\kappa_j^2+\sum_{l=1}^{L-1}\frac{\kappa^2b_L^2}{\kappa_l^2}\prod_{j=1}^{L-1}\kappa_j^2\\
        & =\kappa^2\left(\prod_{j=1}^{L-1}\kappa_j^2\right)\left(1+b_L^2\sum_{l=1}^{L-1}\frac{1}{\kappa_l^2}\right).
    \end{align*}
    The desired result then follows by applying the inequality $\sqrt{x+y}\leq\sqrt{x}+\sqrt{y}$ for positive $x,y$. 

    For the bound on $\mrm{tr}(\mbf{K}_{-S})$, note that
    \begin{align*}
        \mrm{tr}(\mbf{K}_{-S})=\sum_{i=1}^n\norm{\nabla_{\mbf{\theta}}h_{\mbf{\theta}_f}(\mbf{X}_{i,-S})}^2\leq n\kappa^2\left(\prod_{j=1}^{L-1}\kappa_j^2\right)\left(1+b_L^2\sum_{l=1}^{L-1}\frac{1}{\kappa_l^2}\right).
    \end{align*}
\end{proof}

\subsection{Proof of Proposition \ref{Prop: Bound on norm of Delta theta S}}
\begin{proof}
    We first prove the result for $\norm{\Delta\mbf{\theta}_S}$. Since $L_\lambda(\mbf{w})$ is a convex function, the Newton-Raphson algorithm is guaranteed to converge to its minimum. Therefore, $\Delta\mbf{\theta}_S$ needs to satisfy
    \begin{equation}\label{Eq: Necessary Condition of Minimization}
        \mbf{0}=\frac{\partial}{\partial\mbf{w}}L_\lambda(\Delta\mbf{\theta}_S)=\mbf{\Phi}_{-S}^T\left[\mbf{Y}-\sigma(\mbf{h}_{\mbf{\theta}_f, -S}+\mbf{\Phi}_{-S}\Delta\mbf{\theta}_S)\right]-n\lambda\Delta\mbf{\theta}_S,
    \end{equation}
    which implies that $\Delta\mbf{\theta}_S$ needs to satisfy
    \begin{equation}\label{Eq: Necessary Condition of Minimization Delta theta S}
    \Delta\mbf{\theta}_S=\frac{1}{n\lambda}\mbf{\Phi}_{-S}^T\left[\mbf{Y}-\sigma(\mbf{h}_{\mbf{\theta}_f, -S}+\mbf{\Phi}_{-S}\Delta\mbf{\theta}_S)\right].
    \end{equation}
    Therefore, by the triangle inequality,
    \begin{align*}
        \norm{\Delta\mbf{\theta}_S} & =\frac{1}{n\lambda}\norm{\mbf{\Phi}_{-S}^T\left[\mbf{Y}-\sigma(\mbf{h}_{\mbf{\theta}_f, -S}+\mbf{\Phi}_{-S}\Delta\mbf{\theta}_S)\right]}\\
        & \leq\frac{1}{n\lambda}\norm{\mbf{\Phi}_{-S}^T\mbf{\epsilon}}+\frac{1}{n\lambda}\norm{\mbf{\Phi}_{-S}^T\left[\sigma(\mbf{f}_{0,-S})-\sigma(\mbf{h}_{\mbf{\theta}_f,-S}+\mbf{\Phi}_{-S}^T\Delta\mbf{\theta}_S)\right]}.
    \end{align*}
    Since elements in $\mbf{\epsilon}$ are independent, mean-zero sub-Gaussian random variables, it follows from Theorem 1 in \citet{hsu2012tail} that with probability at least $1-\delta$,
    $$
    \norm{\mbf{\Phi}_{-S}^T\mbf{\epsilon}}^2\leq\mrm{tr}(\mbf{K}_{-S})+2\sqrt{\mrm{tr}(\mbf{K}_{-S}^2)\log\frac{1}{\delta}}+2\norm{\mbf{K}_{-S}}_{op}\log\frac{1}{\delta}.
    $$
    Note that $\mrm{tr}(\mbf{K}_{-S}^2)\leq\norm{\mbf{K}_{-S}}_{op}\mrm{tr}(\mbf{K}_{-S})$, then with probability at least $1-\delta$,
    \begin{align*}
        \norm{\mbf{\Phi}_{-S}^T\mbf{\epsilon}}^2 & \leq\mrm{tr}(\mbf{K}_{-S})+2\sqrt{\norm{\mbf{K}_{-S}}_{op}\mrm{tr}(\mbf{K}_{-S})\log\frac{1}{\delta}}+2\norm{\mbf{K}_{-S}}_{op}\log\frac{1}{\delta}\\
        & \left[\sqrt{\mrm{tr}(\mbf{K}_{-S})}+\sqrt{2\norm{\mbf{K}_{-S}}_{op}\log\frac{1}{\delta}}\right]^2.
    \end{align*}
    Hence, with probability at least $1-\delta$,
    $$
    \norm{\mbf{\Phi}_{-S}^T\mbf{\epsilon}}\leq\sqrt{\mrm{tr}(\mbf{K}_{-S})}+\sqrt{2\norm{\mbf{K}_{-S}}_{op}\log\frac{1}{\delta}}.
    $$
    For the second term, note that by the mean value theorem,
    \begin{align*}
    \sigma(\mbf{f}_{0,-S})-\sigma(\mbf{h}_{\mbf{\theta}_f,-S}+\mbf{\Phi}_{-S}\Delta\mbf{\theta}_S) & =\mbf{T}(\mbf{f}_{0, -S}-\mbf{h}_{\mbf{\theta}_f,-S}-\mbf{\Phi}_{-S}\Delta\mbf{\theta}_S)\\
    & =\mbf{T}(-\mbf{e}_{-S}-\mbf{\Phi}_{-S}\Delta\mbf{\theta}_S),\numberthis\label{Eq: MVT for sigma}
    \end{align*}
    where $\mbf{T}=\mrm{Diag}\left\{\sigma'(\gamma\mbf{f}_{0,-S}+(1-\gamma)\mbf{h}_{\mbf{\theta}_f, -S})\right\}$ with some $\gamma\in[0,1]$. Since $\norm{\mbf{T}}_{op}\leq \frac{1}{4}$, we have with probability at least $1-\delta$,
    \begin{align*}
        \norm{\Delta\mbf{\theta}_S} & \leq\frac{1}{n\lambda}\left[\sqrt{\mrm{tr}(\mbf{K}_{-S})}+\sqrt{2\norm{\mbf{K}_{-S}}_{op}\log\frac{1}{\delta}}\right]+\frac{1}{n\lambda}\norm{\mbf{\Phi}_{-S}^T\mbf{T}\left[\mbf{e}_{-S}+\mbf{\Phi}_{-S}\Delta\mbf{\theta}_S\right]}\\
        & \leq\frac{1}{n\lambda}\left[\sqrt{\mrm{tr}(\mbf{K}_{-S})}+\sqrt{2\norm{\mbf{K}_{-S}}_{op}\log\frac{1}{\delta}}\right]+\frac{1}{4n\lambda}\norm{\mbf{\Phi}_{-S}}_{op}\left(\norm{\mbf{e}_{-S}}+\norm{\mbf{\Phi}_{-S}}_{op}\norm{\Delta\mbf{\theta}_S}\right)\\
        & =\frac{\sqrt{\mrm{tr}(\mbf{K}_{-S})}}{n\lambda}+\frac{\sqrt{\norm{\mbf{K}_{-S}}_{op}}}{n\lambda}\left(\frac{1}{4}\norm{\mbf{e}_{-S}}+\sqrt{2\log\frac{1}{\delta}}\right)+\frac{\norm{\mbf{K}_{-S}}_{op}}{4n\lambda}\norm{\Delta\mbf{\theta}_S}, \numberthis\label{Eq: relation of norm of Phi w*}
    \end{align*}
    where the last equality follows since $\norm{\mbf{\Phi}_{-S}}_{op}=\sqrt{\lambda_{\max}(\mbf{\Phi}_{-S}^T\mbf{\Phi}_{-S})}=\norm{\mbf{K}_{-S}}_{op}^{1/2}$. Therefore, by rearranging (\ref{Eq: relation of norm of Phi w*}), we have
    \begin{align*}
        \left(1-\frac{\norm{\mbf{K}_{-S}}_{op}}{4n\lambda}\right)\norm{\Delta\mbf{\theta}_S} &\leq\frac{\sqrt{\mrm{tr}(\mbf{K}_{-S})}}{n\lambda}+\frac{\sqrt{\norm{\mbf{K}_{-S}}_{op}}}{n\lambda}\left(\frac{1}{4}\norm{\mbf{e}_{-S}}+\sqrt{2\log\frac{1}{\delta}}\right)\\
        & \lesssim\frac{1}{\sqrt{n}\lambda}+\frac{\sqrt{\norm{\mbf{K}_{-S}}_{op}}}{n\lambda}\left(\frac{1}{4}\norm{\mbf{e}_{-S}}+\sqrt{2\log\frac{1}{\delta}}\right)
    \end{align*}
    Based on (A2), we have 
    \begin{align*}
        \frac{\norm{\mbf{K}_{-S}}_{op}}{4n\lambda}=\frac{1}{n}\frac{\norm{\mbf{K}_{-S}}_{op}}{4}\left(1\wedge\frac{4}{\norm{\mbf{K}_{-S}}_{op}}\right)o\left(n^{\frac{\alpha-1}{2(\alpha+1)}} \wedge n^{1/4}\right)=o\left(n^{-\frac{\alpha+3}{2(\alpha+1)}}\wedge n^{-3/4}\right)\in(0,1),\numberthis\label{Eq: mu1/(4n*lambda)}
    \end{align*}
    In addition, note that $1/(1+\mcal{O}(1))=\mcal{O}(1)$, we have  
    \begin{align*}
        \left(1-\frac{\norm{\mbf{K}_{-S}}_{op}}{4n\lambda}\right)^{-1}\frac{1}{\sqrt{n}\lambda} & =\mcal{O}(1)o(n^{-\frac{1}{\alpha+1}}\wedge n^{-\frac{1}{4}})=o(n^{-\frac{1}{4}})\\
        \left(1-\frac{\norm{\mbf{K}_{-S}}_{op}}{4n\lambda}\right)^{-1}\frac{\sqrt{\norm{\mbf{K}_{-S}}_{op}}}{n\lambda} & =\mcal{O}(1)\cdot\sqrt{\frac{1}{n\lambda}}\sqrt{\frac{\norm{\mbf{K}_{-S}}_{op}}{n\lambda}}\\
        & =o(n^{-\frac{\alpha+3}{4(\alpha+1)}})o(n^{-\frac{\alpha+3}{4(\alpha+1)}})\\
        & =o(n^{-\frac{\alpha+3}{2(\alpha+1)}}).
    \end{align*}
    Therefore, under assumption (A2), combined with Corollary \ref{Cor: Bound for e}, with probability at least $1-7\delta$,
    \begin{align*}
        \norm{\Delta\mbf{\theta}_S} & \leq\left(1-\frac{\norm{\mbf{K}_{-S}}_{op}}{4n\lambda}\right)^{-1}\frac{1}{\sqrt{n}\lambda} + \left(1-\frac{\norm{\mbf{K}_{-S}}_{op}}{4n\lambda}\right)^{-1}\frac{\sqrt{\norm{\mbf{K}_{-S}}_{op}}}{n\lambda}\left(\frac{1}{4}\norm{\mbf{e}_{-S}}+\sqrt{2\log\frac{1}{\delta}}\right)\\
        & \leq o\left(n^{-\frac{1}{4}}\right)+o(n^{-\frac{\alpha+3}{2(\alpha+1)}})\left(\sqrt{W}\log^{1/2}n+\sqrt{\log\frac{1}{\delta}}\right)
    \end{align*}

    We now prove the result for $\norm{\mbf{\Phi}_{-S}\Delta\mbf{\theta}_S}$. Based on (\ref{Eq: Necessary Condition of Minimization Delta theta S}) and triangle inequality,
    \begin{align*}
        \norm{\mbf{\Phi}_{-S}\Delta\mbf{\theta}_S} & =\frac{1}{n\lambda}\norm{\mbf{\Phi}_{-S}\mbf{\Phi}_{-S}^T\left[\mbf{Y}-\sigma(\mbf{h}_{\mbf{\theta}_f, -S}+\mbf{\Phi}_{-S}\Delta\mbf{\theta}_S)\right]}\\
        & =\frac{1}{n\lambda}\norm{\mbf{K}_{-S}\left[\mbf{Y}-\sigma(\mbf{h}_{\mbf{\theta}_f, -S}+\mbf{\Phi}_{-S}\Delta\mbf{\theta}_S)\right]}\\
        & \leq\frac{1}{n\lambda}\norm{\mbf{K}_{-S}\mbf{\epsilon}}+\frac{1}{n\lambda}\norm{\mbf{K}_{-S}\left[\sigma(\mbf{f}_{0,-S})-\sigma(\mbf{h}_{\mbf{\theta}_f,-S}+\mbf{\Phi}_{-S}\Delta\mbf{\theta}_S)\right]}.
    \end{align*}
Since elements in $\mbf{\epsilon}$ are independent, mean-zero sub-Gaussian random variables, it follows from Theorem 1 in \citet{hsu2012tail} that with probability at least $1-\delta$,
$$
\norm{\mbf{K}_{-S}\mbf{\epsilon}}^2 \leq\mrm{tr}\left(\mbf{K}_{-S}^2\right)+2\sqrt{\mrm{tr}\left(\mbf{K}_{-S}^4\right)\log\frac{1}{\delta}}+2\norm{\mbf{K}_{-S}}_{op}^2\log\frac{1}{\delta}.
$$
Note that
\begin{align*}
    \norm{\mbf{K}_{-S}}_{op}^2 & =\norm{\mbf{K}_{-S}}_{op}^2\\
    \mrm{tr}\left(\mbf{K}_{-S}^2\right) & \leq\norm{\mbf{K}_{-S}}_{op}\mrm{tr}\left(\mbf{K}_{-S}\right)\\
    \mrm{tr}\left(\mbf{K}_{-S}^4\right) & \leq\norm{\mbf{K}_{-S}^2}_{op}\mrm{tr}\left(\mbf{K}_{-S}^2\right)\leq\norm{\mbf{K}_{-S}}_{op}^3\mrm{tr}(\mbf{K}_{-S}),
\end{align*}
then with probability at least $1-\delta$,
\begin{align*}
    \norm{\mbf{K}_{-S}\mbf{\epsilon}}^2 & \leq\norm{\mbf{K}_{-S}}_{op}\mrm{tr}(\mbf{K}_{-S})+2\norm{\mbf{K}_{-S}}_{op}\sqrt{\norm{\mbf{K}_{-S}}_{op}\mrm{tr}(\mbf{K}_{-S})\log\frac{1}{\delta}}+2\norm{\mbf{K}_{-S}}_{op}^2\log\frac{1}{\delta}\\
    & =\norm{\mbf{K}_{-S}}_{op}\left[\mrm{tr}(\mbf{K}_{-S})+2\sqrt{\norm{\mbf{K}_{-S}}_{op}\mrm{tr}(\mbf{K}_{-S})\log\frac{1}{\delta}}+2\norm{\mbf{K}_{-S}}_{op}\log\frac{1}{\delta}\right]\\
    & \leq\norm{\mbf{K}_{-S}}_{op}\left[\sqrt{\mrm{tr}(\mbf{K}_{-S})}+\sqrt{2\norm{\mbf{K}_{-S}}\log\frac{1}{\delta}}\right]^2.
\end{align*}
Hence, with probability at least $1-\delta$,
$$
\norm{\mbf{K}_{-S}\mbf{\epsilon}}\leq\sqrt{\norm{\mbf{K}_{-S}}_{op}}\left[\sqrt{\mrm{tr}(\mbf{K}_{-S})}+\sqrt{2\norm{\mbf{K}_{-S}}\log\frac{1}{\delta}}\right].
$$
For the second term, by (\ref{Eq: MVT for sigma}) we have with probability at least $1-\delta$,
\begin{align*}
    \norm{\mbf{\Phi}_{-S}\Delta\mbf{\theta}_S} & \leq\frac{1}{n\lambda}\sqrt{\norm{\mbf{K}_{-S}}_{op}}\left[\sqrt{\mrm{tr}(\mbf{K}_{-S})}+\sqrt{2\norm{\mbf{K}_{-S}}\log\frac{1}{\delta}}\right]+\frac{1}{n\lambda}\norm{\mbf{K}_{-S}\mbf{T}\left[\mbf{e}_{-S}+\mbf{\Phi}_{-S}\Delta\mbf{\theta}_S\right]}\\
    & \leq\frac{1}{n\lambda}\sqrt{\norm{\mbf{K}_{-S}}_{op}}\left[\sqrt{\mrm{tr}(\mbf{K}_{-S})}+\sqrt{2\norm{\mbf{K}_{-S}}\log\frac{1}{\delta}}\right]+\frac{1}{4n\lambda}\norm{\mbf{K}_{-S}}_{op}\left(\norm{\mbf{e}_{-S}}+\norm{\mbf{\Phi}_{-S}\Delta\mbf{\theta}_S}\right)\\
    & \lesssim\sqrt{\frac{\norm{\mbf{K}_{-S}}_{op}}{n\lambda^2}}+\frac{\norm{\mbf{K}_{-S}}_{op}}{n\lambda}\left(\sqrt{2\log\frac{1}{\delta}}+\frac{1}{4}\norm{\mbf{e}_{-S}}\right)+\frac{\norm{\mbf{K}_{-S}}_{op}}{4n\lambda}\norm{\mbf{\Phi}_{-S}\Delta\mbf{\theta}_S},\numberthis\label{Eq: relation of norm of Phi w*}
\end{align*}
where the last inequality follows from Lemma \ref{Lm: Bound Derivative of Deep ReLU Neural Networks}. Therefore, by rearranging (\ref{Eq: relation of norm of Phi w*}), we have
$$
    \left(1-\frac{\norm{\mbf{K}_{-S}}_{op}}{4n\lambda}\right)\norm{\mbf{\Phi}_{-S}\Delta\mbf{\theta}_S}\lesssim\sqrt{\frac{\norm{\mbf{K}_{-S}}_{op}}{n\lambda^2}}+\frac{\norm{\mbf{K}_{-S}}_{op}}{n\lambda}\left(\sqrt{2\log\frac{1}{\delta}}+\frac{1}{4}\norm{\mbf{e}_{-S}}\right).
$$
Since $x/(1-x)=\mcal{O}(x)$ and $\sqrt{x}/(1-x)=\mcal{O}(\sqrt{x})$ as $x\to0$, based on (\ref{Eq: mu1/(4n*lambda)}), we have 
\begin{align*}
    \left(1-\frac{\norm{\mbf{K}_{-S}}_{op}}{4n\lambda}\right)^{-1}\frac{\norm{\mbf{K}_{-S}}_{op}}{4n\lambda} & =o(n^{-\frac{\alpha+3}{2(\alpha+1)}})\\
    \left(1-\frac{\norm{\mbf{K}_{-S}}_{op}}{4n\lambda}\right)^{-1}\sqrt{\frac{\norm{\mbf{K}_{-S}}_{op}}{n\lambda^2}} & =\left(1-\frac{\norm{\mbf{K}_{-S}}_{op}}{4n\lambda}\right)^{-1}\sqrt{\frac{\norm{\mbf{K}_{-S}}_{op}}{n\lambda}}\sqrt{\frac{1}{\lambda}}\\
    & =\mcal{O}(o(n^{-3/8}))o(n^{1/8})=o(n^{-1/4}).
\end{align*}
Combine everything together, we have under assumption (A2), with probability at least $1-7\delta$,
$$
\norm{\mbf{\Phi}_{-S}\Delta\mbf{\theta}_S}\lesssim o(n^{-\frac{1}{4}})+o(n^{-\frac{\alpha+3}{2(\alpha+1)}})\left(\sqrt{W}\log^{1/2}n+\sqrt{\log\frac{1}{\delta}}\right)
$$
\end{proof}

\subsection{Proof of Theorem \ref{Thm: Distance between h tilde and h}}
\begin{proof}
    Applying Taylor's theorem similar to equation (1.3.2) in \citet{misiakiewicz2024six}, we have
    \begin{align*}
        \abs{h_{\mbf{\theta}_f+\Delta\theta_S}(\mbf{x})-\tilde{h}_{\mbf{\theta}_f+\Delta\theta_S}(\mbf{x})} & =\abs{h_{\mbf{\theta}_f+\Delta\theta_S}(\mbf{x}) - h_{\mbf{\theta}_f}(\mbf{x})-[\nabla_{\mbf{\theta}}h_{\mbf{\theta}_f}(\mbf{x})]^T\Delta\theta_S}\\
        %& =\abs{\int_0^1\left[\nabla_{\mbf{\theta}}h_{(1-t)\mbf{\theta}_f+t(\mbf{\theta}_f+\Delta\theta_S)}(\mbf{x})-\nabla_{\mbf{\theta}}h_{\mbf{\theta}_f}(\mbf{x})\right]^T\Delta\theta_Sdt}\\
        & =\abs{\int_0^1\left[\nabla_{\mbf{\theta}}h_{\mbf{\theta}_f+t\Delta\theta_S}(\mbf{x})-\nabla_{\mbf{\theta}}h_{\mbf{\theta}_f}(\mbf{x})\right]^T\Delta\theta_Sdt}\\
        & \leq\norm{\Delta\theta_S}\int_0^1\norm{\nabla_{\mbf{\theta}}h_{\mbf{\theta}_f+t\Delta\theta_S}(\mbf{x})-\nabla_{\mbf{\theta}}h_{\mbf{\theta}_f}(\mbf{x})}dt\\
        & \leq\norm{\Delta\theta_S}\left[\int_0^1\norm{\nabla_{\mbf{\theta}}h_{\mbf{\theta}_f+t\Delta\theta_S}(\mbf{x})}dt+\norm{\nabla_{\mbf{\theta}}h_{\mbf{\theta}_f}(\mbf{x})}\right]
    \end{align*}
    where the last two inequalities follow from the Cauchy-Schwarz inequality and the triangle inequality. According to Lemma \ref{Lm: Bound Derivative of Deep ReLU Neural Networks}, the norms of the gradient are upper bounded by $\kappa\left(\prod_{j=1}^{L-1}\kappa_j\right)\left(1+b_L\sqrt{\sum_{l=1}^{L-1}\frac{1}{\kappa_l^2}}\right)$. Therefore, for any $\mbf{x}$ satisfying $\norm{\mbf{x}}\leq\kappa$,
    \begin{align*}
        \abs{h_{\mbf{\theta}_f+\Delta\theta_S}(\mbf{x})-\tilde{h}_{\mbf{\theta}_f+\Delta\theta_S}(\mbf{x})}\leq2\norm{\Delta\theta_S}\kappa\left(\prod_{j=1}^{L-1}\kappa_j\right)\left(1+b_L\sqrt{\sum_{l=1}^{L-1}\frac{1}{\kappa_l^2}}\right).
    \end{align*}
    Denote $\Omega_n^{(1)}=o(n^{-\frac{1}{4}})$, $\Omega_n^{(2)}=o(n^{-\frac{\alpha+3}{2(\alpha+1)}})$ and $b^{(\mbf{e})}_n=\sqrt{W}\log^{1/2}n$. Then from Proposition \ref{Prop: Bound on norm of Delta theta S}, we know that
    $$
    \mbb{P}\left(\norm{\Delta\mbf{\theta}_S}\gtrsim \Omega_n^{(1)}+\Omega_n^{(2)}b_n^{(\mbf{e})}+\Omega_n^{(2)}\sqrt{\log\frac{1}{\delta}}\right)\lesssim\delta.
    $$
    Let $a_n=\Omega_n^{(1)}+\Omega_n^{(2)}b_n^{(\mbf{e})}$. Given $u>a_n$, take $\delta=\exp\left\{-\left(\frac{u-a_n}{\Omega_n^{(2)}}\right)^2\right\}$, then $\delta\in(0,1)$ and $\Omega_n^{(2)}\sqrt{\log\frac{1}{\delta}}=u-a_n$. Therefore, for $u>a_n$, $\mbb{P}(\norm{\Delta\mbf{\theta}_S}>u)\lesssim e^{-\left(\frac{u-a_n}{\Omega_n^{(2)}}\right)^2}$, and hence
    \begin{align*}
        \norm{h_{\mbf{\theta}_f+\Delta\mbf{\theta}_S}-\tilde{h}_{\mbf{\theta}_f+\Delta\mbf{\theta}_S}}^2 & =\mbb{E}\left[\abs{h_{\mbf{\theta}_f+\Delta\mbf{\theta}_S}(\mbf{X})-\tilde{h}_{\mbf{\theta}_f+\Delta\mbf{\theta}_S}(\mbf{X})}^2\right]\lesssim\mbb{E}\left[\norm{\Delta\mbf{\theta}_S}^2\right]\\
        & =\left(\int_0^{a_n}+\int_{a_n}^\infty \right)2u\mbb{P}(\norm{\Delta\mbf{\theta}_S}>u)du\\
        & \lesssim a_n^2+\int_0^\infty(v+a_n)e^{-\frac{v^2}{\Omega_n^{(2)^2}}}dv\\
        & =a_n^2+\frac{1}{2}\Omega_n^{(2)}\left(\Omega_n^{(2)}+a_n\sqrt{2\pi}\right)\\
        &  \lesssim \Omega_n^{(1)^2}+\Omega_n^{(2)^2}b_n^{(\mbf{e})^2}+\Omega_n^{(2)^2}+\Omega_n^{(1)}\Omega_n^{(2)}+\Omega_n^{(2)^2}b_n^{(\mbf{e})}.
    \end{align*}
    Since $\Omega_n^{(1)^2}=o(n^{-\frac{1}{2}})$, $\Omega_n^{(2)^2}=o(n^{-\frac{\alpha+3}{\alpha+1}})$, $\Omega_n^{(1)}\Omega_n^{(2)}=o(n^{-\left(\frac{1}{4}+\frac{\alpha+3}{2(\alpha+1)}\right)})=o(n^{-\frac{1}{2}})$, $\Omega_n^{(2)^2}b_n^{(\mbf{e})}=o(n^{-\frac{\alpha+3}{\alpha+1}})\sqrt{W}\log^{1/2}n$ and $\Omega_n^{(2)^2}b_n^{(\mbf{e})^2}=o(n^{-\frac{\alpha+3}{\alpha+1}})W\log n$, we have
    $$
    \norm{h_{\mbf{\theta}_f+\Delta\mbf{\theta}_S}-\tilde{h}_{\mbf{\theta}_f+\Delta\mbf{\theta}_S}}^2\lesssim o(n^{-\frac{1}{2}})+o(n^{-\frac{\alpha+3}{\alpha+1}})W\log n
    $$
    which implies that
    $$
    \norm{h_{\mbf{\theta}_f+\Delta\mbf{\theta}_S}-\tilde{h}_{\mbf{\theta}_f+\Delta\mbf{\theta}_S}}\lesssim o(n^{-\frac{1}{4}})+o(n^{-\frac{\alpha+3}{2(\alpha+1)}})\sqrt{W}\log^{1/2}n.
    $$
\end{proof}

\section{Technical Details on the Convergence Rate of Linearized ReLU Network $\tilde{h}_{\mbf{\theta}_f+\Delta\mbf{\theta}_S}$}\label{appendix: Convergence Rate of Linearized ReLU Network}

\subsection{Proof of Lemma \ref{Lm: Empirical norm Upper bounds}}
\begin{proof}
    Denote $\mbf{Q}_{-S}^{(k)}=\mbf{K}_{-S}\left[\mbf{K}_{-S}+n\lambda\mbf{\Pi}_{-S}^{(k)^{-1}}\right]^{-1}$. First note that
    \begin{align*}
    & \sqrt{n}\norm{\tilde{h}_{\mbf{\theta}_f+\Delta\mbf{\theta}_S}^{(k+1)}-f_{0, -S}}_n = \norm{\tilde{\mbf{h}}_{\mbf{\theta}_f+\Delta\mbf{\theta}_S}^{(k+1)}-\mbf{f}_{0, -S}}\\
    =\; & \norm{\mbf{Q}_{-S}^{(k)} \left[\mbf{\Phi}_{-S}\mbf{w}^{(k)}+\mbf{\Pi}_{-S}^{(k)^{-1}}\left(\mbf{Y}-\sigma\left(\mbf{h}_{\mbf{\theta}_f, -S}+\mbf{\Phi}_{-S}\mbf{w}^{(k)}\right)\right)\right]+\mbf{h}_{\mbf{\theta}_f, -S} - \mbf{f}_{0, -S}}\\
    =\; & \norm{\mbf{Q}_{-S}^{(k)} \left[\mbf{\Phi}_{-S}\mbf{w}^{(k)}+\mbf{\Pi}_{-S}^{(k)^{-1}}\left(\sigma(\mbf{f}_{0, -S})+\mbf{\epsilon}-\sigma\left(\mbf{h}_{\mbf{\theta}_f, -S}+\mbf{\Phi}_{-S}\mbf{w}^{(k)}\right)\right)\right]+\mbf{h}_{\mbf{\theta}_f, -S} - \mbf{f}_{0, -S}}\\
    \overset{(1)}{\leq}\; & \norm{\mbf{Q}_{-S}^{(k)}\mbf{\Phi}_{-S}\mbf{w}^{(k)}} 
    + \norm{\mbf{Q}_{-S}^{(k)}\mbf{\Pi}_{-S}^{(k)^{-1}}\left(\sigma\left(\mbf{f}_{0,-S}\right) + \mbf{\epsilon} - \sigma\left(\mbf{h}_{\mbf{\theta}_f, -S}+\mbf{\Phi}_{-S}\mbf{w}^{(k)}\right)\right)} + \norm{\mbf{Q}_{-S}^{(k)}\left(\mbf{h}_{\mbf{\theta}_f,-S}-\mbf{f}_{0,-S}\right)}\\
     \overset{(2)}{\leq}\; & \norm{\mbf{Q}_{-S}^{(k)}\mbf{\Phi}_{-S}\mbf{w}^{(k)}} 
    + \norm{\mbf{Q}_{-S}^{(k)}\mbf{\Pi}_{-S}^{(k)^{-1}}\left(\sigma\left(\mbf{f}_{0,-S}\right)-\sigma\left(\mbf{h}_{\mbf{\theta}_f, -S}+\mbf{\Phi}_{-S}\mbf{w}^{(k)}\right)\right)} \\
    &  + \norm{\mbf{Q}_{-S}^{(k)}\left(\mbf{h}_{\mbf{\theta}_f,-S}-\mbf{f}_{0,-S}\right)} 
    + \norm{\mbf{Q}_{-S}^{(k)}\mbf{\Pi}_{-S}^{(k)^{-1}}\mbf{\epsilon}}\\
     \overset{(3)}{\leq}\; & {\norm{\mbf{Q}_{-S}^{(k)}}}_{op} \norm{\mbf{\Phi}_{-S}\mbf{w}^{(k)}} 
    + {\norm{\mbf{Q}_{-S}^{(k)}\mbf{\Pi}_{-S}^{(k)^{-1}}}}_{op} \norm{\sigma\left(\mbf{f}_{0,-S}\right)-\sigma\left(\mbf{h}_{\mbf{\theta}_f, -S}+\mbf{\Phi}_{-S}\mbf{w}^{(k)}\right)} \\
    & + {\norm{\mbf{Q}_{-S}^{(k)}}}_{op} \norm{\mbf{h}_{\mbf{\theta}_f,-S}-\mbf{f}_{0,-S}} 
    + \norm{\mbf{Q}_{-S}^{(k)}\mbf{\Pi}_{-S}^{(k)^{-1}}\mbf{\epsilon}}\\
     \overset{(4)}{\leq}\; & {\norm{\mbf{Q}_{-S}^{(k)}}}_{op} \norm{\mbf{\Phi}_{-S}\mbf{w}^{(k)}} 
    + {\norm{\mbf{Q}_{-S}^{(k)}\mbf{\Pi}_{-S}^{(k)^{-1}}}}_{op} \frac{1}{4}\norm{\mbf{f}_{0,-S}-\mbf{h}_{\mbf{\theta}_f, -S}-\mbf{\Phi}_{-S}\mbf{w}^{(k)}} \\
    & + {\norm{\mbf{Q}_{-S}^{(k)}}}_{op} \norm{\mbf{h}_{\mbf{\theta}_f,-S}-\mbf{f}_{0,-S}} 
    + \norm{\mbf{Q}_{-S}^{(k)}\mbf{\Pi}_{-S}^{(k)^{-1}}\mbf{\epsilon}}\\
     \overset{(5)}{\leq}\; & {\norm{\mbf{Q}_{-S}^{(k)}}}_{op} \norm{\mbf{\Phi}_{-S}\mbf{w}^{(k)}} 
    + {\norm{\mbf{Q}_{-S}^{(k)}\mbf{\Pi}_{-S}^{(k)^{-1}}}}_{op} \frac{1}{4}\Big(\norm{\mbf{f}_{0,-S}-\mbf{h}_{\mbf{\theta}_f, -S}}+\norm{\mbf{\Phi}_{-S}\mbf{w}^{(k)}}\Big) \\
    & + {\norm{\mbf{Q}_{-S}^{(k)}}}_{op} \norm{\mbf{h}_{\mbf{\theta}_f,-S}-\mbf{f}_{0,-S}} 
    + \norm{\mbf{Q}_{-S}^{(k)}\mbf{\Pi}_{-S}^{(k)^{-1}}\mbf{\epsilon}}\\
    =\; & \bigg({\norm{\mbf{Q}_{-S}^{(k)}}}_{op} +\frac{{\norm{\mbf{Q}_{-S}^{(k)}\mbf{\Pi}_{-S}^{(k)^{-1}}}}_{op}}{4}\bigg)\norm{\mbf{\Phi}_{-S}\mbf{w}^{(k)}} 
    + \bigg({\norm{\mbf{Q}_{-S}^{(k)}}}_{op} +\frac{{\norm{\mbf{Q}_{-S}^{(k)}\mbf{\Pi}_{-S}^{(k)^{-1}}}}_{op}}{4}\bigg)\norm{\mbf{f}_{0,-S}-\mbf{h}_{\mbf{\theta}_f, -S}} \\
    & + \norm{\mbf{Q}_{-S}^{(k)}\mbf{\Pi}_{-S}^{(k)^{-1}}\mbf{\epsilon}}\\
    =:\; & I_0I_1 + I_0I_2 + I_3,
\end{align*}
where (1), (2), and (5) hold from the triangle inequality, (3) comes from the definition of the operator norm, and (4) follows from Lemma \ref{lem:sig_lip}. We now provide bounds for each term. 
\begin{itemize}
    \item \textit{Bound for $\norm{\mbf{w}^{(k)}}$}. We derive the bound for $\norm{\mbf{w}^{(k)}}$ by iterating Newton's updating equation $k$ times. \citet{nocedal2006optimization}\footnote{\textbf{Theorem}\citep{nocedal2006optimization}
    Suppose that f is twice differentiable, that the Hessian $\nabla^2f(x)$ is Lipschitz continuous, and that $\nabla f(x^*)=0$ and $\nabla^2f(x^*)$ is positive definite. Consider the iteration $x_{k+1}=x_k-\nabla^2f(x_k)^{-1}\nabla f(x_k)$. Then (1) if the starting point $x_0$ is sufficiently close to $x^*$, the sequence of iterates converges to $x^*$ and (2) the rate of convergence of {$x_k$} is quadratic, i.e., there exists some constant $\Lambda>0$ such that $\norm{x^*-x_{k+1}}<\Lambda\norm{x^*-x_k}^2$ for $k\in\mathbb{N}$.} shows that when the target function is twice differentiable and the Hessian is Lipschitz continuous and positive definite, the rate of convergence of Newton's method is quadratic. Since $L_\lambda(\mbf{w})$ is clearly twice differentiable and the Hesssian $\nabla^2L_\lambda(\mbf{w})=\mbf{\Phi}_{-S}^T\mbf{\Pi}_{-S}\mbf{\Phi}_{-S}+n\lambda\mbf{I}_W$ is positive definite due to $\mbf{\Phi}_{-S}^T\mbf{\Pi}_{-S}\mbf{\Phi}_{-S}$ is positive semidefinite. On the other hand, the Hessian is also Lipschitz continuous by noting that
    \begin{align*}
        & \norm{\nabla^2L_\lambda(\mbf{w})-\nabla^2L_\lambda(\tilde{\mbf{w}})}_F = \norm{\mbf{\Phi}_{-S}^T\left(\mbf{\Pi}_{-S}-\tilde{\mbf{\Pi}}_{-S}\right)\mbf{\Phi}_{-S}}_{F}\\
        = & \sqrt{\text{tr}\left(\left(\mbf{\Pi}_{-S}-\tilde{\mbf{\Pi}}_{-S}\right)\mbf{\Phi}_{-S}\mbf{\Phi}_{-S}^T\left(\mbf{\Pi}_{-S}-\tilde{\mbf{\Pi}}_{-S}\right)\mbf{\Phi}_{-S}\mbf{\Phi}_{-S}^T\right)}\\
        \leq & \sqrt{\norm{\mbf{K}_{-S}}_{op}\text{tr}\left(\left(\mbf{\Pi}_{-S}-\tilde{\mbf{\Pi}}_{-S}\right)^2\mbf{K}_{-S}\right)}\\
        \leq & \norm{\mbf{K}_{-S}}_{op}\sqrt{\text{tr}\left(\left(\mbf{\Pi}_{-S}-\tilde{\mbf{\Pi}}_{-S}\right)^2\right)}\\
        = & \norm{\mbf{K}_{-S}}_{op}\sqrt{\sum_{i=1}^n\left[\sigma'\left(\left[\nabla_{\theta}h_{\mbf{\theta}_f}(\mbf{X}_{i,-S})\right]^T\mbf{w}\right)-\sigma'\left(\left[\nabla_{\theta}h_{\mbf{\theta}_f}(\mbf{X}_{i,-S})\right]^T\tilde{\mbf{w}}\right)\right]^2}\\
        \overset{(1)}{\leq} & \norm{\mbf{K}_{-S}}_{op}\sqrt{\sum_{i=1}^n\left[\frac{1}{4}\left[\nabla_{\theta}h_{\mbf{\theta}_f}(\mbf{X}_{i,-S})\right]^T(\mbf{w}-\tilde{\mbf{w}})\right]^2}\\
        \leq & \frac{1}{4} \norm{\mbf{K}_{-S}}_{op}\sqrt{(\mbf{w}-\tilde{\mbf{w}})^T\left(\sum_{i=1}^n\nabla_{\theta}h_{\mbf{\theta}_f}(\mbf{X}_{i,-S})\nabla_{\theta}h_{\mbf{\theta}_f}(\mbf{X}_{i,-S})^T\right)(\mbf{w}-\tilde{\mbf{w}})}\\
        \overset{(2)}{\leq} & \frac{1}{4} \norm{\mbf{K}_{-S}}_{op}\sqrt{\lambda_{\text{max}}\left(\sum_{i=1}^n\left[\nabla_{\theta}h_{\mbf{\theta}_f}(\mbf{X}_{i,-S})\right]\left[\nabla_{\theta}h_{\mbf{\theta}_f}(\mbf{X}_{i,-S})\right]^T\right)}\norm{\mbf{w}-\tilde{\mbf{w}}}\\
        = & \frac{1}{4} \norm{\mbf{K}_{-S}}_{op}^{3/2}\norm{\mbf{w}-\tilde{\mbf{w}}},
\end{align*}
where (1) follows from the Lipschitz continuity of $\sigma'$ and (2) follows from the Rayleigh quotient property of symmetric matrices. Notice $\frac{1}{4} \norm{\mbf{K}_{-S}}_{op}^{3/2}$ does not depend on $\mbf{w}$, proving the Hessian is Lipschitz. Thus, we know that $\norm{\Delta\mbf{\theta}_S - \mbf{w}^{(k+1)}}<\Lambda\norm{\Delta\mbf{\theta}_S-\mbf{w}^{(k)}}^2$. Since $\mbf{w}^{(k))}\to\Delta\mbf{\theta}_S$ as $k\to\infty$, there exists $\bar{K}>0$ such that $\norm{\mbf{w}^{(k)}-\Delta\mbf{\theta}_S}<\frac{1}{\Lambda}$. Therefore for all $k\geq \bar{K}$, combining the assumption (A2) with Proposition \ref{Prop: Bound on norm of Delta theta S}, we have
\begin{align*}
    \norm{\mbf{w}^{(k)}} & \leq \norm{\Delta\mbf{\theta}_S}+\norm{\mbf{w}^{(k)}-\Delta\mbf{\theta}_S}\\
    & \leq\norm{\Delta\mbf{\theta}_S} +\Lambda\norm{\mbf{w}^{(k-1)}-\Delta\mbf{\theta}_S}^2\\
    & \leq\norm{\Delta\mbf{\theta}_S}+\Lambda^3\norm{\mbf{w}^{(k-2)}-\Delta\mbf{\theta}_S}^4\\
    & \leq\cdots\\
    & \leq \norm{\Delta\mbf{\theta}_S}+\Lambda^{2^{k-\bar{K}}-1}\norm{\mbf{w}^{(K)}-\Delta\mbf{\theta}_S}^{2^{k-\bar{K}}}=O(1).
\end{align*}

\item \textit{Bound for $\norm{\mbf{\Pi}_{-S}^{(k)^{-1}}}_{op}$}. Recall that
\[
\mbf{\Pi}_{-S}^{(k)} = \operatorname{Diag}\left( \sigma\big( \mbf{h}_{\mbf{\theta}_f, -S} + \mbf{\Phi}_{-S} \mbf{w}^{(k)} \big) \cdot \big( 1 - \sigma\big( \mbf{h}_{\mbf{\theta}_f, -S} + \mbf{\Phi}_{-S} \mbf{w}^{(k)} \big) \big) \right),
\]
so then
\[
\mbf{\Pi}_{-S}^{(k)^{-1}} = \operatorname{Diag} \left( \left[ \sigma\big( \mbf{h}_{\mbf{\theta}_f, -S} + \mbf{\Phi}_{-S} \mbf{w}^{(k)} \big) \cdot \big( 1 - \sigma\big( \mbf{h}_{\mbf{\theta}_f, -S} + \mbf{\Phi}_{-S} \mbf{w}^{(k)} \big) \big) \right]^{-1} \right).
\]
Note that for any \( z \in \mathbb{R} \), we have
\[
\sigma(z)(1 - \sigma(z)) = \frac{e^z}{(1 + e^z)^2},
\]
so the reciprocal 
\[
\big[\sigma(z)(1-\sigma(z))\big]^{-1}=\frac{(1 + e^z)^2}{e^z} = \frac{1+2e^z+e^{2z}}{e^z}= e^z + e^{-z} + 2.
\]
Therefore,
\begin{align*}
\| \mbf{\Pi}_{-S}^{(k)^{-1}} \|_{\text{op}} 
&= \max_{1 \leq i \leq n} \left\{ 
\frac{ \left( 1 + e^{h_{\mbf{\theta}_f, -S}(\mbf{X}_{i,-S}) + [\nabla_\theta h_{\mbf{\theta}_f}(\mbf{X}_{i,-S})]^T \mbf{w}^{(k)}} \right)^2 }
     { e^{h_{\mbf{\theta}_f, -S}(\mbf{X}_{i,-S}) + [\nabla_\theta h_{\mbf{\theta}_f}(\mbf{X}_{i,-S})]^T \mbf{w}^{(k)}} }
\right\} \\
&= \max_{1 \leq i \leq n} \left\{ 
e^{\gamma_i} + e^{-\gamma_i} + 2
\right\},
\end{align*}
\noindent where 
\[
\gamma_i \coloneq h_{\mbf{\theta}_f, -S}(\mbf{X}_{i,-S}) + [\nabla_\theta h_{\mbf{\theta}_f}(\mbf{X}_{i,-S})]^T \mbf{w}^{(k)}.
\]
Since functions in $\mcal{F}$ are uniformly bounded by $M$, we have 
$$\abs{h_{\mbf{\theta}_f,-S}(\mbf{X}_{-S})} \leq M.$$
Then
\begin{align*}
    \abs{\nabla_{\mbf{\theta}} h_{\mbf{\theta}_f}(\mbf{X}_{i,-S})^T\mbf{w}^{(k)}} & \overset{(1)}{\leq} \norm{\nabla_{\mbf{\theta}} h_{\mbf{\theta}_f}(\mbf{X}_{i,-S})}\norm{\mbf{w}^{(k)}}\\
    & = \sqrt{\left[\nabla_{\mbf{\theta}} h_{\mbf{\theta}_f}(\mbf{X}_{i,-S})\right]^T\nabla_{\mbf{\theta}} h_{\mbf{\theta}_f}(\mbf{X}_{i,-S})}\norm{\mbf{w}^{(k)}}\\
    & = \sqrt{[\mbf{K}_{-S}]_{ii}}\norm{\mbf{w}^{(k)}}\\
    & \leq \sqrt{\text{sup}_{\mbf{x}}\mbf{K}_{-S}(\mbf{x},\mbf{x})}\norm{\mbf{w}^{(k)}}\\
    & \overset{(2)}{=} \mathcal{O}(1)
\end{align*} where (1) holds by Cauchy–Schwarz and (2) holds from Lemma \ref{Lm: Bound Derivative of Deep ReLU Neural Networks} and the bound for $\norm{\mbf{w}^{(k)}}$. So, 
\begin{align*}
    \max_{1 \leq i \leq n} \left\{ e^{\gamma_i} + e^{-\gamma_i} + 2\right\} & \leq \max_{1 \leq i \leq n} \left\{ e^{\gamma_i} \right\} + \max_{1 \leq i \leq n} \left\{e^{-\gamma_i}\right\} + 2\\
    & =  \max_{1 \leq i \leq n} \left\{ e^{h_{\mbf{\theta}_f, -S}(\mbf{X}_{i,-S}) + [\nabla_\theta h_{\mbf{\theta}_f}(\mbf{X}_{i,-S})]^T \mbf{w}^{(k)}} \right\} \\ &\qquad + \max_{1 \leq i \leq n} \left\{e^{-h_{\mbf{\theta}_f, -S}(\mbf{X}_{i,-S}) - [\nabla_\theta h_{\mbf{\theta}_f}(\mbf{X}_{i,-S})]^T \mbf{w}^{(k)}}\right\} + 2\\
    & \leq \max_{1 \leq i \leq n} \left\{ e^{M + \sqrt{\text{sup}_{\mbf{x}}\mbf{K}_{-S}(\mbf{x},\mbf{x})}\norm{\mbf{w}^{(k)}}} \right\} \\&\qquad+ \max_{1 \leq i \leq n} \left\{e^{M + \sqrt{\text{sup}_{\mbf{x}}\mbf{K}_{-S}(\mbf{x},\mbf{x})}\norm{\mbf{w}^{(k)}}}\right\} + 2\\
    & \leq 2  e^{M + \sqrt{\text{sup}_{\mbf{x}}\mbf{K}_{-S}(\mbf{x},\mbf{x})}\norm{\mbf{w}^{(k)}}}  + 2\\
    & = \mcal{O}(1)
\end{align*}
Therefore 
\begin{equation}\label{Eq: Bound for Pi inv}
\| \mathbf{\Pi}_{-S}^{(k)^{-1}} \|_{\text{op}}  = \mathcal{O}(1).
\end{equation}

\item \textit{Bound for $I_0$}. We have the following bound on $\norm{\mbf{Q}_{-S}^{(k)}}_{op}$:
\begin{align*}
    \norm{\mbf{Q}_{-S}^{(k)}}_{op} & =\norm{\mbf{K}_{-S}\left[\mbf{K}_{-S}+n\lambda\mbf{\Pi}_{-S}^{(k)^{-1}}\right]^{-1}}_{op}\\
    & =\norm{\mbf{K}_{-S}\mbf{\Pi}_{-S}^{(k)^{1/2}}\left[\mbf{\Pi}_{-S}^{(k)^{1/2}}\mbf{K}_{-S}\mbf{\Pi}_{-S}^{(k)^{1/2}}+n\lambda\mbf{I}_n\right]^{-1}\mbf{\Pi}_{-S}^{(k)^{1/2}}}_{op}\\
    & =\norm{\mbf{\Pi}_{-S}^{(k)^{1/2}}\mbf{K}_{-S}\mbf{\Pi}_{-S}^{(k)^{1/2}}\left[\mbf{\Pi}_{-S}^{(k)^{1/2}}\mbf{K}_{-S}\mbf{\Pi}_{-S}^{(k)^{1/2}}+n\lambda\mbf{I}_n\right]^{-1}}_{op}\\
    & =\frac{\norm{\mbf{\Pi}_{-S}^{(k)^{1/2}}\mbf{K}_{-S}\mbf{\Pi}_{-S}^{(k)^{1/2}}}_{op}}{\norm{\mbf{\Pi}_{-S}^{(k)^{1/2}}\mbf{K}_{-S}\mbf{\Pi}_{-S}^{(k)^{1/2}}}_{op}+n\lambda}\leq1.
\end{align*}
Using the bound for $\norm{\mbf{\Pi}_{-S}^{(k)^{-1}}}_{op}$ and the bound on $\norm{\mbf{Q}_{-S}^{(k)}}_{op}$, we have
\begin{equation*}
    I_0 = \mathcal{O}(1).
\end{equation*}

\item \textit{Bound for $I_2$}. Since $I_2=\norm{\mbf{f}_{0, -S}-\mbf{h}_{\mbf{\theta}_f, -S}}=\norm{\mbf{e}_{-S}}$, it then follows from the Corollary \ref{Cor: Bound for e} that with probability at least $1-6\delta$, $\norm{\mbf{e}_{-S}}\lesssim \sqrt{W}\log^{1/2}n+\sqrt{\log\frac{1}{\delta}}$.

\item \textit{Bound for $I_3$}. Note that $\mbf{Y}\in\{0,1\}^n$ and $\sigma(\mbf{f}_{0, -S})\in[0,1]^n$, it then follows that $\mbf{\epsilon}\in[-1,1]^n$. By Hoeffding inequality \citep{hoeffding1994probability},
$$
\mbb{P}(\left.\abs{\epsilon_i}>t\right|\mbf{X})\leq\exp\left\{-\frac{2t^2}{2^2}\right\},
$$
which implies $\epsilon_i$, $i=1,\ldots,n$ are sub-Gaussian random variables with sub-Gaussian parameter 1. Therefore, for any $\mbf{a}\in\mbb{R}^n$,

\begin{align*}
    \mbb{E}\left[\left.e^{\mbf{a}^T\mbf{\epsilon}}\right|\mbf{X}\right] & =\mbb{E}\left[\left.\prod_{i=1}^ne^{a_i\epsilon_i}\right|\mbf{X}\right]\\
    & =\prod_{i=1}^n\mbb{E}\left[\left.e^{a_i\epsilon_i}\right|\mbf{X}\right]\\
    & \leq\prod_{i=1}^ne^{\frac{a_i^2}{2}}=\exp\left\{\norm{\mbf{a}}^2/2\right\}.
\end{align*}
By Theorem 1 in \citet{hsu2012tail}, for any $t>0$,
\begin{align*}
    \mbb{P}\left(\left.\norm{\mbf{Q}_{-S}^{(k)} \mbf{\Pi}_{-S}^{(k)^{-1}}\mbf{\epsilon}}^2>\mrm{tr}\left(\mbf{\Sigma}_{-S}^{(k)}\right)+2\sqrt{\mrm{tr}\left(\mbf{\Sigma}_{-S}^{(k)^2}\right)t}+2\norm{\mbf{\Sigma}_{-S}^{(k)}}_{op}t\right|\mbf{X}\right)\leq e^{-t},
\end{align*}
where $\mbf{\Sigma}_{-S}^{(k)}=\left(\mbf{Q}_{-S}^{(k)} \mbf{\Pi}_{-S}^{(k)^{-1}}\right)^T\left(\mbf{Q}_{-S}^{(k)} \mbf{\Pi}_{-S}^{(k)^{-1}}\right)=\mbf{\Pi}_{-S}^{(k)^{-1}}\mbf{Q}_{-S}^{(k)^T}\mbf{Q}_{-S}^{(k)}\mbf{\Pi}_{-S}^{(k)^{-1}}$. Note that
\begin{align*}
    \norm{\mbf{\Sigma}_{-S}^{(k)}}_{op} & =\norm{\mbf{\Pi}_{-S}^{(k)^{-1}}\mbf{Q}_{-S}^{(k)^T}\mbf{Q}_{-S}^{(k)}\mbf{\Pi}_{-S}^{(k)^{-1}}}_{op}\\
    & \leq\norm{\mbf{\Pi}_{-S}^{(k)^{-1}}}_{op}^2\norm{\mbf{Q}_{-S}^{(k)^T}\mbf{Q}_{-S}^{(k)}}_{op}\\
    & =\norm{\mbf{\Pi}_{-S}^{(k)^{-1}}}_{op}^2\norm{\mbf{Q}_{-S}^{(k)}}_{op}^2\\
    & \leq\norm{\mbf{\Pi}_{-S}^{(k)^{-1}}}_{op}^2.
\end{align*}    
In addition, we have
\begin{align*}
    \mrm{tr}\left(\mbf{\Sigma}_{-S}^{(k)}\right) & =\mrm{tr}\left(\mbf{\Pi}_{-S}^{(k)^{-1}}\mbf{Q}_{-S}^{(k)^T}\mbf{Q}_{-S}^{(k)}\mbf{\Pi}_{-S}^{(k)^{-1}}\right)\\
    & =\mrm{tr}\left(\mbf{\Pi}_{-S}^{(k)^{-1}}\left[\mbf{K}_{-S}+n\lambda\mbf{\Pi}_{-S}^{(k)^{-1}}\right]^{-1}\mbf{K}_{-S}^2\left[\mbf{K}_{-S}+n\lambda\mbf{\Pi}_{-S}^{(k)^{-1}}\right]^{-1}\mbf{\Pi}_{-S}^{(k)^{-1}}\right)\\
    & =\mrm{tr}\left(\mbf{\Pi}_{-S}^{(k)^{-1}}\mbf{\Pi}_{-S}^{(k)^{1/2}}\left[\mbf{\Pi}_{-S}^{(k)^{1/2}}\mbf{K}_{-S}\mbf{\Pi}_{-S}^{(k)^{1/2}}+n\lambda\mbf{I}_n\right]^{-1}\mbf{\Pi}_{-S}^{(k)^{1/2}}\mbf{K}_{-S}^2\right.\\
    & \qquad\qquad\qquad\left.\mbf{\Pi}_{-S}^{(k)^{1/2}}\left[\mbf{\Pi}_{-S}^{(k)^{1/2}}\mbf{K}_{-S}\mbf{\Pi}_{-S}^{(k)^{1/2}}+n\lambda\mbf{I}_n\right]^{-1}\mbf{\Pi}_{-S}^{(k)^{1/2}}\mbf{\Pi}_{-S}^{(k)^{-1}}\right)\\
    & \overset{(a)}{\leq}\norm{\mbf{\Pi}_{-S}^{(k)^{-1}}}_{op}\mrm{tr}\left(\left[\mbf{\Pi}_{-S}^{(k)^{1/2}}\mbf{K}_{-S}\mbf{\Pi}_{-S}^{(k)^{1/2}}+n\lambda\mbf{I}_n\right]^{-2}\mbf{\Pi}_{-S}^{(k)^{1/2}}\mbf{K}_{-S}^2\mbf{\Pi}_{-S}^{(k)^{1/2}}\right)\\
    & =\norm{\mbf{\Pi}_{-S}^{(k)^{-1}}}_{op}\mrm{tr}\left(\left[\mbf{\Pi}_{-S}^{(k)^{1/2}}\mbf{K}_{-S}\mbf{\Pi}_{-S}^{(k)^{1/2}}+n\lambda\mbf{I}_n\right]^{-2}\mbf{\Pi}_{-S}^{(k)^{1/2}}\mbf{K}_{-S}\mbf{\Pi}_{-S}^{(k)^{1/2}}\right.\\
    & \qquad\qquad\qquad\qquad\left.\mbf{\Pi}_{-S}^{(k)^{-1}}\mbf{\Pi}_{-S}^{(k)^{1/2}}\mbf{K}_{-S}\mbf{\Pi}_{-S}^{(k)^{1/2}}\right)\\
    & \overset{(b)}{\leq}\norm{\mbf{\Pi}_{-S}^{(k)^{-1}}}_{op}^2\mrm{tr}\left(\left[\mbf{\Pi}_{-S}^{(k)^{1/2}}\mbf{K}_{-S}\mbf{\Pi}_{-S}^{(k)^{1/2}}+n\lambda\mbf{I}_n\right]^{-2}\left[\mbf{\Pi}_{-S}^{(k)^{1/2}}\mbf{K}_{-S}\mbf{\Pi}_{-S}^{(k)^{1/2}}\right]^2\right)\\
    & =\norm{\mbf{\Pi}_{-S}^{(k)^{-1}}}_{op}^2\sum_{j=1}^n\left(\frac{\lambda_j\left(\mbf{\Pi}_{-S}^{(k)^{1/2}}\mbf{K}_{-S}\mbf{\Pi}_{-S}^{(k)^{1/2}}\right)}{n\lambda+\lambda_j\left(\mbf{\Pi}_{-S}^{(k)^{1/2}}\mbf{K}_{-S}\mbf{\Pi}_{-S}^{(k)^{1/2}}\right)}\right)^2\\
    & \overset{(c)}{\leq}\norm{\mbf{\Pi}_{-S}^{(k)^{-1}}}_{op}^2\sum_{j=1}^n\frac{\lambda_j\left(\mbf{\Pi}_{-S}^{(k)^{1/2}}\mbf{K}_{-S}\mbf{\Pi}_{-S}^{(k)^{1/2}}\right)}{4n\lambda}\\
    & =\norm{\mbf{\Pi}_{-S}^{(k)^{-1}}}_{op}^2\frac{\mrm{tr}\left(\mbf{\Pi}_{-S}^{(k)^{1/2}}\mbf{K}_{-S}\mbf{\Pi}_{-S}^{(k)^{1/2}}\right)}{4n\lambda}\\
    & \overset{(d)}{\leq}\norm{\mbf{\Pi}_{-S}^{(k)^{-1}}}_{op}^2\norm{\mbf{\Pi}_{-S}^{(k)}}_{op}\frac{\mrm{tr}(\mbf{K}_{-S})}{4n\lambda}\\
    & \overset{(e)}{\leq}\norm{\mbf{\Pi}_{-S}^{(k)^{-1}}}_{op}^2\frac{\mrm{tr}(\mbf{K}_{-S})}{4n\lambda}
\end{align*}
where (a), (b) and (d) follow from the von Neumann trace inequality \citep{mirsky1975trace}; (c) follows from the AM-GM inequality $n\lambda+\lambda_j\left(\mbf{\Pi}_{-S}^{(k)^{1/2}}\mbf{K}_{-S}\mbf{\Pi}_{-S}^{(k)^{1/2}}\right)\geq2\sqrt{n\lambda\lambda_j\left(\mbf{\Pi}_{-S}^{(k)^{1/2}}\mbf{K}_{-S}\mbf{\Pi}_{-S}^{(k)^{1/2}}\right)}$ and (e) follows since $\mbf{\Pi}_{-S}^{(k)}$ is a diagonal matrix and the largest diagonal element is bounded by 1. 

Combining these two observations together, we have
\begin{align*}
    \mrm{tr}\left(\mbf{\Sigma}_{-S}^{(k)^2}\right) & \leq\norm{\mbf{\Sigma}_{-S}^{(k)}}_{op}\mrm{tr}\left(\mbf{\Sigma}_{-S}^{(k)}\right)\\
    & \leq\norm{\mbf{\Pi}_{-S}^{(k)^{-1}}}_{op}^2\norm{\mbf{\Pi}_{-S}^{(k)^{-1}}}_{op}^2\frac{\mrm{tr}(\mbf{K}_{-S})}{4n\lambda}\\
    & =\norm{\mbf{\Pi}_{-S}^{(k)^{-1}}}_{op}^4\frac{\mrm{tr}(\mbf{K}_{-S})}{4n\lambda}
\end{align*}
Therefore, for any $\delta>0$, with probability at least $1-\delta$,
\begin{align*}
I_3 & \leq\sqrt{\norm{\mbf{\Pi}_{-S}^{(k)^{-1}}}_{op}^2\frac{\mrm{tr}(\mbf{K}_{-S})}{4n\lambda}+2\sqrt{\norm{\mbf{\Pi}_{-S}^{(k)^{-1}}}_{op}^4\frac{\mrm{tr}(\mbf{K}_{-S})}{4n\lambda}\log\frac{1}{\delta}}+2\norm{\mbf{\Pi}_{-S}^{(k)^{-1}}}_{op}^2\log\frac{1}{\delta}}\\
    & \leq\norm{\mbf{\Pi}_{-S}^{(k)^{-1}}}_{op}\sqrt{\left(\sqrt{\frac{\mrm{tr}(\mbf{K}_{-S})}{4n\lambda}}+\sqrt{2\log\frac{1}{\delta}}\right)^2}\\
    & =\norm{\mbf{\Pi}_{-S}^{(k)^{-1}}}_{op}\left(\sqrt{\frac{\mrm{tr}(\mbf{K}_{-S})}{4n\lambda}}+\sqrt{2\log\frac{1}{\delta}}\right).\\
    & \overset{(*)}{=} \mcal{O}(1)\left(\mcal{O}(\lambda^{-\frac{1}{2}}) + \sqrt{2\log\frac{1}{\delta}}\right)\\
    & \overset{(**)}{=} o\left(n^{\frac{\alpha-1}{4(\alpha+1)}}\right)+\mcal{O}(1)\sqrt{2\log\frac{1}{\delta}},
\end{align*}
where (*) holds from the bound for $\norm{\mbf{\Pi}_{-S}^{(k)^{-1}}}_{op}$ and Lemma \ref{Lm: Bound Derivative of Deep ReLU Neural Networks}; (**) holds from the Assumption (A2) that $\lambda^{-1}=o(n^{\frac{\alpha-1}{2(\alpha+1)}})$.
\end{itemize}
Putting all the pieces together, we have with probability at least $1-7\delta$,
$$
\sqrt{n}\norm{\tilde{h}_{\mbf{\theta}_f+\Delta\mbf{\theta}_S}^{(k+1)}-f_{0,-S}}_n\lesssim\norm{\mbf{\Phi}_{-S}\mbf{w}^{(k)}}+\sqrt{W}\log^{1/2}n+\sqrt{\log\frac{1}{\delta}}+o(n^{\frac{\alpha-1}{4(\alpha+1)}}),
$$
which proves (\ref{Eq: Upper Bound for empirical L2 norm}). Equation (\ref{Eq: Upper Bound for empirical L2 norm at convergence}) now follows immediately from Proposition \ref{Prop: Bound on norm of Delta theta S}.
\end{proof}

\subsection{Proof of Lemma \ref{Lm: RKHS norm upper bound}}
\begin{proof}
    Recall that $\mbf{Y} = \sigma(\mbf{f}_{0,-S})+\mbf{\epsilon}$ and denote $\mbf{\upsilon}_{-S}^{(k)}=\mbf{\Phi}_{-S} \mbf{w}^{(k)} + \mbf{\Pi}_{-S}^{(k)^{-1}}\sigma(\mbf{f}_{0,-S})+\mbf{\Pi}_{-S}^{(k)^{-1}}\mbf{\epsilon} - \mbf{\Pi}_{-S}^{(k)^{-1}}\sigma(\tilde{\mbf{h}}_{\mbf{\theta}_f+\Delta\mbf{\theta}_S}^{(k)})$. Then $\mbf{\alpha}_{-S}^{(k)}=\left[\mbf{K}_{-S}+n\lambda\mbf{\Pi}_{-S}^{(k)^{-1}}\right]^{-1}\mbf{\upsilon}_{-S}^{(k)}$ and 
    \begin{align*}
    & \norm{\tilde{h}^{(k+1)}_{\mbf{\theta}_f+\Delta\mbf{\theta}_S}-h_{\mbf{\theta}_f,-S}}_\mathcal{H}=\sqrt{\mbf{\alpha}_{-S}^{(k)^T}\mbf{K}_{-S}\mbf{\alpha}_{-S}^{(k)}}\\ 
    = & \sqrt{\mbf{\upsilon}_{-S}^{(k)^T}\left[\mbf{K}_{-S}+n\lambda\mbf{\Pi}_{-S}^{(k)^{-1}}\right]^{-1}\mbf{K}_{-S}\left[\mbf{K}_{-S}+n\lambda\mbf{\Pi}_{-S}^{(k)^{-1}}\right]^{-1}\mbf{\upsilon}_{-S}^{(k)}}\\
    \overset{(*)}{\leq} & \sqrt{\mbf{\upsilon}_{-S}^{(k)^T}\left[\mbf{K}_{-S}+n\lambda\mbf{\Pi}_{-S}^{(k)^{-1}}\right]^{-1}\mbf{\upsilon}_{-S}^{(k)}}\\
    = & \sqrt{\left(\mbf{\upsilon}_{-S}^{(k)^T}-\mbf{\Pi}_{-S}^{(k)^{-1}}\mbf{\epsilon}+\mbf{\Pi}_{-S}^{(k)^{-1}}\mbf{\epsilon}\right)\left[\mbf{K}_{-S}+n\lambda\mbf{\Pi}_{-S}^{(k)^{-1}}\right]^{-1}\left(\mbf{\upsilon}_{-S}^{(k)}-\mbf{\Pi}_{-S}^{(k)^{-1}}\mbf{\epsilon}+\mbf{\Pi}_{-S}^{(k)^{-1}}\mbf{\epsilon}\right)}\\
    \overset{(**)}{\leq} & \sqrt{\left(\mbf{\upsilon}_{-S}^{(k)}-\mbf{\Pi}_{-S}^{(k)^{-1}}\mbf{\epsilon}\right)^T\left[\mbf{K}_{-S}+n\lambda\mbf{\Pi}_{-S}^{(k)^{-1}}\right]^{-1}\left(\mbf{\upsilon}_{-S}^{(k)}-\mbf{\Pi}_{-S}^{(k)^{-1}}\mbf{\epsilon}\right)}\\
    & \qquad + \sqrt{\left(\mbf{\Pi}_{-S}^{(k)^{-1}}\mbf{\epsilon}\right)^T\left[\mbf{K}_{-S}+n\lambda\mbf{\Pi}_{-S}^{(k)^{-1}}\right]^{-1}\left(\mbf{\Pi}_{-S}^{(k)^{-1}}\mbf{\epsilon}\right)}\\
    =: & J_1 + J_2,
\end{align*}
where (*) follows by noting that 
\begin{align*}
& \left[\mbf{K}_{-S}+n\lambda\mbf{\Pi}_{-S}^{(k)^{-1}}\right]^{-1}-\left[\mbf{K}_{-S}+n\lambda\mbf{\Pi}_{-S}^{(k)^{-1}}\right]^{-1}\mbf{K}_{-S}\left[\mbf{K}_{-S}+n\lambda\mbf{\Pi}_{-S}^{(k)^{-1}}\right]^{-1}\\
= & \left[\mbf{K}_{-S}+n\lambda\mbf{\Pi}_{-S}^{(k)^{-1}}\right]^{-1}\left[\mbf{I}_n-\mbf{K}_{-S}\left[\mbf{K}_{-S}+n\lambda\mbf{\Pi}_{-S}^{(k)^{-1}}\right]^{-1}\right]\\
= & n\lambda\left[\mbf{K}_{-S}+n\lambda\mbf{\Pi}_{-S}^{(k)^{-1}}\right]^{-1}\mbf{\Pi}_{-S}^{(k)^{-1}}\left[\mbf{K}_{-S}+n\lambda\mbf{\Pi}_{-S}^{(k)^{-1}}\right]^{-1},
\end{align*}
and $n\lambda\left[\mbf{K}_{-S}+n\lambda\mbf{\Pi}_{-S}^{(k)^{-1}}\right]^{-1}\mbf{\Pi}_{-S}^{(k)^{-1}}\left[\mbf{K}_{-S}+n\lambda\mbf{\Pi}_{-S}^{(k)^{-1}}\right]^{-1}$ is clearly a positive semidefinite matrix. Additionally, (**) follows from the triangle inequality based on the norm $\norm{\mbf{x}}_{\mbf{A}}:=\sqrt{\mbf{x}^T\mbf{Ax}}$ with $\mbf{A}$ being a positive semidefinite matrix. We now bound $J_1$ and $J_2$.
\begin{itemize}
    \item \textit{Bound for $J_1$}. Note that
        \begin{align*}
            & \sqrt{\left(\mbf{\upsilon}_{-S}^{(k)}-\mbf{\Pi}_{-S}^{(k)^{-1}}\mbf{\epsilon}\right)^T\left[\mbf{K}_{-S}+n\lambda\mbf{\Pi}_{-S}^{(k)^{-1}}\right]^{-1}\left(\mbf{\upsilon}_{-S}^{(k)}-\mbf{\Pi}_{-S}^{(k)^{-1}}\mbf{\epsilon}\right)}\\
            = & \sqrt{\left(\left[\mbf{K}_{-S}+n\lambda\mbf{\Pi}_{-S}^{(k)^{-1}}\right]^{-1/2}\left(\mbf{\upsilon}_{-S}^{(k)}-\mbf{\Pi}_{-S}^{(k)^{-1}}\mbf{\epsilon}\right)\right)^T\left[\mbf{K}_{-S}+n\lambda\mbf{\Pi}_{-S}^{(k)^{-1}}\right]^{-1/2}\left(\mbf{\upsilon}_{-S}^{(k)}-\mbf{\Pi}_{-S}^{(k)^{-1}}\mbf{\epsilon}\right)}\\
            = & \norm{\left[\mbf{K}_{-S}+n\lambda\mbf{\Pi}_{-S}^{(k)^{-1}}\right]^{-1/2}\left(\mbf{\upsilon}_{-S}^{(k)}-\mbf{\Pi}_{-S}^{(k)^{-1}}\mbf{\epsilon}\right)}.
        \end{align*}
    On the other hand, since
    \begin{align*}
        & \left[\mbf{K}_{-S}+n\lambda\mbf{\Pi}_{-S}^{(k)^{-1}}\right]^{-1/2}\left(\mbf{\upsilon}_{-S}^{(k)}-\mbf{\Pi}_{-S}^{(k)^{-1}}\mbf{\epsilon}\right)\\
        = & \left[\mbf{K}_{-S}+n\lambda\mbf{\Pi}_{-S}^{(k)^{-1}}\right]^{-1/2}\left[\mbf{\Phi}_{-S} \mbf{w}^{(k)} +\mbf{\Pi}_{-S}^{(k)^{-1}}(\sigma(\mbf{f}_{0,-S})-\sigma(\tilde{\mbf{h}}_{\mbf{\theta}_f+\Delta\mbf{\theta}_S}^{(k)})\right]\\
        = & \left[\mbf{K}_{-S}+n\lambda\mbf{\Pi}_{-S}^{(k)^{-1}}\right]^{-1/2}\mbf{\Phi}_{-S}\mbf{w}^{(k)}+\left[\mbf{K}_{-S}+n\lambda\mbf{\Pi}_{-S}^{(k)^{-1}}\right]^{-1/2}\mbf{\Pi}_{-S}^{(k)^{-1}}\left[\sigma(\mbf{f}_{0,-S})-\sigma(\tilde{\mbf{h}}_{\mbf{\theta}_f+\Delta\mbf{\theta}_S}^{(k)})\right],
    \end{align*}
    it then follows that
    \begin{align*}
        & \norm{\left[\mbf{K}_{-S}+n\lambda\mbf{\Pi}_{-S}^{(k)^{-1}}\right]^{-1/2}\left(\mbf{\upsilon}_{-S}^{(k)}-\mbf{\Pi}_{-S}^{(k)^{-1}}\mbf{\epsilon}\right)}\\
        \leq & \norm{\left[\mbf{K}_{-S}+n\lambda\mbf{\Pi}_{-S}^{(k)^{-1}}\right]^{-1/2}\mbf{\Phi}_{-S}\mbf{w}^{(k)}}+\norm{\left[\mbf{K}_{-S}+n\lambda\mbf{\Pi}_{-S}^{(k)^{-1}}\right]^{-1/2}\mbf{\Pi}_{-S}^{(k)^{-1}}\left[\sigma(\mbf{f}_{0,-S})-\sigma(\tilde{\mbf{h}}_{\mbf{\theta}_f+\Delta\mbf{\theta}_S}^{(k)})\right]}\\
        \overset{(*)}{\leq} & \norm{\left[\mbf{K}_{-S}+n\lambda\mbf{\Pi}_{-S}^{(k)^{-1}}\right]^{-1/2}}_{op}\left[\norm{\mbf{\Phi}_{-S}\mbf{w}^{(k)}}+\frac{1}{4}\norm{\mbf{\Pi}_{-S}^{(k)^{-1}}}_{op}\norm{\mbf{f}_{0,-S}-\mbf{h}_{\mbf{\theta}_f}-\mbf{\Phi}_{-S}\mbf{w}^{(k)}}\right]\\
        \leq & \frac{1}{4}\norm{\left[\mbf{K}_{-S}+n\lambda\mbf{\Pi}_{-S}^{(k)^{-1}}\right]^{-1/2}}_{op}\left[\left(1+\frac{1}{4}\norm{\mbf{\Pi}_{-S}^{(k)}}_{op}\right)\norm{\mbf{\Phi}_{-S}\mbf{w}^{(k)}}+\frac{1}{4}\norm{\mbf{\Pi}_{-S}^{(k)}}\norm{\mbf{e}_{-S}}\right],
    \end{align*}
    where (*) follows from the Lipshitz continuity of the sigmoid function as given in Lemma \ref{lem:sig_lip}. Moreover, note that
    \begin{align*}
     &   {\norm{\left[\mbf{K}_{-S}+n\lambda\mbf{\Pi}_{-S}^{(k)^{-1}}\right]^{-\frac{1}{2}}}}_{op} \\
     = & {\norm{\Big(\mbf{\Pi}_{-S}^{(k)^{\frac{1}{2}}}\left[\mbf{\Pi}_{-S}^{(k)^{\frac{1}{2}}}\mbf{K}_{-S}\mbf{\Pi}_{-S}^{(k)^{\frac{1}{2}}}+n\lambda\mbf{I}_n\right]^{-1}\mbf{\Pi}_{-S}^{(k)^{\frac{1}{2}}}\Big)^{\frac{1}{2}}}}_{op}\\
     = & \sqrt{\norm{\Big(\mbf{\Pi}_{-S}^{(k)^{\frac{1}{2}}}\left[\mbf{\Pi}_{-S}^{(k)^{\frac{1}{2}}}\mbf{K}_{-S}\mbf{\Pi}_{-S}^{(k)^{\frac{1}{2}}}+n\lambda\mbf{I}_n\right]^{-1}\mbf{\Pi}_{-S}^{(k)^{\frac{1}{2}}}\Big)}}_{op}\\
     \leq & \sqrt{{\norm{\mbf{\Pi}_{-S}^{(k)}}_{op}}\norm{\left[\mbf{\Pi}_{-S}^{(k)^{\frac{1}{2}}}\mbf{K}_{-S}\mbf{\Pi}_{-S}^{(k)^{\frac{1}{2}}}+n\lambda\mbf{I}_n\right]^{-1}}}_{op}\\
     \leq & \sqrt{\frac{1}{4}\norm{\left[\mbf{\Pi}_{-S}^{(k)^{\frac{1}{2}}}\mbf{K}_{-S}\mbf{\Pi}_{-S}^{(k)^{\frac{1}{2}}}+n\lambda\mbf{I}_n\right]^{-1}}}_{op}\\
     \leq & \sqrt{\frac{1}{4n\lambda}}.
    \end{align*}
    Also recall that $\norm{\mbf{\Pi}_{-S}^{(k)^{-1}}}_{op}=\mcal{O}(1)$ from (\ref{Eq: Bound for Pi inv}), combined with the assumption (A2) and Corollary \ref{Cor: Bound for e}, we have with probability at least $1-6\delta$,
    \begin{align*}
        J_1 & \lesssim\mcal{O}\left(\frac{1}{\sqrt{n\lambda}}\right)\left(\norm{\mbf{e}_{-S}}+\norm{\mbf{\mbf{\Phi}_{-S}\mbf{w}}^{(k)}}\right)\\
        & =o\left(n^{-\frac{\alpha+3}{4(\alpha+1)}}\right)\left(\sqrt{W}\log^{1/2}n+\sqrt{\log\frac{1}{\delta}}+\norm{\mbf{\Phi}_{-S}\mbf{w}^{(k)}}\right).
    \end{align*}

    \item \textit{Bound for $J_2$}. Note that
        \begin{align*}
            J_2 & = \sqrt{\left(\mbf{\Pi}_{-S}^{(k)^{-1}}\mbf{\epsilon}\right)^T\left[\mbf{K}_{-S}+n\lambda\mbf{\Pi}_{-S}^{(k)^{-1}}\right]^{-1}\left(\mbf{\Pi}_{-S}^{(k)^{-1}}\mbf{\epsilon}\right)}\\
            & = \sqrt{\left(\mbf{\Pi}_{-S}^{(k)^{-1}}\mbf{\epsilon}\right)^T\left[\mbf{K}_{-S}+n\lambda\mbf{\Pi}_{-S}^{(k)^{-1}}\right]^{-1/2}\left[\mbf{K}_{-S}+n\lambda\mbf{\Pi}_{-S}^{(k)^{-1}}\right]^{-1/2}\left(\mbf{\Pi}_{-S}^{(k)^{-1}}\mbf{\epsilon}\right)}\\
            & =\norm{\left[\mbf{K}_{-S}+n\lambda\mbf{\Pi}_{-S}^{(k)^{-1}}\right]^{-1/2}\mbf{\Pi}_{-S}^{(k)^{-1}}\mbf{\epsilon}}\\
            & \leq \norm{\left[\mbf{K}_{-S}+n\lambda\mbf{\Pi}_{-S}^{(k)^{-1}}\right]^{-1/2}}_{op}\norm{\mbf{\Pi}_{-S}^{(k)^{-1}}}_{op}\norm{\mbf{\epsilon}}\\
        \end{align*}
        Now, using the proof method from \citet{gao2022lazy}, we can use the concentration inequality again from \citet{hsu2012tail} and the fact that the sub-Gaussian parameter of $\mbf{\epsilon}$ is $1$ to get with probability at least $1-\delta$,
        \begin{align*}
            \norm{\mbf{\epsilon}} & = \sqrt{\mbf{\epsilon}^T\mbf{\epsilon}}\\
            &\leq \sigma\sqrt{n+2\sqrt{n\log{\frac{1}{\delta}}}+2\log{\frac{1}{\delta}}}\\
            & \leq \sigma\left(\sqrt{n} + \sqrt{2\log{\frac{1}{\delta}}}\right)\\
            & \leq \sqrt{n} + \sqrt{2\log{\frac{1}{\delta}}}
        \end{align*}
        As a result, with probability at least $1-\delta$,
        \begin{align*}
            J_2
            & \leq {\norm{\mbf{\Pi}_{-S}^{(k)^{-1}}}}_{op} \left(\sqrt{n} + \sqrt{2\log{\frac{1}{\delta}}}\right) \sqrt{\frac{1}{4n\lambda}}\\
            & \lesssim \mcal{O}\left(\frac{1}{\sqrt{n\lambda}}\right)\left(\sqrt{n} + \sqrt{\log{\frac{1}{\delta}}}\right)\\
            & =o\left(n^{-\frac{\alpha+3}{4(\alpha+1)}}\right)\left(\sqrt{n}+\sqrt{\log\frac{1}{\delta}}\right)
        \end{align*}
\end{itemize}
Putting all the pieces together with the assumption (A2), we have with probability at least $1-7\delta$,
\begin{align*}
    \norm{\tilde{h}^{(k+1)}_{\mbf{\theta}_f+\Delta\mbf{\theta}_S}-h_{\mbf{\theta}_f,-S}}_\mathcal{H} & \leq J_1 + J_2\\
    & \lesssim o\left(n^{-\frac{\alpha+3}{4(\alpha+1)}}\right)\left(\sqrt{W}\log^{1/2}n+\sqrt{\log\frac{1}{\delta}}+\norm{\mbf{\Phi}_{-S}\mbf{w}^{(k)}}\right) \\
    & \hspace{4cm}+ o\left(n^{-\frac{\alpha+3}{4(\alpha+1)}}\right)\left(\sqrt{n}+\sqrt{\log\frac{1}{\delta}}\right)\\
    & \lesssim o\left(n^{-\frac{\alpha+3}{4(\alpha+1)}}\right)\left(\sqrt{n}+\sqrt{W}\log^{1/2}n+\norm{\mbf{\Phi}_{-S}\mbf{w}^{(k)}}+\sqrt{\log\frac{1}{\delta}}\right),
\end{align*}
which proves (\ref{Eq: Upper bound for the RKHS norm}) and (\ref{Eq: Upper bound for the RKHS norm at convergence}) follows by applying the upper bound for $\norm{\mbf{\Phi}_{-S}\Delta\mbf{\theta_S}}$ given in Proposition \ref{Prop: Bound on norm of Delta theta S}.
\end{proof}

\subsection{Proof of Lemma \ref{Lm: Upper Bound for L2 Norm of f in HB}}
\begin{proof}
    We prove this result by applying Theorem \ref{Thm: Empirical Norm and L2 Norm}. Using the notation in Theorem \ref{Thm: Empirical Norm and L2 Norm}, we let $\mcal{F}=\mcal{H}_B$ and $f^*=f_{0,-S}-h_{\mbf{\theta}_f,-S}$. Then for any $f\in\mrm{star}(\mcal{F}, f^*)$ satisfying $\norm{f-f^*}_n^2\leq 2\tilde{M}^2r$, we know that there exists $\gamma\in[0,1]$ and $h\in\mcal{H}_B$ such that
    $$
    f=f^*+\gamma[h-f^*],
    $$
    which implies that
    \begin{align*}
        \norm{\gamma h}_n^2 & =\norm{f-f^*+\gamma f^*}_n^2\\
            & \leq 2\norm{f-f^*}_n^2+2\gamma^2\norm{f^*}_n^2\\
            & \leq 4\tilde{M}^2r+2\tilde{M}^2\rho\\
            & \leq4\tilde{M}^2(r+\rho).
    \end{align*}
    On the other hand, note that $\mcal{H}_B$ is star-shaped around 0, we know that $\gamma h\in\mcal{H}_B$. Therefore,
    $$
    \left\{f\in\mrm{star}(\mcal{F}, f^*):\norm{f-f^*}_n^2\leq2\tilde{M}^2r\right\}\subseteq\left\{h\in\mcal{H}_B:\norm{h}_n^2\leq4\tilde{M}^2(r+\rho)\right\},
    $$
    which implies that
    \begin{align*}
        & \mbb{E}_\xi\left[\left.\sup_{f\in\mrm{star}(\mcal{F}, f^*), \norm{f-f^*}_n^2\leq 2\tilde{M}^2r}\frac{1}{n}\sum_{i=1}^n\xi_i\frac{f(\mbf{X}_{i,-S})}{\tilde{M}}\right|\mbf{X}_1,\ldots,\mbf{X}_n\right]\\
        \leq & \mbb{E}_\xi\left[\left.\sup_{h\in\mcal{H}_B,\norm{h}_n^2\leq4\tilde{M}^2(r+\rho)}\frac{1}{n}\sum_{i=1}^n\xi_i\frac{h(\mbf{X}_{i,-S})}{\tilde{M}}\right|\mbf{X}_1,\ldots,\mbf{X}_n\right]\\
        = & \mbb{E}_\xi\left[\left.\sup_{h\in\bar{\mcal{H}}_B,\norm{h}_n^2\leq4(r+\rho)}\frac{1}{n}\sum_{i=1}^n\xi_i h(\mbf{X}_{i,-S})\right|\mbf{X}_1,\ldots,\mbf{X}_n\right].
    \end{align*}
    Hence the desired result follows from Theorem \ref{Thm: Empirical Norm and L2 Norm}.
\end{proof}

\subsection{Proof of Lemma \ref{Lm: Fixed point of sub-root function for HB}}
\begin{proof}
    For notation simplicity, denote $\tau^2=4(r+\rho)$. Let $N_1$ be such that $\tau^2=cN_1^{-\alpha}$ and $N=\max\{N_0,N_1\}$. Then $\mu_j \leq \tau^2$ for all $j \geq N$. Then 
    \begin{align*} 
        \sum_{j=1}^n \min\{\tau^2, \mu_j\} 
        &=
        \sum_{j=1}^{\lfloor N \rfloor} \min\{\tau^2,\mu_j\} + \sum_{j=\lfloor N \rfloor + 1}^n \mu_j \\
        &\leq \sum_{j=1}^{\lfloor N \rfloor} \tau^2 + \sum_{j=\lfloor N \rfloor + 2}^n \mu_j + \mu_{\lfloor N \rfloor + 1} \\
        &\leq \lfloor N \rfloor \tau^2 + \tau^2 + \int_N^{\infty} c j^{-\alpha} \, dj \\
        &= \lfloor N \rfloor \tau^2 + \tau^2 + \frac{c}{\alpha - 1} N^{1 - \alpha}.
    \end{align*}
    Then since $N^{-\alpha}\leq\frac{\tau^2}{c},$
    \begin{align*}
        \lfloor N \rfloor \tau^2 + \tau^2 + \frac{c}{\alpha - 1} N \cdot N^{-\alpha} 
        &\leq \lfloor N \rfloor \tau^2 + \tau^2 + \frac{c}{\alpha - 1} N \cdot \frac{\tau^2}{c} \\
        &\leq \tau^2 \left( N + 1 + \frac{N}{\alpha - 1} \right) \\
        &\leq \tau^2 \left( 1 + \frac{1}{N} + \frac{1}{\alpha - 1} \right) N \\
        &\leq \left( 2 + \frac{1}{\alpha - 1} \right) \tau^2 N \\
        &= \left( 2 + \frac{1}{\alpha - 1} \right) \tau^2 \left( \frac{c}{\tau^2} \right)^{1/\alpha} \\
        &= \left( 2 + \frac{1}{\alpha - 1} \right) c^{1/\alpha} (\tau^2)^{1 - 1/\alpha}.
    \end{align*}
    Let $c_\alpha=\left( 2 + \frac{1}{\alpha - 1} \right)c^{\frac{1}{\alpha}}$ and we have shown that
    $$
    \hat{\mcal{R}}_n(\tau^2,\bar{\mcal{H}}_B)\leq\sqrt{\frac{2}{n}}\sqrt{c_\alpha(\tau^2)^{1-\frac{1}{\alpha}}}.
    $$
    Now take \[
    \hat{\psi}_n(r) :=20\sqrt{\frac{2}{n}}\sqrt{c_\alpha(8(r\vee\rho))^{1-\frac{1}{\alpha}}}+\frac{62}{3n}\log\frac{1}{\delta}\geq 20\sqrt{\frac{2}{n}}\sqrt{c_\alpha(\tau^2)^{1-\frac{1}{\alpha}}}+\frac{62}{3n}\log\frac{1}{\delta}
    \]
    Now let us claim that $\hat{\psi}_n(r)$ is a sub-root function. It is obvious that $\hat{\psi}_n(r)$ is nonnegative and nondecreasing since $\alpha>1$. So it suffices to show that $\hat{\psi}_n(r)/\sqrt{r}$ is nonincreasing. But note that
    \begin{align*}
    \frac{\hat{\psi}_n(r)}{\sqrt{r}} & =20\sqrt{\frac{2}{n}}\sqrt{\frac{c_\alpha}{r}(8(r\vee\rho))^{1-\frac{1}{\alpha}}}+\frac{62}{3n}\log\frac{1}{\delta}\\
    & =20\sqrt{\frac{2}{n}}\sqrt{c_\alpha 8^{1-\frac{1}{\alpha}}\left(r^{-1/\alpha}\vee r^{-1}\rho^{1-\frac{1}{\alpha}}\right)}+\frac{62}{3n}\log\frac{1}{\delta},
    \end{align*}
    and since both $r^{-1/\alpha}$ and $r^{-1}\rho^{1-\frac{1}{\alpha}}$ are decreasing functions in $r$, it follows that $\hat{\psi}_n(r)/\sqrt{r}$ is nonincreasing for $r>0$. On the other hand, let $\tilde{\psi}_n(r)=20\sqrt{\frac{2}{n}}\sqrt{c_\alpha(8(r\vee\rho))^{1-\frac{1}{\alpha}}}$. In view of Lemma \ref{Lm: adding a constant to sub-root function}, we have $\hat{r}^*\leq\tilde{r}^*+\frac{124}{3n}\log\frac{1}{\delta}$, where $\hat{r}^*$ and $\tilde{r}^*$ represent the fixed points of $\hat{\psi}_n$ and $\tilde{\psi}_n$ respectively. Therefore, it suffices to find $\tilde{r}^*$. Let $\tilde{c}_\alpha=20\sqrt{2c_\alpha\cdot 8^{1-\frac{1}{\alpha}}}$ and note that
    $$
    \tilde{\psi}_n(r)=\left\{\begin{array}{cc}
        \tilde{c}_\alpha n^{-\frac{1}{2}}\rho^{\frac{1}{2}\left(1-\frac{1}{\alpha}\right)} &  \mrm{ if }r\leq \rho\\
        \tilde{c}_\alpha n^{-\frac{1}{2}}r^{\frac{1}{2}\left(1-\frac{1}{\alpha}\right)} &  \mrm{ if }r\geq \rho
    \end{array}
    \right.,
    $$
    and solving $\tilde{\psi}_n(r)=r$ gives
    $$
    \tilde{r}^*=\left\{\begin{array}{cc}
        \tilde{c}_\alpha n^{-\frac{1}{2}}\rho^{\frac{1}{2}\left(1-\frac{1}{\alpha}\right)}  & \mrm{ if }\tilde{c}_\alpha^{\frac{\alpha}{\alpha+1}}n^{-\frac{\alpha}{\alpha+1}}\leq\rho \\
        \tilde{c}_\alpha^{\frac{\alpha}{\alpha+1}}n^{-\frac{\alpha}{\alpha+1}} & \mrm{ if }\tilde{c}_\alpha^{\frac{2\alpha}{\alpha+1}}n^{-\frac{\alpha}{\alpha+1}}\geq\rho
    \end{array}
    \right.,
    $$
    which implies that $\tilde{r}^*\leq\max\left\{\tilde{c}_\alpha n^{-1/2}\rho^{\frac{1}{2}\left(1-\frac{1}{\alpha}\right)}, \tilde{c}_\alpha^{\frac{2\alpha}{\alpha+1}}n^{-\frac{\alpha}{\alpha+1}}\right\}$. Putting all the pieces together, 
    $$
    \hat{r}^*\leq\tilde{r}^*+\frac{124}{3n}\log\frac{1}{\delta}\leq\max\left\{\tilde{c}_\alpha n^{-1/2}\rho^{\frac{1}{2}\left(1-\frac{1}{\alpha}\right)}, \tilde{c}_\alpha^{\frac{2\alpha}{\alpha+1}}n^{-\frac{\alpha}{\alpha+1}}\right\}+\frac{124}{3n}\log\frac{1}{\delta}
    $$
\end{proof}

\subsection{Proof of Theorem \ref{Thm: Estimation error of h tilde}}
\begin{proof}
    We first prove (\ref{Eq: L2 norm of hk}). Recall that $\tilde{M}=\sup_{\mbf{x}}\abs{f_{0,-S}(\mbf{x})-h_{\mbf{\theta}_f,-S}(\mbf{x})}+\sup_{f\in\mcal{H}_B}\sup_{\mbf{x}}\abs{f(\mbf{x})}$. By Lemma \ref{Lm: Uniform Boundedness of HB},
    $$
    \tilde{M}\leq\sup_{\mbf{x}}\abs{f_{0,-S}(\mbf{x})-h_{\mbf{\theta}_f,-S}(\mbf{x})}+B\sqrt{\sup_{\mbf{x}}K_{-S}(\mbf{x},\mbf{x})}.
    $$
    Based on Lemma \ref{Lm: RKHS norm upper bound}, with probability at least $1-7e^{-n^{1/2}}$,
    \begin{align*}
    B\lesssim o\left(n^{-\frac{\alpha+3}{4(\alpha+1)}}\right)\left(\sqrt{n}+\sqrt{W}\log^{1/2}n+\norm{\mbf{\Phi}_{-S}\mbf{w}^{(k)}}\right).
    \end{align*}
    On the other hand, based on Lemma \ref{Lm: Bound Derivative of Deep ReLU Neural Networks},
    \begin{align*}
        \norm{\mbf{\Phi}_{-S} \mbf{w}^{(k)}} & \leq \norm{\mbf{\Phi}_{-S}}_{op}\norm{\mbf{w}^{(k)}} \\
        & = \sqrt{\lambda_{\max}(\mbf{K}_{-S})}\norm{\mbf{w}^{(k)}}\\
        & \leq \sqrt{\mrm{tr}(\mbf{K}_{-S})}\norm{\mbf{w}^{(k)}} \\
        & \lesssim\sqrt{n}\mathcal{O}(1) \\
        & = \mathcal{O}(n^\frac{1}{2}).
    \end{align*}
    Therefore, with probability at least $1-7e^{-n^{1/2}}$,  
    \begin{equation}\label{Eq: Bound on M}
        \tilde{M}\lesssim 1+B\lesssim 1+o\left(n^{-\frac{\alpha+3}{4(\alpha+1)}}\right)\left(\sqrt{n}+\sqrt{W}\log^{1/2}n\right)=o(n^{\frac{\alpha-1}{4(\alpha+1)}})+o\left(n^{-\frac{\alpha+3}{4(\alpha+1)}}\right)\sqrt{W}\log^{1/2}n.
    \end{equation}
    Then by combining Lemma \ref{Lm: Empirical norm Upper bounds}, Lemma \ref{Lm: Upper Bound for L2 Norm of f in HB} and Lemma \ref{Lm: Fixed point of sub-root function for HB}, we have with probability at least $1-10e^{-n^{\frac{1}{\alpha+1}}}$,
    \begin{align*}
        \norm{h_{\mbf{\theta}_f+\Delta\mbf{\theta}_S}^{(k+1)}-f_{0,-S}} & \lesssim\frac{1}{\sqrt{n}}\left[\norm{\mbf{\Phi}_{-S}\mbf{w}^{(k)}}+\sqrt{W}\log^{1/2}n+o(n^{\frac{\alpha-1}{4(\alpha+1)}})\right]+\tilde{M}\sqrt{\hat{r}^*}+\tilde{M}n^{-\frac{\alpha}{2(\alpha+1)}}\\
        & \lesssim_\alpha\frac{1}{\sqrt{n}}\norm{\mbf{\Phi}_{-S}\mbf{w}^{(k)}}+\frac{\sqrt{W}}{\sqrt{n}}\log^{1/2}n+o(n^{-\frac{\alpha+3}{4(\alpha+1)}})+\\
        & \qquad\qquad\quad \tilde{M}^{\frac{\alpha+1}{2\alpha}}n^{-1/4}\norm{f_{0,-S}-h_{\mbf{\theta}_f,-S}}^{\frac{\alpha-1}{2\alpha}}+\tilde{M}n^{-\frac{\alpha}{2(\alpha+1)}}\\
        & \overset{(*)}{\lesssim}_\alpha\frac{1}{\sqrt{n}}\norm{\mbf{\Phi}_{-S}\mbf{w}^{(k)}}+\left[n^{-1/2}+o(n^{-3/4})\right]\sqrt{W}\log^{1/2}n+o(n^{-1/4})+\\
        & \qquad\qquad\quad \left[o\left(n^{-\frac{1}{4}\left(1-\frac{\alpha-1}{2\alpha}\right)}\right)+o\left(n^{-\frac{1}{4}\left(1+\frac{\alpha+3}{2\alpha}\right)}\right)\left[\sqrt{W}\log^{1/2}n\right]^{\frac{\alpha+1}{2\alpha}}\right]\norm{f_{0,-S}-h_{\mbf{\theta}_f,-S}}^{\frac{\alpha-1}{2\alpha}}\\
        & \lesssim_\alpha\frac{1}{\sqrt{n}}\norm{\mbf{\Phi}_{-S}\mbf{w}^{(k)}}+n^{-1/2}\sqrt{W}\log^{1/2}n+o(n^{-1/4})+\\
        & \qquad\qquad\quad \left[o\left(n^{-\frac{1}{4}\left(1-\frac{\alpha-1}{2\alpha}\right)}\right)+o\left(n^{-\frac{1}{4}\left(1+\frac{\alpha+3}{2\alpha}\right)}\right)\left[\sqrt{W}\log^{1/2}n\right]^{\frac{\alpha+1}{2\alpha}}\right]\norm{f_{0,-S}-h_{\mbf{\theta}_f,-S}}^{\frac{\alpha-1}{2\alpha}}
    \end{align*}
    where the inequality (*) follows since 
    \begin{align*}
        o\left(n^{-\frac{\alpha+3}{4(\alpha+1)}}\right) & =o\left(n^{-\frac{1}{4}\left(1+\frac{2}{\alpha+1}\right)}\right)=o\left(n^{-1/4}\right)\\
        \tilde{M}^{\frac{\alpha+1}{2\alpha}}n^{-1/4} & \lesssim \left[o\left(n^{\frac{\alpha-1}{8\alpha}}\right)+o\left(n^{-\frac{\alpha+3}{8\alpha}}\right)\left(\sqrt{W}\log^{1/2}n\right)^{\frac{\alpha+1}{2\alpha}}\right]n^{-1/4}\\
        & =o\left(n^{-\frac{1}{4}\left(1-\frac{\alpha-1}{2\alpha}\right)}\right)+o\left(n^{-\frac{1}{4}\left(1+\frac{\alpha+3}{2\alpha}\right)}\right)\left(\sqrt{W}\log^{1/2}n\right)^{\frac{\alpha+1}{2\alpha}}\numberthis\label{Eq: Mn bound 1}\\
        \tilde{M}n^{-\frac{\alpha}{2(\alpha+1)}} & \lesssim\left[o\left(n^{\frac{\alpha-1}{4(\alpha+1)}}\right)+o\left(n^{-\frac{\alpha+3}{4(\alpha+1)}}\right)\sqrt{W}\log^{1/2}n\right]n^{-\frac{\alpha}{2(\alpha+1)}}\\
        & =o\left(n^{-1/4}\right)+o(n^{-3/4})\sqrt{W}\log^{1/2}n\numberthis\label{Eq: Mn bound 2}.
    \end{align*}

    We now prove (\ref{Eq: L2 norm of h*}). Based on Lemma \ref{Lm: RKHS norm upper bound}, we know that for $n$ sufficiently large, with probability at least $1-15e^{-n^{\frac{1}{\alpha+1}}}$,
    \begin{align*}
    \tilde{M} \lesssim 1+B & \lesssim1+o\left(n^{-\frac{\alpha+3}{4(\alpha+1)}}\right)\left[\sqrt{n}+\sqrt{W}\log^{1/2}n+n^{\frac{1}{2(\alpha+1)}}\right]\\
    & \lesssim o\left(n^{\frac{\alpha-1}{4(\alpha+1)}}\right)+o\left(n^{-\frac{\alpha+3}{4(\alpha+1)}}\right)\sqrt{W}\log^{1/2}n,
    \end{align*}
    where the last inequality follows since $n^{\frac{1}{2(\alpha+1)}}=o(n^{1/2})$. Combined with Proposition \ref{Prop: Bound on norm of Delta theta S} and Lemma \ref{Lm: Empirical norm Upper bounds}, we know that with probability at least $1-32e^{-n^{\frac{1}{\alpha+1}}}$,
    \begin{align*}
        \norm{h_{\mbf{\theta}_f+\Delta\mbf{\theta}_S}^*-f_{0,-S}} & \lesssim_\alpha \frac{1}{\sqrt{n}}\left[\sqrt{W}\log^{1/2}n+n^{\frac{1}{2(\alpha+1)}}\right]+o(n^{-1/4})+\\
        & \qquad\qquad \tilde{M}^{\frac{\alpha+1}{2\alpha}}n^{-1/4}\norm{f_{0,-S}-h_{\mbf{\theta}_f,-S}}^{\frac{1}{2}\left(1-\frac{1}{\alpha}\right)}+\tilde{M}n^{-\frac{\alpha}{2(\alpha+1)}}+\tilde{M}n^{-\frac{\alpha}{2(\alpha+1)}}\\
        & \lesssim_\alpha n^{-1/2}\sqrt{W}\log^{1/2}n+o(n^{-1/4})+\\
        & \qquad\qquad\quad \left[o\left(n^{-\frac{1}{4}\left(1-\frac{\alpha-1}{2\alpha}\right)}\right)+o\left(n^{-\frac{1}{4}\left(1+\frac{\alpha+3}{2\alpha}\right)}\right)\left[\sqrt{W}\log^{1/2}n\right]^{\frac{\alpha+1}{2\alpha}}\right]\norm{f_{0,-S}-h_{\mbf{\theta}_f,-S}}^{\frac{\alpha-1}{2\alpha}}
    \end{align*}
    where the last inequality follows from (\ref{Eq: Mn bound 1}), (\ref{Eq: Mn bound 2}) and
    \begin{align*}
        n^{-1/2}\cdot n^{\frac{1}{2(\alpha+1)}} & =n^{-\frac{1}{4}\left(1+\frac{\alpha-1}{\alpha+1}\right)}=o(n^{-1/4})
    \end{align*}
\end{proof}

\section{Proof of the Main Results}
\subsection{Proof of Theorem \ref{Thm: Empirical Norm and L2 Norm}}
\begin{proof}
Define
    \begin{equation}\label{Eq: Def Function Class F0}
        \bar{\mcal{F}}^0=\left\{\frac{f-f^*}{\tilde{M}}:f\in\mcal{\mcal{F}}\right\},\quad \bar{\mcal{F}}=\left\{\frac{f}{\tilde{M}}:f\in\mcal{F}\right\}.
    \end{equation}
Fix $\delta >0$. Let $\mcal{G}^0=\{\bar{f}^2:f\in\bar{\mcal{F}}^0\}$ and $\mcal{G}_r^0=\{g\in\mcal{G}^0:\mbb{E}[g(\mbf{X}_i)]\leq r\}$. Additionally, let us define a functional $T:\mcal{G}^0\to\mbb{R}^+$ as
    $$
    T(g)=\mbb{E}[g(\mbf{X})],\quad\forall g\in\mcal{G}^0.
    $$
    Since functions in $\bar{\mcal{F}}^0$ are uniformly bounded by 1, we have for any $g=\bar{f}^2\in\mcal{G}_r^0$,
    \begin{equation}\label{Eq: Property of T}
    \mrm{Var}[g(\mbf{X})]=\mrm{Var}[\bar{f}^2(\mbf{X})]\leq\mbb{E}\left[\bar{f}^4(\mbf{X})\right]\leq\mbb{E}[\bar{f}^2(\mbf{X})]=\mbb{E}[g(\mbf{X})]\leq r.
    \end{equation}
    According to (\ref{Eq: Property of T}), the property that the functional $T$ needs to satisfy in order to apply Theorem \ref{Thm: Thm 3.3 in Bartlett et al 2005} is satisfied with $D=1$. Therefore, with probability at least $1-\delta$, for any $g\in\mcal{G}^0$,
    \begin{align*}
        \mbb{E}[g(\mbf{X})] & \leq \frac{2}{n}\sum_{i=1}^ng(\mbf{X}_i)+2c_1r^*+\frac{11+2c_2}{n}\log\frac{1}{\delta},\numberthis\label{Eq: Bound L2 norm by empirical norm}\\
        \frac{1}{n}\sum_{i=1}^ng(\mbf{X}_i) & \leq\frac{3}{2}\mbb{E}[g(\mbf{X})]+2c_1r^*+\frac{11+2c_2}{n}\log\frac{1}{\delta},\numberthis\label{Eq: Bound empirical norm by L2 norm},
    \end{align*}
    where $r^*$ is the fixed point for some sub-root function $\psi(r)\geq\mbb{E}\left[\sup_{g\in\mcal{G}_r^0}\frac{1}{n}\sum_{i=1}^n\xi_ig(\mbf{X}_i)\right]$. In addition, note that
    \begin{align*}
        \mbb{E}\left[\sup_{g\in\mcal{G}_r^0}\frac{1}{n}\sum_{i=1}^n\xi_ig(\mbf{X}_{i})\right] & =\mbb{E}\left[\sup_{\bar{f}\in\bar{\mcal{F}}^0, \norm{\bar{f}}^2\leq r}\frac{1}{n}\sum_{i=1}^n\xi_i\bar{f}^2(\mbf{X}_{i})\right]\\
        & \leq 2\mbb{E}\left[\sup_{\bar{f}\in\bar{\mcal{F}}^0, \norm{\bar{f}}^2\leq r}\frac{1}{n}\sum_{i=1}^n\xi_i\bar{f}(\mbf{X}_{i})\right],
    \end{align*}
    where the last inequality follows from the Ledoux-Talagrand contraction principle \citep{ledoux2013probability}. Therefore, it suffices to find a sub-root function that upper bounds $2\mbb{E}\left[\sup_{\bar{f}\in\bar{\mcal{F}}^0, \norm{\bar{f}}^2\leq r}\frac{1}{n}\sum_{i=1}^n\xi_i\bar{f}(\mbf{X}_{i})\right]$. 

    To do this, note that for $\bar{f}\in\bar{\mcal{F}}^0$, there exists $f\in\mcal{F}$ such that
    $$
    \bar{f}=\frac{f}{\tilde{M}}-\frac{f^*}{\tilde{M}}.
    $$
    On the other hand, it follows from (\ref{Eq: Property of T}) and Theorem 2.1 in \citet{bartlett2005local} (with the choice of $\alpha=1/4$) that with probability at least $1-\delta$,
    \begin{align*}
        \frac{1}{n}\sum_{i=1}^n\bar{f}^2(\mbf{X}_i) & \leq\frac{5}{2}\mbb{E}\left[\sup_{\bar{f}\in\bar{\mcal{F}}^0,\norm{\bar{f}}^2\leq r}\frac{1}{n}\sum_{i=1}^n\xi_i\bar{f}^2(\mbf{X}_i)\right]+\sqrt{\frac{2r}{n}\log\frac{1}{\delta}}+\frac{13}{3n}\log\frac{1}{\delta}+r\\
        & \overset{(*)}{\leq}\frac{5}{2}\mbb{E}\left[\sup_{\bar{f}\in\bar{\mcal{F}}^0,\norm{\bar{f}}^2\leq r}\frac{1}{n}\sum_{i=1}^n\xi_i\bar{f}^2(\mbf{X}_i)\right]+\frac{16}{3n}\log\frac{1}{\delta}+\frac{3r}{2}\\
        & \overset{(**)}{\leq} 5\mbb{E}\left[\sup_{\bar{f}\in\bar{\mcal{F}}^0,\norm{\bar{f}}^2\leq r}\frac{1}{n}\sum_{i=1}^n\xi_i\bar{f}(\mbf{X}_i)\right]+\frac{16}{3n}\log\frac{1}{\delta}+\frac{3r}{2}\\
        & \leq 5\mbb{E}\left[\sup_{f\in\mrm{star}(\mcal{F},f^*), \norm{f-f^*}^2\leq\tilde{M}^2r}\frac{1}{n}\sum_{i=1}^n\xi_i\frac{f(\mbf{X}_i)}{\tilde{M}}\right]+\frac{16}{3n}\log\frac{1}{\delta}+\frac{3r}{2},
    \end{align*}
    where (*) follows from the AM-GM inequality:
    $$
    \sqrt{\frac{2r}{n}\log\frac{1}{\delta}}\leq 2\sqrt{\frac{r}{2n}\log\frac{1}{\delta}}\leq \frac{r}{2}+\frac{1}{n}\log\frac{1}{\delta},
    $$
    (**) follows from the Ledoux-Talagrand contraction principle \citep{ledoux2013probability} and the last inequality follows since
    $$
    \mbb{E}\left[\sup_{f\in\mrm{star}(\mcal{F},f^*), \norm{f-f^*}^2\leq\tilde{M}^2r}\frac{1}{n}\sum_{i=1}^n\xi_i\frac{f^*(\mbf{X}_i)}{\tilde{M}}\right]=\mbb{E}\left[\frac{1}{n}\sum_{i=1}^n\xi_i\frac{f^*(\mbf{X}_i)}{\tilde{M}}\right]=0.
    $$
    Denote
    $$
    \psi(r)=10\mbb{E}\left[\sup_{f\in\mrm{star}(\mcal{F},f^*), \norm{f-f^*}^2\leq\tilde{M}^2r}\frac{1}{n}\sum_{i=1}^n\xi_i\frac{f(\mbf{X}_i)}{\tilde{M}}\right]+\frac{32}{3n}\log\frac{1}{\delta}.
    $$
    It then follows from Lemma \ref{Lm: localRad is sub-root} that $\psi(r)$ is a sub-root function. Therefore, we have for all $r\geq0$,
    \begin{align*}
        \psi(r) & =10\mbb{E}\left[\sup_{f\in\mrm{star}(\mcal{F},f^*), \norm{f-f^*}^2\leq\tilde{M}^2r}\frac{1}{n}\sum_{i=1}^n\xi_i\frac{f(\mbf{X}_i)}{\tilde{M}}\right]+\frac{32}{3n}\log\frac{1}{\delta}\\
        & \geq 2\mbb{E}\left[\sup_{f\in\mrm{star}(\mcal{F},f^*), \norm{f-f^*}^2\leq\tilde{M}^2r}\frac{1}{n}\sum_{i=1}^n\xi_i\frac{f(\mbf{X}_i)}{\tilde{M}}\right]\\
        & \geq 2\mbb{E}\left[\sup_{\bar{f}\in\bar{\mcal{F}}^0,\norm{\bar{f}}^2\leq r}\frac{1}{n}\sum_{i=1}^n\xi_i\bar{f}(\mbf{X}_i)\right]\\
        & \geq\mbb{E}\left[\sup_{g\in\mcal{G}_r^0}\frac{1}{n}\sum_{i=1}^n\xi_ig(\mbf{X}_i)\right].
    \end{align*}
    In addition, if $r$ satisfies $r\geq\psi(r)$, then with probability at least $1-\delta$,
    $$
    \frac{1}{n}\sum_{i=1}^n\bar{f}^2(\mbf{X}_i)\leq \frac{r}{2}+\frac{3r}{2}=2r.
    $$
    Consequently, if $r\geq\psi(r)$, with probability at least $1-\delta$,
    $$
    \left\{f\in\mrm{star}(\mcal{F}, f^*):\norm{f-f^*}^2\leq\tilde{M}^2r \right\}\subseteq\left\{f\in\mrm{star}(\mcal{F},f^*):\norm{f-f^*}_n^2\leq2\tilde{M}^2r\right\}.
    $$
    On the other hand, by Lemma A.4 in \citet{bartlett2005local}, with probability at least $1-\delta$,
    \begin{align*}
        \mbb{E}\left[\sup_{f\in\mrm{star}(\mcal{F},f^*), \norm{f-f^*}^2\leq \tilde{M}^2r}\frac{1}{n}\sum_{i=1}^n\xi_i\frac{f(\mbf{X}_{i})}{\tilde{M}}\right]\leq & \frac{1}{n}\log\frac{1}{\delta}+\\
        & 2\mbb{E}_\xi\left[\left.\sup_{f\in\mrm{star}(\mcal{F},f^*), \norm{f-f^*}^2\leq \tilde{M}^2r}\frac{1}{n}\sum_{i=1}^n\xi_i\frac{f(\mbf{X}_{i})}{\tilde{M}}\right|\mbf{X}_1,\ldots,\mbf{X}_n\right].\numberthis\label{Eq: Upper bound Rad complexity by emp Rad complexity}
    \end{align*}
    Then if $r\geq\psi(r)$, with probability at least $1-2\delta$,
    \begin{align*}
        \psi(r) & =10\mbb{E}\left[\sup_{f\in\mrm{star}(\mcal{F},f^*), \norm{f-f^*}^2\leq \tilde{M}^2r}\frac{1}{n}\sum_{i=1}^n\xi_i\frac{f(\mbf{X}_{i})}{\tilde{M}}\right]+\frac{32}{3n}\log\frac{1}{\delta}\\
        & \leq10\left(2\mbb{E}_\xi\left[\left.\sup_{f\in\mrm{star}(\mcal{F},f^*), \norm{f-f^*}^2\leq \tilde{M}^2r}\frac{1}{n}\sum_{i=1}^n\xi_i\frac{f(\mbf{X}_i)}{\tilde{M}}\right|\mbf{X}_1,\ldots,\mbf{X}_n\right]+\frac{1}{n}\log\frac{1}{\delta}\right)+\frac{32}{3n}\log\frac{1}{\delta}\\
        & =20\mbb{E}_\xi\left[\left.\sup_{f\in\mrm{star}(\mcal{F},f^*), \norm{f-f^*}^2\leq \tilde{M}^2r}\frac{1}{n}\sum_{i=1}^n\xi_i\frac{f(\mbf{X}_{i})}{\tilde{M}}\right|\mbf{X}_1,\ldots,\mbf{X}_n\right]+\frac{62}{3n}\log\frac{1}{\delta}\\
        & \leq20\mbb{E}_\xi\left[\left.\sup_{f\in\mrm{star}(\mcal{F},f^*), \norm{f-f^*}_n^2\leq 2\tilde{M}^2r}\frac{1}{n}\sum_{i=1}^n\xi_i\frac{f(\mbf{X}_{i})}{\tilde{M}}\right|\mbf{X}_1,\ldots,\mbf{X}_n\right]+\frac{62}{3n}\log\frac{1}{\delta}.
    \end{align*}  
    Based on our assumption, $\hat{\psi}_n(r)\geq\psi(r)$, we have with probability at least $1-2\delta$,
    \begin{equation}\label{Eq: r1*<r1*hat}
        r^*=\psi(r^*)\leq\hat{\psi}_n(r^*).
    \end{equation}
    Then by Lemma 4.3 in \citet{bartlett2005local}, with probability at least $1-2\delta$, $r^*\leq\hat{r}^*$. Putting (\ref{Eq: r1*<r1*hat}) and (\ref{Eq: Bound empirical norm by L2 norm}), (\ref{Eq: Bound L2 norm by empirical norm}) together, with probability at least $1-3\delta$, for any $g\in\mcal{G}^0$,
    \begin{align*}
        \mbb{E}[g(\mbf{X})] & \leq \frac{2}{n}\sum_{i=1}^ng(\mbf{X}_i)+2c_1\hat{r}^*+\frac{11+2c_2}{n}\log\frac{1}{\delta},\\
        \frac{1}{n}\sum_{i=1}^ng(\mbf{X}_i) & \leq\frac{3}{2}\mbb{E}[g(\mbf{X})]+2c_1\hat{r}^*+\frac{11+2c_2}{n}\log\frac{1}{\delta},
    \end{align*}
    which is equivalent to (recall that $g=[f-f^*]^2/\tilde{M}^2$),
    \begin{align*}
        \norm{f-f^*}^2 & \leq 2\norm{f-f^*}_n^2+2\tilde{M}^2c_1\hat{r}^*+\frac{\tilde{M}^2(11+2c_2)}{n}\log\frac{1}{\delta},\quad\forall f\in\mcal{F},\\
        \norm{f-f^*}_n^2 & \leq\frac{3}{2}\norm{f-f^*}^2+2\tilde{M}^2c_1\hat{r}^*+\frac{\tilde{M}^2(11+2c_2)}{n}\log\frac{1}{\delta},\quad\forall f\in\mcal{F}.
    \end{align*}
    The desired result then follows by using the elementary inequality $\sqrt{a+b}\leq\sqrt{a}+\sqrt{b}$ for $a,b\geq0$.
\end{proof}

\subsection{Proof of Theorem \ref{Thm: Extension of Bartlett Thm 5.4}}
To prove the theorem, we start by proving a preparatory lemma:

\begin{lemma}\label{Lm: Extension of Bartlett Thm 5.4}
	Let $\mcal{F}$ be a class of functions with ranges in $[-M, M]$ and let $\ell(\cdot, \cdot)$ be a loss function satisfying the conditions in Theorem \ref{Thm: Extension of Bartlett Thm 5.4}. Let $\hat{f}$ be any element of $\mcal{F}$ satisfying $R_n(\hat{f})=\inf_{f\in\mcal{F}}R_n(f)$. Assume $\psi$ is a sub-root function for which
	$$
	\psi(r)\geq B^*L^3\mbb{E}\left[\sup_{f\in\mcal{F},L^2\norm{f-f^*}^2\leq r}\frac{1}{n}\sum_{i=1}^n\xi_if(\mbf{X}_i)\right].
	$$
	Then for any $\delta>0$ and any $r\geq\psi(r)$, with probability at lease $1-\delta$,
	$$
	R(\hat{f})-R(f^*) \leq \frac{c_1C}{B^*L^2}r^*+\frac{11\bar{U}+c_2B^*L^2C}{n}\log\frac{1}{\delta},
	$$
        where $\bar{U}=\sup_{f\in\mcal{F}}\norm{\ell_f-\ell_{f^*}}_{\infty}$ and $C$, $c_1$ and $c_2$ are the same as those in Theorem \ref{Thm: Thm 3.3 in Bartlett et al 2005}.
\end{lemma}

\begin{proof}
	For simplicity, denote $\ell_f=\ell(f(\mbf{X}), Y)$ and let $\mcal{G}=\{\ell_f-\ell_{f^*}:f\in\mcal{F}\}$. For any $g=\ell_f-\ell_{f^*}\in\mcal{G}$, define
	$$
	T(g)=L^2\norm{f-f^*}^2.
	$$
	Then by the condition 2 and 3 in the lemma, 
	\begin{align*}
		\mrm{Var}[g] & \leq\mbb{E}[g^2]=\mbb{E}[(\ell_f-\ell_{f^*})^2]\leq L^2\mbb{E}[(f-f^*)^2]=T(g)\\
		T(g) & =L^2\mbb{E}[(f-f^*)^2]\leq L^2B^*[R(f)-R(f^*)]=L^2B^*\mbb{E}[g].
	\end{align*}
	On the other hand, by the contraction principle \citep{ledoux2013probability},
	\begin{align*}
		& \mbb{E}\left[\sup_{f\in\mcal{F}, L^2\norm{f-f^*}^2\leq r}\frac{1}{n}\sum_{i=1}^n\xi_i[\ell_f(\mbf{X}_i)-\ell_{f^*}(\mbf{X}_i)]\right]\\
		\leq & L\mbb{E}\left[\sup_{f\in\mcal{F},L^2\norm{f-f^*}^2\leq r}\frac{1}{n}\sum_{i=1}^n\xi_i[f(\mbf{X}_i)-f^*(\mbf{X}_i)]\right]\\
		= & L\mbb{E}\left[\sup_{f\in\mcal{F},L^2\norm{f-f^*}^2\leq r}\frac{1}{n}\sum_{i=1}^n\xi_if(\mbf{X}_i)\right]
	\end{align*}
	Hence, by the assumption on $\psi$,
	\begin{align*}
		\psi(r) & \geq B^*L^3\mbb{E}\left[\sup_{f\in\mcal{F},L^2\norm{f-f^*}^2\leq r}\frac{1}{n}\sum_{i=1}^n\xi_if(\mbf{X}_i)\right]\\
			& \geq B^*L^2\mbb{E}\left[\sup_{f\in\mcal{F}, L^2\norm{f-f^*}^2\leq r}\frac{1}{n}\sum_{i=1}^n\xi_i[\ell_f(\mbf{X}_i)-\ell_{f^*}(\mbf{X}_i)]\right]\\
			& =B^*L^2\mbb{E}\left[\sup_{g\in\mcal{G}, T(g)\leq r}\frac{1}{n}\sum_{i=1}^n\xi_ig(\mbf{X}_i)\right].
	\end{align*}
	Since $0\leq g\leq\sup_{f\in\mcal{F}}\norm{\ell_f-\ell_{f^*}}_{\infty}=\bar{U}$, it then follows from Theorem \ref{Thm: Thm 3.3 in Bartlett et al 2005} that for any $\delta>0$, with probability at least $1-\delta$,
	$$
	\mbb{E}[g(\mbf{X})]\leq\frac{C}{C-1}\frac{1}{n}\sum_{i=1}^ng(\mbf{X}_i)+\frac{c_1C}{B^*L^2}r^*+\frac{11\bar{U}+c_2B^*L^2C}{n}\log\frac{1}{\delta},\quad\forall g\in\mcal{G},
	$$
	where $r^*$ is the fixed point of $\psi(r)$. Replace $g$ by $\ell_{\hat{f}}-\ell_{f^*}$ and note that $R_n(\hat{f})-R_n(f^*)\leq0$ since $\hat{f}$ is the minimizer of $R_n(f)$. Hence, with probability at least $1-\delta$
	$$
	R(\hat{f})-R(f^*) \leq \frac{c_1C}{B^*L^2}r^*+\frac{11\bar{U}+c_2B^*L^2C}{n}\log\frac{1}{\delta}.
	$$
	Finally, by Lemma 3.2 in \citet{bartlett2005local} that $r\geq\psi(r)$ is equivalent to $r^*\leq r$, which leads to the desired result.
\end{proof}

We are now ready to prove the theorem.
\begin{proof}
    Define
    \begin{align*}
    \psi(r) & =\frac{\tilde{c}_1}{2}\mbb{E}\left[\sup_{f\in\mrm{star}(\mcal{F},f^*),L^2\norm{f-f^*}^2\leq r}\frac{1}{n}\sum_{i=1}^n\xi_if(\mbf{X}_i)\right]+\frac{(\tilde{c}_2-2M\tilde{c}_1)}{n}\log\frac{1}{\delta}\\
        & =\frac{\tilde{c}_1}{2L}\mbb{E}\left[\sup_{g\in\mcal{G}, \norm{g}^2\leq r}\frac{1}{n}\sum_{i=1}^n\xi_if(\mbf{X}_i)\right]+\frac{(\tilde{c}_2-2M\tilde{c}_1)}{n}\log\frac{1}{\delta},
    \end{align*}
    where $\mcal{G}=L\mrm{star}(\mcal{F},f^*)-\{Lf^*\}$. It is easy to see that $\norm{g}_\infty=2ML$ for any $g\in\mcal{G}$. Since $\mcal{F}\subseteq\mrm{star}(\mcal{F}, f^*)$, we have
    $$
    \psi(r)\geq B^*L^3\mbb{E}\left[\sup_{f\in\mcal{F}, L^2\norm{f-f^*}^2\leq r}\frac{1}{n}\sum_{i=1}^n\xi_if(\mbf{X}_i)\right].
    $$
    In addition, according to Lemma \ref{Lm: localRad is sub-root} with the choice of $T(g)=\norm{g}^2$ that $\psi$ is subroot. Now for $r\geq\psi(r)$, Lemma \ref{Lm: Extension of Bartlett Thm 5.4} and condition 3 on the loss function imply that for any $\delta>0$, with probability at least $1-\delta$,
    \begin{align*}
        L^2\norm{\hat{f}-f^*}^2 & \leq B^*L^2[R(\hat{f})-R(f^*)]\\
            & \leq c_1Cr+\frac{B^*L^2(11\bar{U}+c_2B^*L^2C)}{n}\log\frac{1}{\delta}.\numberthis\label{Eq: Delta}
    \end{align*}
    On the other hand, for $r\geq\psi(r)$, by the definition of $\tilde{c}_1$ and $\tilde{c}_2$, we have.
    $$
    r\geq\psi(r)\geq10\cdot(2ML)\mbb{E}\left[\sup_{f\in\mcal{G}, \norm{g}^2\leq r}\frac{1}{n}\sum_{i=1}^n\xi_if(\mbf{X}_i)\right]+\frac{11\cdot(2ML)^2}{n}\log\frac{1}{\delta}.
    $$
    Then by Corollary 2.2 in \citet{bartlett2005local}, with probability at least $1-\delta$,
    $$
    \{g\in\mcal{G}:\norm{g}^2\leq r\}\subseteq\{g\in\mcal{G}:\norm{g}_n^2\leq 2r\}.
    $$
    Combined with Lemma A.4 in \citet{bartlett2002rademacher}, with probability at least $1-2\delta$,
    \begin{align*}
        \psi(r) & \leq \frac{\tilde{c}_1}{2L}\left(2\mbb{E}_\xi\left[\sup_{g\in\mcal{G},\norm{g}^2\leq r}\frac{1}{n}\sum_{i=1}^n\xi_ig(\mbf{X}_i)\right]+\frac{4ML}{n}\log\frac{1}{\delta}\right)+\frac{\tilde{c}_2-2M\tilde{c}_1}{n}\log\frac{1}{\delta}\\
        & =\frac{\tilde{c}_1}{L}\mbb{E}_\xi\left[\sup_{g\in\mcal{G},\norm{g}^2\leq r}\frac{1}{n}\sum_{i=1}^n\xi_ig(\mbf{X}_i)\right]+\frac{\tilde{c}_2}{n}\log\frac{1}{\delta}\\
        & \leq\frac{\tilde{c}_1}{L}\mbb{E}_\xi\left[\sup_{g\in\mcal{G},\norm{g}_n^2\leq 2r}\frac{1}{n}\sum_{i=1}^n\xi_ig(\mbf{X}_i)\right]+\frac{\tilde{c}_2}{n}\log\frac{1}{\delta}\\
        & =\tilde{c}_1\mbb{E}_\xi\left[\sup_{f\in\mrm{star}(\mcal{F}, f^*),\norm{f-f^*}_n^2\leq 2\frac{r}{L^2}}\frac{1}{n}\sum_{i=1}^n\xi_if(\mbf{X}_i)\right]+\frac{\tilde{c}_2}{n}\log\frac{1}{\delta}\\
        & \leq\hat{\psi}_n(r).
    \end{align*}
    Setting $r=r^*$, the fixed point of $\psi(r)$ in the above argument and applying Lemma 4.3 in \citet{bartlett2005local} shows that $r^*\leq\hat{r}^*=\hat{\psi}_n(\hat{r}^*)$, which together with (\ref{Eq: Delta}) concludes the proof.
\end{proof}

\subsection{Proof of Theorem \ref{Thm: Lazy VI HT}}
\begin{proof}
In view of Theorem \ref{Thm: Thm 1 in williamson}, it suffices to check the required conditions. 
\begin{itemize}
    \item[] \textit{\textbf{Condition (D1)}}. By the first order Taylor expansion of $\ell(a,y)=-ya+\log(1+e^a)$ with respect to $a$,
    \begin{align*}
         & \abs{V(f,P_0)-V(f_0,P_0)}\\
        = & \abs{\mbb{E}_{P_0}\left[-Y f(\mbf{X})+\log\left(1+e^{f(\mbf{X})}\right)\right]-\mbb{E}_{P_0}\left[-Yf_0(\mbf{X})+\log\left(1+e^{f_0(\mbf{X})}\right)\right]}\\
        = &\abs{\mbb{E}_{P_0}\left[(-Y+\sigma(f_0(\mbf{X})))(f(\mbf{X})-f_0(\mbf{X}))+\frac{1}{2}\sigma(\tau(\mbf{X}))(1-\sigma(\tau(\mbf{X})))(f(\mbf{X})-f_0(\mbf{X}))^2\right]}\\
        = &\abs{\mbb{E}_{\mbf{X}}\left[(f(\mbf{X})-f_0(\mbf{X}))\mbb{E}_{Y|\mbf{X}}[-Y+\sigma(f_0(\mbf{X}))]\right]+\frac{1}{2}\mbb{E}_{P_0}\left[\sigma(\tau(\mbf{X}))(1-\sigma(\tau(\mbf{X})))(f(\mbf{X})-f_0(\mbf{X}))^2\right]}\\
        \overset{(*)}{=} & \frac{1}{2}\abs{\mbb{E}_{P_0}\left[\sigma(\tau(\mbf{X}))(1-\sigma(\tau(\mbf{X})))(f(\mbf{X})-f_0(\mbf{X}))^2\right]}\\
        \leq & \frac{1}{8}\norm{f-f_0}^2,
    \end{align*}
    where $\tau(\mbf{X})=\gamma f(\mbf{X})+(1-\gamma)f_0(\mbf{X})$ for some $\gamma\in[0,1]$ and (*) follows since $\mbb{E}_{Y|\mbf{X}}[Y]=\sigma(f_0(\mbf{X}))$ from Proposition \ref{Prop: About f0}.

    \item[] \textit{\textbf{Condition (D2)}}. For any $\eta_1, \eta_2,\cdots\in\mbb{R}$ and $H, H_1, H_2, \cdots\in\mcal{S}$ satisfying $\eta_j\to0$ and $\norm{H-H_j}_\infty\to0$, by (\ref{Eq: Gateaux derivative}),
    \begin{align*}
        & \sup_{f\in\mcal{F}} \left| \frac{V(f, P_0 + \eta_j H_j) - V(f, P_0)}{\eta_j} - \dot{V}(f, P_0; H_j) \right| \\
        = & \sup_{f\in\mcal{F}} \left| \frac{V(f, P_0 + \eta_j H_j) - V(f, P_0)}{\eta_j} - \lim_{t \to 0} \frac{V(f, P_0 + t H_j) - V(f, P_0)}{t} \right| \\   
        = & \sup_{f\in\mcal{F}} \left| \mathbb{E}_{H_j}[\ell(f(\mbf{X}), Y)] - \mathbb{E}_{H_j}[\ell(f(\mbf{X}), Y)] \right| = 0. 
    \end{align*}

     \item[] \textit{\textbf{Condition (R1)}}. Under the assumption that $W=o(\sqrt{n}/\log n)$, Theorem \ref{Thm: RoC of DNN} shows that with probability at least $1-e^{-o(n^{1/2})}$, 
     \begin{equation}\label{Eq: DoC of DNN-2}
     \norm{h_{\mbf{\theta}_f}-f_0}=o(n^{-1/4}),\quad \norm{h_{\mbf{\theta}_f,-S}-f_{0,-S}}=o(n^{-1/4})
     \end{equation}
     On the other hand, by combining Theorem \ref{Thm: Distance between h tilde and h} and Theorem \ref{Thm: Estimation error of h tilde}, we have with probability at least $1-31e^{-n^{\frac{1}{\alpha+1}}}$,
     \begin{align*}
         \norm{h_{\mbf{\theta}_f+\Delta\mbf{\theta}_S}-f_{0,-S}} & \leq\norm{h_{\mbf{\theta}_f+\Delta\mbf{\theta}_S}-\tilde{h}_{\mbf{\theta}_f+\Delta\mbf{\theta}_S}} +\norm{\tilde{h}_{\mbf{\theta}_f+\Delta\mbf{\theta}_S}-f_{0,-S}}\\
         & \lesssim_\alpha o\left(n^{-\frac{1}{4}}\right)+o\left(n^{-\frac{\alpha+3}{2(\alpha+1)}}\right)\sqrt{W}\log^{1/2}n+n^{-1/2}\sqrt{W}\log^{1/2}n+o(n^{-1/4})+\\
        & \quad\left[o\left(n^{-\frac{1}{4}\left(1-\frac{\alpha-1}{2\alpha}\right)}\right)+o\left(n^{-\frac{1}{4}\left(1+\frac{\alpha+3}{2\alpha}\right)}\right)\left[\sqrt{W}\log^{1/2}n\right]^{\frac{\alpha+1}{2\alpha}}\right]\norm{f_{0,-S}-h_{\mbf{\theta}_f,-S}}^{\frac{\alpha-1}{2\alpha}}\\
        & \lesssim_\alpha o\left(n^{-\frac{1}{4}}\right)+o\left(n^{-\frac{\alpha+3}{2(\alpha+1)}}\right)\cdot o\left(n^{1/4}\right)+o(n^{-1/4})+\\
        & \quad\left[o\left(n^{-\frac{1}{4}\left(1-\frac{\alpha-1}{2\alpha}\right)}\right)+o\left(n^{-\frac{1}{4}\left(1+\frac{\alpha+3}{2\alpha}\right)}\right)o\left(n^{\frac{\alpha+1}{8\alpha}}\right)\right]o\left(n^{-\frac{\alpha-1}{8\alpha}}\right)\\
        & \lesssim_\alpha o(n^{-1/4}),
     \end{align*}
     where the last inequality follows since
     \begin{align*}
         o\left(n^{-\frac{\alpha+3}{2(\alpha+1)}}\right)\cdot o\left(n^{1/4}\right) & = o\left(n^{-\frac{1}{4}\frac{\alpha+2}{\alpha+1}}\right)=o(n^{-1/4})\\
         o\left(n^{-\frac{1}{4}\left(1-\frac{\alpha-1}{2\alpha}\right)}\right)\cdot o\left(n^{-\frac{\alpha-1}{8\alpha}}\right) & =o\left(n^{-1/4}\right)\\
         o\left(n^{-\frac{1}{4}\left(1+\frac{\alpha+3}{2\alpha}\right)}\right)\cdot o\left(n^{\frac{\alpha+1}{8\alpha}}\right)\cdot o\left(n^{-\frac{\alpha-1}{8\alpha}}\right) & = o\left(n^{-\frac{1}{4}-\frac{\alpha+1}{8\alpha}}\right)=o(n^{-1/4}).\numberthis\label{Eq: DNNRoc under H1}
     \end{align*}

     \item[] \textit{\textbf{Condition (R2)}}. Let  \( g_n(\mbf{z}) \coloneq \dot{V}(h_{\mbf{\theta}_f}, P_0; \delta_{\mbf{z}} - P_0) - \dot{V}(f_0, P_0; \delta_{\mbf{z}} - P_0) \). Then
    \begin{align*}
        \int [g_n(\mbf{z})]^2 \, dP_0(\mbf{z})
        &= \int \left[\mathbb{E}_{\delta_{\mbf{z}} - P_0}[\ell(h_{\mbf{\theta}_f}(\mbf{X}), Y)] - \mathbb{E}_{\delta_{\mbf{z}} - P_0}[\ell(f_0(\mbf{X}), Y)]\right]^2 \, dP_0(\mbf{z}) \\
        &= \mathbb{E}_{P_0}\left[\left(\ell(h_{\mbf{\theta}_f}(\mbf{X}), Y) - \ell(f_0(\mbf{X}), Y) - \mathbb{E}_{P_0}[\ell(h_{\mbf{\theta}_f}(\mbf{X}), Y) - \ell(f_0(\mbf{X}), Y)]\right)^2\right] \\
        &= \text{Var}_{P_0}[\ell(h_{\mbf{\theta}_f}(\mbf{X}), Y) - \ell(f_0(\mbf{X}), Y)] \\
        &\leq \mathbb{E}_{P_0}[(\ell(h_{\mbf{\theta}_f}(\mbf{X}), Y) - \ell(f_0(\mbf{X}), Y))^2] \\
        &\leq \|h_{\mbf{\theta}_f} - f_0\|^2, 
    \end{align*}
    where the last inequality follows from the Lipshitzness of $\ell$. Therefore, it follows from (\ref{Eq: DoC of DNN-2}) that with probability at least $1-e^{-o(n^{1/2})}$,
    $$
    \int [g_n(\mbf{z})]^2 \, dP_0(\mbf{z})=o(1).
    $$
    Similarly, let $g_{n,-S}(\mbf{z})\coloneq \dot{V}(h_{\mbf{\theta}_f+\Delta\mbf{\theta}_S}, P_0; \delta_{\mbf{z}} - P_0) - \dot{V}(f_{0,-S}, P_0; \delta_{\mbf{z}} - P_0)$ and by the same arguments as above, we have
    $$
    \int [g_{n,-S}(\mbf{z})]^2 \, dP_0(\mbf{z})\leq\norm{h_{\mbf{\theta}_f+\Delta\mbf{\theta}_S}-f_{0,-S}}^2.
    $$
    Combined with (\ref{Eq: DNNRoc under H1}), with probability at least $1-31e^{-n^{\frac{1}{\alpha+1}}}$,
    $$
    \int [g_{n,-S}(\mbf{z})]^2 \, dP_0(\mbf{z})=o(1).
    $$

    \item[] \textit{\textbf{Condition (R3)}}. As we have seen above, $g_n(\mbf{Z})=\ell(h_{\mbf{\theta}_f}(\mbf{X}), Y)-\ell(f_0(\mbf{X}), Y)-\mbb{E}_{P_0}[\ell(h_{\mbf{\theta}_f}(\mbf{X}), Y)-\ell(f_0(\mbf{X}), Y)]$, let us define 
    $$
    \mcal{D}=\{\mbf{Z}\mapsto\ell(f(\mbf{X}), Y)-\ell(f_0(\mbf{X}), Y)-\mbb{E}_{P_0}\left[\ell(f(\mbf{X}), Y)-\ell(f_0(\mbf{X}), Y)\right]:f\in\mcal{F}\}.
    $$
    So it is obvious that $g_n\in\mcal{D}$, $P_0$-a.s. In view of Theorem 3.1 in \citet{ossiander1987central}, to show that $\mcal{D}$ is a $P_0$-Donsker class, it suffice to show that 
    $$
    \int_0^\infty\sqrt{\log N_{[\mrm{ }]}(\varepsilon,\mcal{D}, \norm{\cdot}_{L_2(P_0)})}d\varepsilon<\infty.
    $$
    On the other hand, since $N_{[\mrm{ }]}(\varepsilon,\mcal{D}, \norm{\cdot}_{L_2(P_0)})\leq N_{[\mrm{ }]}(\varepsilon,\mcal{D}, \norm{\cdot}_{\sup})=N\left(\frac{\varepsilon}{2},\mcal{D},\norm{\cdot}_{\sup}\right)$, where the last inequality follows from the relationship between covering number and bracketing number mentioned in \citet{van1996weak}, we have
    $$
    \int_0^\infty\sqrt{\log N_{[\mrm{ }]}(\varepsilon,\mcal{D}, \norm{\cdot}_{L_2(P_0)})}d\varepsilon\leq\int_0^\infty\sqrt{\log N\left(\frac{\varepsilon}{2},\mcal{D}, \norm{\cdot}_{\sup}\right)}d\varepsilon.
    $$
    Now, let $\{f_1,\ldots, f_N\}$ be a minimal $\varepsilon$-cover of $\mcal{F}$ with respect to the uniform norm $\norm{\cdot}_{\sup}$. Then $N=N(\varepsilon,\mcal{F},\norm{\cdot}_{\sup})$ and for any $f\in\mcal{F}$, there exists $j\in[N]$ such that $\norm{f-f_j}_{\sup}<\varepsilon$. Denote
    $$
    g_n^{(j)}(\mbf{Z})=\ell(f_j(\mbf{X}), Y)-\ell(f_0(\mbf{X}), Y)-\mbb{E}_{P_0}\left[\ell(f_j(\mbf{X}), Y)-\ell(f_0(\mbf{X}), Y)\right].
    $$
    Then for any $g_n\in\mcal{D}$, by the Lipschitz continuity of $\ell$, we have
    \begin{align*}
        \norm{g_n-g_n^{(j)}}_{\sup} & =\sup_{\mbf{X}, Y}\abs{\ell(f(\mbf{X}), Y)-\mbb{E}_{P_0}[\ell(f(\mbf{X}), Y)]-[\ell(f_j(\mbf{X}), Y)-\mbb{E}_{P_0}[\ell(f_j(\mbf{X}), Y)]]}\\
        & \leq\sup_{\mbf{X}, Y}\abs{\ell(f(\mbf{X}), Y)-\ell(f_j(\mbf{X}), Y)}+\sup_{\mbf{X}, Y}\abs{\mbb{E}_{P_0}\left[\ell(f(\mbf{X}), Y)-\ell(f_j(\mbf{X}), Y)\right]}\\
        & \leq\sup_{\mbf{X}}\abs{f(\mbf{X})-f_j(\mbf{X})}+\sup_{\mbf{X}}\mbb{E}_{P_0}\left[\abs{f(\mbf{X})-f_j(\mbf{X})}\right]\\
        & \leq 2\sup_{\mbf{X}}\abs{f(\mbf{X})-f_j(\mbf{X})}\\
        & =2\norm{f-f_j}_{\sup},
    \end{align*}
    which implies that
    $$
    N\left(\frac{\varepsilon}{2},\mcal{D},\norm{\cdot}_{\sup}\right)\leq N\left(\frac{\varepsilon}{4},\mcal{F},\norm{\cdot}_{\sup}\right).
    $$
    It then follows from Theorem \ref{Thm: Sup Covering Number Upper Bound for DNN} that
    \begin{align*}
        \int_0^\infty\sqrt{\log N\left(\frac{\varepsilon}{4},\mcal{F},\norm{\cdot}_{\sup}\right)}d\varepsilon & \leq\int_0^{4M}\sqrt{W}\log^{1/2}\left(1+\frac{8\kappa}{\varepsilon}\left[\sum_{l=1}^L\left(\frac{b_l}{\kappa_l}\right)^{1/2}\prod_{l=1}^L\kappa_l^{1/2}\right]^2\right)d\varepsilon<\infty,
    \end{align*}
    where in the first inequality, we can change the upper bound of the entropy integral to $4M$ since functions in $\mcal{F}$ are uniformly bounded by $M$ and the finiteness of the entropy integral follows by a similar argument as in the proof of Lemma \ref{Lm: Local Rademacher Complexity of DNN}. Hence, $\mcal{D}$ is a $P_0$-Donsker class. 
\end{itemize}
\end{proof}

\section{Auxiliary Results}\label{appendix: Auxiliary results}
\subsection{The RKHS Associated with the NTK}
In this section, we provide the detailed proofs of the results on the RKHS generated by the NTK. 

\begin{lemma}[Local Rademacher Complexity of $\mcal{H}_B\}$]\label{Lm: Local Rad Comp of HB}
\begin{align*}
    \hat{\mcal{R}}_n(r,\mcal{H}_B) & \coloneq\mbb{E}_\xi\left[\left.\sup_{f\in\mcal{H}_B, \norm{f}_n^2\leq r}\frac{1}{n}\sum_{i=1}^n\xi_if(\mbf{X}_{i,S})\right|\mbf{X}_1,\ldots,\mbf{X}_n\right]\\
        & \leq (1\vee B)\sqrt{\frac{2}{n}}\sqrt{\sum_{j=1}^n\min\{r,\mu_j\}},
\end{align*}
where $\mu_1\geq\mu_2\geq\cdots\geq\mu_n>0$ are the eigenvalues of $\frac{1}{n}\mbf{K}_{-S}$ arranged in a decreasing order.
\end{lemma}

\begin{proof}
    Note that
    \begin{align*}
        \norm{f}_n^2 & =\frac{1}{n}\sum_{i=1}^nf^2(\mbf{X}_{i,-S})=\frac{1}{n}\sum_{i=1}^n\left(\sum_{j=1}^n\alpha_j K_{-S}(\mbf{X}_{i}, \mbf{X}_{j})\right)^2=\frac{1}{n}\sum_{i=1}^n[\mbf{K}_{-S}\mbf{\alpha}]_i^2\\
            & =\frac{1}{n}(\mbf{K}_{-S}\mbf{\alpha})^T(\mbf{K}_{-S}\mbf{\alpha})=n\left(\frac{1}{n}\mbf{K}_{-S}\mbf{\alpha}\right)^T\left(\frac{1}{n}\mbf{K}_{-S}\mbf{\alpha}\right).
    \end{align*}
    Let $\frac{1}{n}\mbf{K}_{-S}=\mbf{U}\mbf{\Lambda}^{-1}\mbf{U}^T$ be the spectral decomposition of $\frac{1}{n}\mbf{K}_{-S}$ with $\mbf{\Lambda}=\mrm{Diag}\{\mu_1,\mu_2,\ldots,\mu_n\}$ and denote $\mbf{\beta}=\mbf{U}^T\left(\frac{1}{n}\mbf{K}_{-S}\right)\mbf{\alpha}$. Then
    \begin{align*}
        \norm{f}_n^2\leq r & \Leftrightarrow n\left(\frac{1}{n}\mbf{K}_{-S}\mbf{\alpha}\right)^T\mbf{UU}^T\left(\frac{1}{n}\mbf{K}_{-S}\mbf{\alpha}\right)\leq r\\
            & \Leftrightarrow n\mbf{\beta}^T\mbf{\beta}\leq r\\
            & \Leftrightarrow\sum_{i=1}^n\frac{n\beta_i^2}{r}\leq 1.\\
        \norm{f}_{\mcal{H}}^2\leq B^2 & \Leftrightarrow \mbf{\alpha}^T\mbf{K}_{-S}\mbf{\alpha}\leq B^2\\
            & \Leftrightarrow n\left(\frac{1}{n}\mbf{K}_{-S}\mbf{\alpha}\right)^T\mbf{U}\mbf{\Lambda}^{-1}\mbf{U}^T\left(\frac{1}{n}\mbf{K}_{-S}\right)\mbf{\alpha}\leq B^2\\
            & \Leftrightarrow n\mbf{\beta}^T\mbf{\Lambda}^{-1}\mbf{\beta}\leq B^2\\
            & \Leftrightarrow \sum_{i=1}^n\frac{n\beta_i^2}{B^2\mu_i}\leq 1,
    \end{align*}
    Therefore,
    $$
    \{f\in\mcal{H}_B:\norm{f}_n^2\leq r\}\subseteq\mcal{D}\coloneq\left\{\mbf{\beta}\in\mbb{R}^n:\sum_{i=1}^n\eta_i\beta_i^2\leq2\right\},
    $$
    where $\eta_i=\max\{nr^{-1}, nB^{-2}\mu_i^{-1}\}$. Let $\mbf{\xi}^T=[\xi_1,\ldots, \xi_n]^T$, it then follows from Cauchy-Schwarz inequality that
    \begin{align*}
        \hat{\mcal{R}}_n(r;\mcal{H}_B) & =\frac{1}{n}\mbb{E}_\xi\left[\left.\sup_{\mbf{\alpha}^T\mbf{K}_{-S}\mbf{\alpha}\leq B^2, \alpha^T\mbf{K}_{-S}^2\mbf{\alpha}\leq r}\sum_{i=1}^n\xi_i[\mbf{K}_{-S}\mbf{\alpha}]_i\right|\mbf{X}_1,\ldots,\mbf{X}_n\right]\\
        & =\mbb{E}_\xi\left[\left.\sup_{\mbf{\alpha}^T\mbf{K}_{-S}\mbf{\alpha}\leq B^2, \alpha^T\mbf{K}_{-S}^2\mbf{\alpha}\leq r}\mbf{\xi}^T\mbf{UU}^T\left(\frac{1}{n}\mbf{K}_{-S}\mbf{\alpha}\right)\right|\mbf{X}_1,\ldots,\mbf{X}_n\right]\\
        & \leq \mbb{E}_\xi\left[\left.\sup_{\mbf{\beta}\in\mcal{D}}\sum_{i=1}^n(\mbf{\xi}^T\mbf{u}_i)\beta_i\right|\mbf{X}_1,\ldots,\mbf{X}_n\right]\\
        & =\mbb{E}_\xi\left[\left.\sup_{\mbf{\beta}\in\mcal{D}}\sum_{i=1}^n\frac{(\mbf{\xi}^T\mbf{u}_i)}{\sqrt{\eta_i}}\sqrt{\eta_i}\beta_i\right|\mbf{X}_1,\ldots,\mbf{X}_n\right]\\
        & \leq\mbb{E}_\xi\left[\left.\sup_{\mbf{\beta}\in\mcal{D}}\left(\sum_{i=1}^n\frac{(\mbf{\xi}^T\mbf{u}_i)^2}{\eta_i}\right)^{1/2}\left(\sum_{i=1}^n\eta_i\beta_i^2\right)^{1/2}\right|\mbf{X}_1,\ldots,\mbf{X}_n\right]\\
        & \leq\sqrt{2}\mbb{E}_\xi\left[\left.\sqrt{\sum_{i=1}^n\frac{(\mbf{\xi}^T\mbf{u}_i)^2}{\eta_i}}\right|\mbf{X}_1,\ldots,\mbf{X}_n\right]\\
        & \overset{(*)}{\leq}\sqrt{2}\sqrt{\sum_{i=1}^n\frac{1}{\eta_i}}\\
        & =\sqrt{\frac{2}{n}}\sqrt{\sum_{i=1}^n\min\left\{\frac{r}{n}, \frac{B^2\mu_i}{n}\right\}}\\
        & \leq (1\vee B)\sqrt{\frac{2}{n}}\sqrt{\sum_{j=1}^n\min\{r,\mu_j\}},
    \end{align*}
    where (*) follows from Jensen's inequality and $\mbf{E}_{\xi}[(\mbf{\xi}^T\mbf{u}_i)^2|\mbf{X}_1,\ldots,\mbf{X}_n]=\mbb{E}_{\xi}\left[\mbf{u}_i^T\mbf{\xi\xi}^T\mbf{u}_i|\mbf{X}_1,\ldots,\mbf{X}_n\right]=\mbf{u_1}\mbb{E}_{\xi}[\mbf{\xi\xi}^T]\mbf{u}_i=\mbf{u}_i^T\mbf{u}_i=1$. 
\end{proof}

\begin{remark}
It is well-known that the empirical Rademacher complexity of $\mcal{H}_B$ is $\hat{\mcal{R}}_n(\mcal{H}_B)\leq\frac{B}{n}\sqrt{\mrm{tr}(\mbf{K}_{-S})}$ \citep{bartlett2002rademacher}. If one takes $r=\infty$ in the upper bound in Lemma \ref{Lm: Local Rad Comp of HB}, we get
\begin{align*}
\hat{\mcal{R}}_n(\infty,\mcal{H}_B) & \leq\frac{\sqrt{2}(1\vee B)}{\sqrt{n}}\sqrt{\sum_{j=1}^n\mu_j}\\
    & =\frac{\sqrt{2}(1\vee B)}{\sqrt{n}}\sqrt{\mrm{tr}\left(\frac{1}{n}\mbf{K}_{-S}\right)}\\
    & =\frac{\sqrt{2}(1\vee B)}{n}\sqrt{\mrm{tr}(\mbf{K}_{-S})}.
\end{align*}
So the bound reproduces the bound for the empirical Rademacher complexity up to some constants. 
\end{remark}

\begin{lemma}[Uniform Boundedness of Functions in $\mcal{H}_B$]\label{Lm: Uniform Boundedness of HB}
    For any $f\in\mcal{H}_B$,
    $$
    \sup_{\mbf{x}}\abs{f(\mbf{x})}\leq B\sqrt{\sup_{\mbf{x}}K_{-S}(\mbf{x},\mbf{x})}.
    $$
\end{lemma}

\begin{proof}
    For any $f\in\mcal{H}_B$, there exists $\alpha_1,\ldots,\alpha_n\in\mbb{R}$ such that $f(\mbf{x})=\sum_{j=1}^n\alpha_jK_{-S}(\mbf{x}, \mbf{X}_{j,-S})$. Then by Cauchy-Schwarz inequality,
    \begin{align*}
        \abs{f(\mbf{x})} & =\abs{\sum_{j=1}^n\alpha_j K_{-S}(\mbf{x}, \mbf{X}_{j,-S})}=\abs{\sum_{j=1}^n\alpha_j\inprod{K_{-S}(\cdot, \mbf{X}_{j,-S})}{K_{-S}(\cdot, \mbf{x})}_{\mcal{H}}}\\
        & =\abs{\inprod{\sum_{j=1}^n\alpha_jK_{-S}(\cdot, \mbf{X}_{j,-S})}{K_{-S}(\cdot, \mbf{x})}_{\mcal{H}}}\\
        & \leq\sqrt{\inprod{\sum_{j=1}^n\alpha_jK_{-S}(\cdot, \mbf{X}_{j,-S})}{\sum_{i=1}^n\alpha_jK_{-S}(\cdot, \mbf{X}_{i,-S})}_{\mcal{H}}}\sqrt{\inprod{K_{-S}(\cdot, \mbf{x})}{K_{-S}(\cdot, \mbf{x})}_{\mcal{H}}}\\
        & =\sqrt{\sum_{i=1}^n\sum_{j=1}^n\alpha_i\alpha_jK_{-S}(\mbf{X}_{i,-S}, \mbf{X}_{j,-S})}\sqrt{K_{-S}(\mbf{x},\mbf{x})}\\
        & \leq B\sqrt{K_{-S}(x,x)}.
    \end{align*}
    Therefore, for any $f\in\mcal{H}_B$
    $$
    \sup_{\mbf{x}}\abs{f(\mbf{x})}\leq B\sup_{\mbf{x}}\sqrt{K_{-S}(\mbf{x},\mbf{x})}.
    $$
\end{proof}

\subsection{Lipschitzness of Sigmoid Function}
\begin{lemma}\label{lem:sig_lip}
    For any two vectors $\mbf{v}, \mbf{w} \in \mathbb{R}^d$, we have
\begin{equation*}
    \norm{\sigma(\mbf{v})-\sigma(\mbf{w})}  \leq \frac{1}{4} \norm{\mbf{v}-\mbf{w}}
\end{equation*}   
\end{lemma}

\begin{proof}
By Lipschitz continuity of the sigmoid function, we know that for any $x_1, x_2\in \mathbb{R}$, the following is true:
\begin{equation*}
    |\sigma(x_1)-\sigma(x_2)|\leq \frac{1}{4}|x_1-x_2|
\end{equation*}\\
Thus, for any 2 vectors $\mbf{v}, \mbf{w} \in \mathbb{R}^d$,
\begin{equation*}
    \norm{\sigma(\mbf{v})-\sigma(\mbf{w})}^2  =\sum_{i=1}^{d} \abs{\sigma(v_i)-\sigma(w_i)}^2\\
    \leq \sum_{i=1}^{d} \Big(\frac{1}{4}{\abs{v_i-w_i}\Big)}^2\\
    \leq \frac{1}{16}\sum_{i=1}^{d} ({v_i-w_i)}^2\\
    = \frac{1}{16} \norm{\mbf{v}-\mbf{w}}^2
\end{equation*}
Taking the square root yields
\begin{equation*}
    \norm{\sigma(\mbf{v})-\sigma(\mbf{w})}  \leq \frac{1}{4} \norm{\mbf{v}-\mbf{w}}
\end{equation*}
\end{proof}

\subsection{Covering Number of the Star Hull of a Function Class}
The following lemma is similar to Lemma 4.5 in \citet{mendelson2002improving}, which shows that the covering number of the star hull of a function class $\mcal{F}$ is almost the same as the covering number of $\mcal{F}$.
\begin{lemma}\label{Lm: Covering Number of Star-hull}
    Suppose that $\mcal{F}$ is a class of functions with ranges in $[-M, M]$ and $f_0\in\mcal{F}$ be a fixed function. Then for any pseudo-norm $\norm{\cdot}$,
    $$
    \log N(2\varepsilon, \mrm{star}(\mcal{F}, f_0), \norm{\cdot})\leq\log\frac{2M}{\varepsilon}+\log N(\varepsilon,\mcal{F},\norm{\cdot}).
    $$
\end{lemma}

\begin{proof}
    Fix $0<\varepsilon\leq 2M$. Let $\{f_1,\ldots,f_N\}$ be a minimal $\varepsilon$-cover of $\mcal{F}$. Then $N=N(\varepsilon,\mcal{F},\norm{\cdot})$. In addition, define
    $$
    \mrm{star}(\{f_1,\ldots, f_N\}, f_0)=\bigcup_{j=1}^N\{f_0+\alpha(f_j-f_0): \alpha\in[0,1]\}.
    $$
    For any $g\in\mrm{star}(\mcal{F}, f_0)$, there exists $f\in\mcal{F}$ and $\tilde{\alpha}\in[0,1]$ such that $g=f_0+\tilde{\alpha}(f-f_0)$. For such a function $f$, there exists $j\in[N]$ such that $\norm{f-f_j}<\varepsilon$. Denote
    $$
    h_j=f_0+\tilde{\alpha}(f_j-f_0)\in\mrm{star}(\{f_1,\ldots, f_N\}, f_0)
    $$
    which implies that
    $$
    \norm{g-h_j}=\alpha\norm{f-f_j}<\varepsilon.
    $$
    On the other hand, let $g_1,\ldots,g_{N'}$ be a minimal $\varepsilon$-cover of $\mrm{star}(\{f_1,\ldots,f_N\}, f_0)$ so that $N'=N(\varepsilon,\mrm{star}(\{f_1,\ldots,f_N\}, f_0),\norm{\cdot})$ there exists $k\in[N']$ such that 
    $$
    \norm{h_j-g_k}<\varepsilon.
    $$
    As a result, we have
    $$
    \norm{g-g_k}\leq\norm{g-h_j}+\norm{h_j-g_k}<2\varepsilon.
    $$
    This shows that $\{g_1,\ldots,g_{N'}\}$ is a $2\varepsilon$-cover of $\mrm{star}(\mcal{F},f_0)$ and hence
    \begin{align*}
    N(2\varepsilon,\mrm{star}(\mcal{F},f_0),\norm{\cdot}) & \leq N(\varepsilon,\mrm{star}(\{f_1,\ldots,f_N\},f_0),\norm{\cdot})\\
        & =N\cdot \max_{1\leq j\leq N}N(\varepsilon,\mrm{star}(\{f_j\}, f_0),\norm{\cdot}).
    \end{align*}
    Note that for each $j$, we have $\norm{f_j-f_0}\leq 2M$. Now consider the partition of $[0,1]$ with breakpoints $\varepsilon_i=i\varepsilon/(2M)$, $i=1, \ldots, \lfloor 2M/\varepsilon\rfloor$. For any $\phi\in\mrm{star}(\{f_j\}, f_0)$, there exists $\alpha\in[0,1]$ such that $\phi=f_0+\alpha(f_j-f_0)$. Based on the construction of partition, there exists $i_\phi\in\{0, 1, \ldots, \lfloor 2M/\varepsilon\rfloor\}$ such that $\abs{\alpha-\varepsilon_{i_\phi}}\leq\frac{\varepsilon}{2M}$. As a result
    $$
    \norm{\phi-(f_0+\varepsilon_{i\phi}(f_j-f_0))}=\abs{\alpha-\varepsilon_{i_\phi}}\norm{f_j-f_0}\leq\varepsilon,
    $$
    which implies that $N(\varepsilon,\mrm{star}(\{f_j\}, f_0),\norm{\cdot})\leq\frac{2M}{\varepsilon}$. Hence,
    $$
    N(2\varepsilon,\mrm{star}(\mcal{F},f_0),\norm{\cdot})\leq N(\varepsilon,\mrm{star}(\{f_1,\ldots,f_N\},f_0),\norm{\cdot})\leq\frac{2M}{\varepsilon}N(\varepsilon,\mcal{F},\norm{\cdot}).
    $$
    Taking logarithm on both sides yields the desired result.
\end{proof}

\subsection{Fixed Points of Adding a Constant to Sub-root Functions}
\begin{lemma}\label{Lm: adding a constant to sub-root function}
    Let $\psi(r)$ be a sub-root function and let $a>0$ be a constant. Define
    $$
    \psi_a(r)=\psi(r)+a,
    $$
    and let $r^*, r_a^*$ be the fixed points of $\psi$ and $\psi_a$ respectively. Then
    $$
    r_a^*\leq r^*+2a. 
    $$
\end{lemma}

\begin{proof}
    By the definition of $\psi_a$, we have
    $$
    \psi(r_a^*)=\psi_a(r_a^*)-a=r_a^*-a\leq r_a^*,
    $$
    it then follows from Lemma 3.2 in \citet{bartlett2005local} that $r^*\leq r_a^*$. On the other hand, since $\psi(r)/\sqrt{r}$ is non-increasing, we have
    $$
    \frac{\psi(r_a^*)}{\sqrt{r_a^*}}\leq\frac{\psi(r^*)}{\sqrt{r^*}}=\sqrt{r^*},
    $$
    which implies that
    $$
    r_a^*-a=\psi_a(r_a^*)-a=\psi(r_a^*)\leq\sqrt{r_a^*r^*}.
    $$
    Squaring both sides of the above inequality yields
    $$
    r_a^*-(2a+r^*)r_a^*+a^2\leq0,
    $$
    and solving this inequality with respect to $r_a^*$, we have
    $$
    r_a^*\leq\frac{(2a+r^*)+\sqrt{(2a+r^*)^2-4a^2}}{2}\leq r^*+2a.
    $$
\end{proof}

\subsection{Covering Number of Deep ReLU Network Under the $L^\infty$-norm}
\begin{lemma}\label{Lm: Covering Number of 1-layer}
    Let $\mcal{Z}_\zeta=\{\mbf{z}\in\mbb{R}^k:\norm{\mbf{z}}\leq \zeta\}$ for some $\zeta>0$ and let $\sigma$ be a 1-Lipschitz function. Define
    \begin{align*}
        f:\mcal{Z}_\zeta & \to\mbb{R}^m\\
            \mbf{z} & \mapsto \sigma(\mbf{Wz})=\begin{bmatrix}
                \sigma(\mbf{w}_1^T\mbf{z})\\
                \sigma(\mbf{w}_2^T\mbf{z})\\
                \vdots\\
                \sigma(\mbf{w}_m^T\mbf{z})
            \end{bmatrix},\numberthis\label{Eq: 1 layer network}
    \end{align*}
    where $\mbf{W}[\mbf{w}_1,\ldots,\mbf{w}_m]^T\in\mbb{R}^{m\times k}\in\mcal{W}$ and
    $$
    \mcal{W}=\left\{\mbf{W}:\norm{\mbf{W}}_{op}\leq\kappa_0,\norm{\mbf{W}^T}_{2,1}\leq b_0\right\},
    $$
    for some $\kappa_0, b_0>0$. Denote $\mcal{M}=\{f\mrm{ as defined in (\ref{Eq: 1 layer network})}:\mbf{W}\in\mcal{W}\}$. Then
    $$
    N(\varepsilon,\mcal{M},\norm{\cdot}_{\sup})\leq\left(1+\frac{\zeta b_0}{\varepsilon}\right)^{mk}.
    $$
    Here $\norm{f}_{\sup}=\sup_{\mbf{z}\in\mcal{Z}_\zeta}\norm{f(\mbf{z})}$.
\end{lemma}

\begin{proof}
    Note that for any $f, f'\in\mcal{M}$ with $f(\mbf{z})=\sigma\left(\mbf{W}\mbf{z}\right)$ and $f'(\mbf{z})=\sigma\left(\mbf{W}'\mbf{z}\right)$, we have
    \begin{align*}
        \norm{f-f'}_{\sup} & =\sup_{\mbf{z}\in\mcal{Z}_\zeta}\norm{f(\mbf{z})-f'(\mbf{z})}\\
            & =\sup_{\mbf{z}\in\mcal{Z}_\zeta}\sqrt{\sum_{t=1}^m\abs{\sigma\left(\mbf{w}_t^T\mbf{z}\right)-\sigma\left(\mbf{w'}_t{^T}\mbf{z}\right)}^2}\\
            & \leq \sup_{\mbf{z}\in\mcal{Z}_\zeta}\sqrt{\sum_{t=1}^m\abs{\mbf{w}_t^T\mbf{z}-\mbf{w'}_t^T\mbf{z}}^2}\\
            & \leq\sup_{\mbf{z}\in\mcal{Z}_\zeta}\sqrt{\sum_{t=1}^m\norm{\mbf{w}_t-\mbf{w}_t'}^2\norm{\mbf{z}}^2}\\
            & \leq\sup_{\mbf{z}\in\mcal{Z}_\zeta}\norm{\mbf{z}}\sqrt{\sum_{t=1}^M\norm{\mbf{w}_t-\mbf{w}_t'}^2}\\
            & \leq\zeta\norm{\mbf{W}-\mbf{W}'}_F.
    \end{align*}
    Therefore if $\{\mbf{W}^{(1)}, \ldots,\mbf{W}^{(N)}\}$ is a minimal $\varepsilon/\zeta$-cover of $\mcal{W}$ with respect to $\norm{\cdot}_F$, then $\{f^{(1)},\ldots, f^{(N)}\}$, where $f^{(j)}(\mbf{z})=\sigma\left(\mbf{W}^{(j)^T}\mbf{z}\right)$, is an $\varepsilon$-cover for $\mcal{M}$. Hence,
    $$
    N\left(\varepsilon,\mcal{M},\norm{\cdot}_{\sup}\right)\leq N\left(\frac{\varepsilon}{\zeta},\mcal{W},\norm{\cdot}_F\right),
    $$
    and the desired result then follows from the similar arguments as in the proof of Lemma \ref{Lm: Local Rademacher Complexity of DNN} and Proposition 4.2.12 in \citet{Vershynin_2018}.
\end{proof}

\begin{lemma}\label{Lm: Lipschitz Composition}
    Let $\mcal{H}=\{h:\mbb{R}^d\to\mbb{R}^{d_1}\}$ be a class of functions and let $\mcal{G}=\{g:\mbb{R}^{d_1}\to\mbb{R}^{d_2}\}$ be a class of $\tau$-Lipschitz function. Define
    $$
    \mcal{G}\circ\mcal{H}=\{g\circ h:g\in\mcal{G}, h\in\mcal{H}\}.
    $$
    Then
    $$
    N(\varepsilon_1+\tau\varepsilon_2,\mcal{G}\circ\mcal{H},\norm{\cdot}_{\sup})\leq N\left(\varepsilon_1, \mcal{G}, \norm{\cdot}_{\sup}\right)N(\varepsilon_2, \mcal{H}, \norm{\cdot}_{\sup}).
    $$
\end{lemma}

\begin{proof}
    For any $f\in\mcal{G}\circ\mcal{H}$, there exists $g\in\mcal{G}$ and $h\in\mcal{H}$ such that $f=g\circ h$. Let $\{g_1,\ldots, g_{N_1}\}$ be a minimal $\varepsilon_1$-cover of $\mcal{G}$ so that $N_1=N(\varepsilon_1,\mcal{G},\norm{\cdot}_{\sup})$ and let $\{h_1,\ldots, h_{N_2}\}$ be a minimal $\varepsilon_2$-cover of $\mcal{H}$ so that $N_2=N(\varepsilon_2,\mcal{H},\norm{\cdot}_{\sup})$. Then there exists $i\in[N_1]$ such that
    $$
    \norm{g-g_i}_{\sup}<\varepsilon_1,
    $$
    and there exists $j\in[N_2]$ such that
    $$
    \norm{h-h_j}<\varepsilon_2.
    $$
    Consequently,
    $$
    \norm{g\circ h-g_i\circ h_j}_{\sup}\leq\norm{g\circ h-g\circ h_j}_{\sup}+\norm{g\circ h_j-g_i\circ h_j}_{\sup}.
    $$
    For the first term, by the Lipschitz continuity of $\mcal{G}$,
    \begin{align*}
        \norm{g\circ h-g\circ h_j}_{\sup} & \leq\sup_{\mbf{x}}\norm{g(h(\mbf{x}))-g(h_j(\mbf{x}))}\\
            & \leq\sup_{\mbf{x}}\tau\norm{h(\mbf{x})-h_j(\mbf{x})}\\
            & =\tau\norm{h-h_j}_{\sup}\\
            & <\tau\varepsilon_2.
    \end{align*}
    For the second term, we have
    \begin{align*}
        \norm{g\circ h_j-g_i\circ h_j}_{\sup} & =\sup_{\mbf{x}}\norm{g(h_j(\mbf{x}))-g_i(h_j(\mbf{x}))}\\
            & =\sup_{\mbf{y}=h_j(\mbf{x})}\norm{g(\mbf{y})-g_i(\mbf{y})}\\
            & \leq\norm{g-g_i}_{\sup}<\varepsilon_1.
    \end{align*}
    Therefore, $\{g_i\circ h_j:i\in[N_1], j\in[N_2]\}$ forms an $\varepsilon_1+\tau\varepsilon_2$-cover for $\mcal{G}\circ\mcal{H}$ and the desired result follows.
\end{proof}

\begin{theorem}\label{Thm: Sup Covering Number Upper Bound for DNN}
    Let $\mcal{F}$ be the class of deep ReLU neural networks as defined in (\ref{Eq: Class of DNN}). Then 
    $$
    \log N(\varepsilon,\mcal{F},\norm{\cdot}_{\sup})\leq  W\log\left(1+\frac{2\kappa}{\varepsilon}\left[\sum_{l=1}^L\left(\frac{b_l}{\kappa_l}\right)^{1/2}\prod_{l=1}^L\kappa_l^{1/2}\right]^2\right).
    $$
\end{theorem}

\begin{proof}
    Let $\mcal{F}_i$ and $\mcal{W}_i$, $i=1,\ldots, L$ be the same as defined in the proof of Lemma \ref{Lm: Local Rademacher Complexity of DNN}. It then follows from Lemma \ref{Lm: Covering Number of 1-layer} that
    \begin{align*}
    N(\varepsilon_i,\mcal{F}_i,\norm{\cdot}_{\sup}) &\leq N\left(\frac{\varepsilon_i}{\kappa\prod_{l=1}^{i-1}\kappa_l}, \mcal{W}_i,\norm{\cdot}_F\right)\\
        & \leq \left(1+\frac{2b_i\kappa\prod_{l=1}^{i-1}\kappa_l}{\varepsilon_i}\right)^{p_ip_{i-1}}.
    \end{align*}
    Combined with Lemma \ref{Lm: Lipschitz Composition}, we have
    \begin{align*}
        N\left(\sum_{i=1}^L\varepsilon_i\prod_{l=i+1}^{L}\kappa_l,\mcal{F},\norm{\cdot}_{\sup}\right) & \leq\prod_{i=1}^L N(\varepsilon_i,\mcal{F}_i,\norm{\cdot}_{\sup})\\
        & \leq\prod_{i=1}^L\left(1+\frac{2b_i\kappa\prod_{l=1}^{i-1}\kappa_l}{\varepsilon_i}\right)^{p_ip_{i-1}}
    \end{align*}
    By the same choice of $\varepsilon_i$ as in (\ref{Eq: choice of ei}) and the same argument as in Lemma \ref{Lm: Local Rademacher Complexity of DNN}, we have
    \begin{align*}
     \log N(\varepsilon,\mcal{F},\norm{\cdot}_{\sup}) 
        & \leq W\log\left(1+\frac{2\kappa}{\varepsilon}\left[\sum_{l=1}^L\left(\frac{b_l}{\kappa_l}\right)^{1/2}\prod_{l=1}^L\kappa_l^{1/2}\right]^2\right).
    \end{align*}
\end{proof}

\end{document}